%% file: main.tex
\documentclass[twoside,11pt]{article}

\usepackage[abbrvbib,preprint]{jmlr2e}

\usepackage{amsmath}

\usepackage{amsthm}
\usepackage{dsfont}
\usepackage{amssymb}
\usepackage{mathtools}

\usepackage{algorithm}
\usepackage[noend]{algpseudocode}

\usepackage{enumerate}

\usepackage{hyperref}

\usepackage{natbib}

\include{main-setup}

\jmlrheading{1}{2020}{1-xx}{4/00}{10/00}{xxx}{Damien Garreau and Ulrike von Luxburg}

\ShortHeadings{Looking Deeper into Tabular LIME}{Garreau and von Luxburg}
\firstpageno{1}

\begin{document}

\title{Looking Deeper into Tabular LIME}

\author{\name Damien Garreau 
\email damien.garreau@unice.fr \\
       \addr Laboratoire J. A. Dieudonn\'e \& Inria Maasai project-team\\
       Universit\'e C\^ote d'Azur \\
       Nice, France
       \AND
       \name Ulrike von Luxburg 
       \email ulrike.luxburg@uni-tuebingen.de\\
       \addr University of T\"ubingen \\
       Department of Computer Science\\
       T\"ubingen, Germany}

\editor{Francis Bach, David Blei, and Bernhard Sch{\"o}lkopf}

\maketitle

\begin{abstract}
In this paper, we present a thorough theoretical analysis of the default implementation of LIME in the case of tabular data. 
We prove that in the large sample limit, the interpretable coefficients provided by Tabular LIME can be computed in an explicit way as a function of the algorithm parameters and some expectation computations related to the black-box model. When the function to explain has some nice algebraic structure (linear, multiplicative, or sparsely depending on a subset of the coordinates), our analysis provides interesting insights into the explanations provided by LIME. These can be applied to a range of machine learning models including  Gaussian kernels or CART random forests. As an example, for linear functions we show that LIME has the desirable property to provide explanations that are proportional to the coefficients of the function to explain and to ignore coordinates that are not used by the function to explain.  For partition-based regressors, on the other side, we show that LIME produces undesired artifacts that may provide misleading explanations. 
\end{abstract}



\section{Introduction}

In recent years, many methods aiming to provide \emph{interpretability} of machine learning algorithms have appeared. 
These methods give explanations for virtually all machine learning models currently in use, including the most complex. 
However, it is still unclear how these methods operate in the absence of proper theoretical treatment. 

Interpretability is important for several reasons. 
First, most of the recent advances in machine learning have been achieved by models whose increasing complexity seems to know no limit, namely deep neural networks (DNNs). 
Even though methods such as tree boosting \citep{Chen_Guestrin_2016} remain strong contenders for tabular data \citep{Borisov_et_al_2021}, DNNs are increasingly employed in this setting. 
Though in most cases the model itself may be inaccessible to us, it is interesting to note that sometime the architecture and the parameters are known. 
One could very well read the code and sometime even check the value of each individual coefficient. 
Even then, it is today very challenging for a human to understand how a particular prediction is made simply because of the sheer number of parameters and the complicated architecture. 
Maybe for this reason recent models are more and more frequently referred to as \emph{black boxes}. 

Second, while some users mainly care about performance (that is, accuracy), some specific applications demand \emph{interpretability} of the algorithms involved in the decision-making process. 
This is particularly true in healthcare. 
The main worry of practitioners is that the model that they are training learns a rule yielding good accuracy on the train and test sets, but making little common sense and leading to dramatic decisions when deployed in the wild. 
For instance, \citet{Caruana_et_al_2015} describe a model trained to predict probability of death from pneumonia. 
This model ends up assigning less risk to patients if they also have asthma. 
Of course, from a medical point of view, this is nonsense, and deploying the model in a real-life situation would lead asthmatic patients to receive less prompt treatment and therefore increase their risk of ending up in critical condition. 
One can surmise that the model learned that asthma was predictive of a lower risk because these patients, \emph{a contrario}, received the quickest treatment. 
In this hypothetical example, interpretability of the model would help us not releasing the flawed model. 
For instance, one could investigate a few cases with an interpretability algorithm, and realize that asthma is associated with a decrease in the risk. 
Such sanity check is not possible ``as is'' with most recent machine learning models, due to their complexity.
We refer to \citet{Turner_2016} for other interesting examples. 

The current spread of machine learning in all aspects of our life make the previous example not so hypothetical anymore. 
Thus there is an urgent need for interpretability. 
It is interesting to note that this need is recognized by the lawmakers,  
at least in the European Union, where the European Parliament adopted in 2018 the General Data Protection Regulation (GDPR). 
Part of the GDPR is a clause on automated decision-making, stating to some extent a right for all individuals to obtain ``meaningful explanations of the logic involved.'' 
Even if there is an ongoing debate on whether this disposition is legally binding \citep{Wachter_et_al_2017}, for the first time, we find written in law the idea that decision-making algorithms cannot stay black boxes.

The main question underlying our work is the following: \emph{``Do these explanations make sense from a mathematical point of view, in particular when the black-box model is simple?''}
By simple, we mean a model that a human can already explain to some degree (\emph{e.g.}, a linear model or a shallow decision tree). 
If the answer is negative for some models, we think that this raises concern for the widespread use of such interpretability methods. 
Indeed, if a particular method fails to explain how a simple linear model predicted a value for a given example, how can we trust this same method to explain how a deep convolutional neural network predicted the class of an object in an image?
We believe that there is a need for theoretical guarantees for interpretability. 
There should be some minimal proof of correctness for any interpretability method in order to trust the results thereof. 
For instance, showing that one recovers the important coefficients when the model to explain is linear, or that the algorithm is provably ignoring coordinates that are not used by the model to explain. 
This paper attempts to answer these questions in the case of a method called \emph{Local Interpretable Model-agnostic Explanations} \citep[LIME,][]{Ribeiro_et_al_2016}.
We will see that the answer to both of these questions is \emph{affirmative} when the default implementation of LIME is used for tabular data with no feature selection. 

Without giving too much details on the inner working of LIME (which we will do in Section~\ref{section:tabular-lime}), we want to briefly summarize how it operates. 
Essentially, given a black-box model~$f$, in order to explain an example $\xi\in\Reals^d$, LIME 
\begin{enumerate}[(1)]
	\item creates perturbed examples $x_i$;
	\item gets the prediction of the black box model $f$ at these examples;
	\item weights the predictions with respect to the proximity between $x_i$ and $\xi$;
	\item trains a weighted interpretable model on these new examples.
\end{enumerate}
The output of LIME is then the top coefficients of the interpretable model (if it is linear). 
What makes LIME really interesting and so popular is the use of \emph{interpretable features}, namely discretized features. 
For instance, if the third feature is real-valued, instead of saying that ``feature number $3$ is important for the prediction $f(\xi)$,'' LIME indicates to the user that ``feature $3$ being between $1.5$ and $7.8$ is important.'' 
More precisely, LIME outputs a linear surrogate model, whose coefficients (the \emph{interpretable coefficients}) are given as explanation to the user. 
These coefficients are our primary center of interest.

\paragraph{Main findings. }
In this paper, we restrict ourselves to the default implementation of LIME for tabular data (that is, data lying in $\Reals^d$). 
Our main results are: 
\begin{itemize}
	\item {\bf Explicit expression for LIME's surrogate model (Section~\ref{section:statement-main-result}).} We show that when the surrogate model is obtained by ordinary least-squares and the number of perturbed samples is large, the interpretable coefficients obtained by LIME are close to a vector~$\beta^f$ whose expression is explicit. This statement is true with high probability with respect to the sampling of perturbed examples. 
\item {\bf Linearity of explanations (Section~\ref{sec:explanation-linearity}).} 
Leveraging the explicit expression of~$\beta^f$, we show several interesting properties of LIME. In particular,~$\beta^f$ is \emph{linear} in~$f$, and as a consequence is robust with respect to small perturbations of $f$. 
\item {\bf Stability of explanations (Section~\ref{sec:bin-indices}).} Explanations only depends on the bins into which $\xi$ belongs. As a consequence, they are quite stable with respect to perturbation of the input while the example does not cross a boundary and can be completely different otherwise. 
\item {\bf Role of the hyperparameters (Section~\ref{section:bandwidth}).}  We also obtain the behavior of $\beta^f$ for small and large bandwidth (the main free parameter of the method). In particular, we show that the interpretable coefficient can cancel out for certain choices of bandwidth. 
\item {\bf Linear model (Section~\ref{section:additive}).} When $f$ has some simple algebraic structure, we show how to compute~$\beta^f$ in closed-form. In particular, when $f$ is linear, we recover the main result of \citet{Garreau_von_Luxburg_2020}, but this time for the default weights, arbitrary bins, and arbitrary input parameters. We end up with essentially the same conclusion: the explanations provided by LIME are proportional to the coefficients of~$f$ along each coordinate. 
\item {\bf More general models (Section~\ref{section:multiplicative}).} In this paper, we reach beyond the linear case, obtaining explicit results for  multiplicative~$f$. This encompasses, for instance, indicator functions with rectangular domain and the \textbf{Gaussian kernel}. By linearity, we obtain a closed-form statement for partition-based regressor, for instance CART trees. In this last case, we leverage the theory to explain artifacts observed when explaining a \textbf{CART tree} with LIME. 
\item {\bf Dummy features (Section~\ref{section:sparse}).} We also show that unused features are ignored by LIME's explanation: they receive weight~$0$ in the surrogate model (up to noise coming from the sampling). 
\end{itemize}
The main difficulty in our analysis comes from the non-linear nature of the interpretable features (defined as indicator functions) and the complicated overall machinery of LIME. 
In contrast with our previous work \citep{Garreau_von_Luxburg_2020}, we managed to keep the analysis very close to the default implementation (found at \url{https://github.com/marcotcr/lime} as of April 2022), the only visible price to pay being additional notation.
As we will see, these become quite manageable in the simple cases that we consider. 
All our theoretical claims are validated experimentally and the code of all the experiments of the paper can be found at \url{https://github.com/dgarreau/tabular_lime_theory}.

\paragraph{Related work. }
LIME is a \emph{posthoc, local} interpretability method. 
In other words, it provides explanations (i) ``after the fact'' (the model is already trained), and (ii) for a specific example. 
We refer to the exhaustive review papers by \citet{Guidotti_et_al_2018} (especially Section~7.2), and \citet{Adadi_Berrada_2018} for an overview of such methods. 
It seems that LIME has quickly become one of the most widely-used posthoc interpretability methods. 
But besides the practical interest, LIME has also generated consequent academic attention, with many variations on the method being proposed in the last three years. 
For instance, the same set of authors later proposed Anchors \citep{Ribeiro_et_al_2018}, also based on the production of perturbed examples, but producing simpler ``if-then'' rules as explanations. 
Further specializations of LIME were proposed, in specific settings, namely time series analysis \citep{Mishra_et_al_2017}, and survival model analysis \citep{Kovalev_et_al_2020,Utkin_et_al_2020}.
More recently, some attempts have been made to refine further the method in order to improve both the interpretability and the fidelity (the accuracy of the surrogate model) \citep[LEDSNA,][]{Shi_Du_Fan_2020}. 

A very related line of work is connected to Shapley value. 
Rooted in cooperative game theory~\citep{Shapley_1953}, its original purpose is to share the overall gain among the players by looking at the gain for any given coalition of players with or without a given player. 
When considering as players the features of the model, as coalition a subset of the features, and as gain for a coalition the value of the model retrained on this subset of the features, \citet{Strumbelj_Kononenko_2010} showed that it was a valid feature attribution method for machine learning models. 
Retraining the model for each subset of features can be costly, and \citet{Strumbelj_Kononenko_2014} proposed to use the conditional expectation of the model with a subset of the features being fixed instead. 
Testing all coalitions has a computational cost that is exponential in the number of features, and numerous approximations have been proposed. 
Most notably, \citet{Lundberg_Lee_2017} proposed SHAP (SHapley Additive exPlanations) and more specifically kernel SHAP as a way to approximate Shapley values by solving a weighted least squares problem. 
This least squares problem is very close to the one at the core of LIME (Eq.~\eqref{eq:surrogate-model-general}). 
We identify two key differences: (i) the weights given by kernel SHAP to perturbed samples, chosen to recover Shapley value at the limit, which do not fall into our setting, and (ii) the absence of discretization for kernel SHAP. 
For a synthetic view on Shapley value and their application to interpretability, we refer to \citet{Sundararajan_Najmi_2020}. 


The closest work to the present paper is our previous work \citet{Garreau_von_Luxburg_2020}, which considers a modification of the LIME algorithm for tabular data. 
Namely, the interpretable components are chosen in a very specific way (the quantiles are those of a Gaussian distribution), and the parameters of the algorithm match the mean of this Gaussian. 
Moreover, the weights of the perturbed examples are not the weights used in the default implementation of LIME. 
In this setting, we showed that the values of the interpretable coefficients stabilize towards some values which are attained in closed-form when the model to interpret is \emph{linear}. 

The analysis in the current paper goes far beyond that conducted in \citet{Garreau_von_Luxburg_2020}. 
In particular, we consider (i) arbitrary discretization, (ii) general weights including the weights chosen by the default implementation, and (iii) non-linear models, including radial basis kernel function and indicator functions with rectangular support. 
Hence, our analysis does not concern a considerably simplified version, but the LIME algorithm in its default implementation, and we can apply our more general results to much larger range of algorithms including kernel algorithms and random forests, for example. 
In the Appendix we explain how the present paper can recover our earlier analysis as a special case. 
%
When no discretization step is considered, \citet{Agarwal_et_al_2021} showed that LIME's explanations were close to that of gradient-based methods. 
Finally, we note that during the revision process of this paper, the analysis of LIME was also extended to text data \citet{Mardaoui_Garreau_2021} and images \citet{Garreau_Mardaoui_2021}.

\paragraph{Summary of the paper.}
In Section~\ref{section:tabular-lime}, we introduce LIME in the context of tabular data. 
From now on, we will refer to this version of LIME as \emph{Tabular LIME}. 
Section~\ref{section:main-result} contains our main result, as well as a short discussion and an outline of the proof. 
The implications of the result for models having a nice algebraic structure are discussed further in Section~\ref{section:discussion}. 
All the proofs and additional results are collected in the Appendix. 

\section{Tabular LIME}
\label{section:tabular-lime}

In this section we present Tabular LIME and introduce our notation along the way. 
Originally proposed for text and image data, the public implementation of LIME also has a version for tabular data. 
It is on this version that we focus in the present paper. 

We will assume that all the features are \emph{continuous}. 
Indeed, when there are some discrete features, the sampling scheme of Tabular LIME changes slightly and so would our analysis. 
Note that the default implementation discretizes the continuous features: this is the road we will follow. 
It is important to note that other implementation may have different discretization schemes and therefore may have different behaviors. 
For instance, \texttt{localModel}\footnote{\url{https://cran.r-project.org/web/packages/localModel/index.html}} uses a discretization scheme specific to each instance to explain, 
whereas \texttt{iml} \citep{Molnar_et_al_2018} does not consider a discretization step at all. 


\subsection{Quick overview of the algorithm}

\begin{figure}
	\begin{center}
		\includegraphics[scale=0.3]{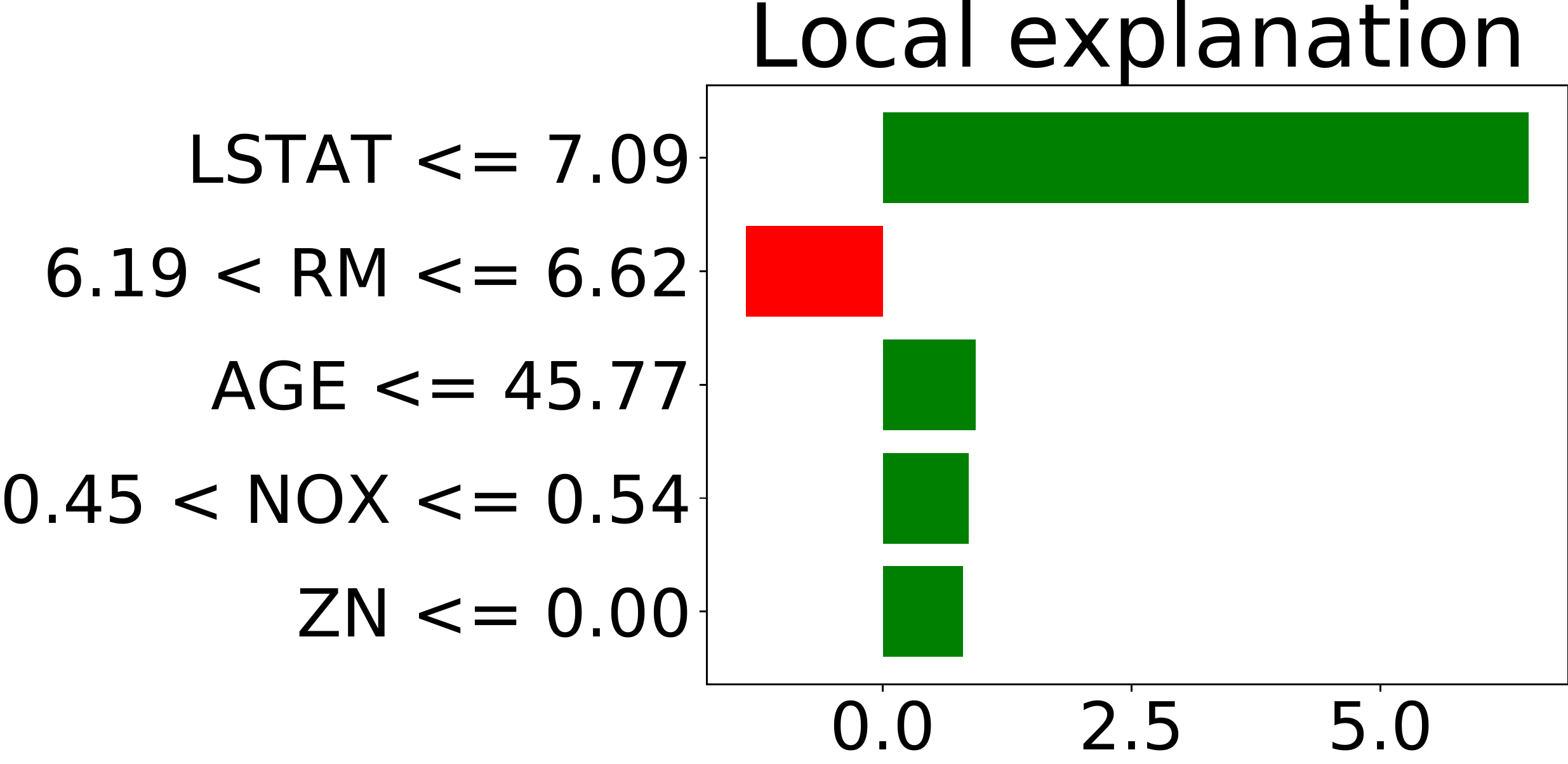}
	\end{center}
	\caption{\label{fig:example-default}Example output of Tabular LIME. Here we considered a random forest classifier on the Boston Housing dataset \citep{Harrison_Rubinfeld_1978}. 
	Each row corresponds to a bin along a particular feature, the bar in each row shows the influence of this discretized feature for the prediction according to Tabular LIME. 
	This influence can be positive (green) or negative (red). 
	Here we can see (for instance) that a low value of LSTAT (lower status population) had a positive influence on the predicted variable, the price. See Figure~\ref{fig:example-ours} for an alternative presentation. }
\end{figure} 

\paragraph{Regression. }
The core procedure of LIME is designed for regression: LIME aims to explain a real-valued function $f:\Reals^d\to\Reals$ at a fixed example $\xi\in\Reals^d$. 
If one desires to apply LIME for classification, then one must use LIME for regression with~$f$ the likelihood of being in a given class, that is,
\[
f(x) = \condproba{y(x) = c}{\zeta_1,\ldots,\zeta_m}
\, ,
\]
where $\zeta_1,\ldots,\zeta_m\in\Reals^d$ is a training set. 
In this paper, we study Tabular LIME in the context of \emph{regression}, with the undertone that it is possible to readily transpose our results for \emph{classification}. 

\begin{algorithm}[ht]
	\caption{\label{algo:getting-bins}\texttt{GetSummaryStatistics}: Getting summary statistics from training data}
	\begin{algorithmic}[1]
		\Require{Train set $\traindata=\{\zeta_1,\ldots,\zeta_m\}$, number of bins $p$}
		\For{$j=1$ to $d$}
		\For{$b=0$ to $p$}
		\State{$q_{j,b}\leftarrow \texttt{Quantile}(\zeta_{1,j},\ldots,\zeta_{m,j};b/p)$} \Comment{split each dimension into $p$ bins}
		\EndFor 
		\For{$b=1$ to $p$}
		\State{$\mathcal{S}\leftarrow \{\zeta_{i,j} \quad\text{s.t.} \quad q_{j,b-1}< \zeta_{i,j}\leq q_{j,b}\}$} \Comment{get the training data falling into bin $b$}
		\State{$\mu_{j,b} \leftarrow \texttt{Mean}(\mathcal{S})$} \Comment{compute the empirical mean}
		\State{$\sigma_{j,b}^2\leftarrow \texttt{Var}(\mathcal{S})$}\Comment{compute the empirical variance}
		\EndFor 
		\EndFor
		\State{\Return $q_{j,b},\mu_{j,b},$ and $\sigma_{j,b}^2$ for $1\leq j\leq d$ and $1\leq b\leq p$}
	\end{algorithmic}
\end{algorithm}

\paragraph{General overview. }
Let us now detail how Tabular LIME operates, a prerequisite to the analysis we aim to conduct. 
We begin with a high-level description of the algorithm. 
In the next section, we detail each step and fix additional notation. %
\begin{itemize}
    \item \textbf{Step 1: binning.} First, Tabular LIME creates interpretable features by splitting each feature's input space in a fixed number of bins $p$. The idea is to have ranges, not only features, as outputs of the algorithm. As in Figure~\ref{fig:example-default}, we prefer to know that a low value of a parameter is important for the prediction, not only that the parameter itself is important. We explain the bin creation in more details in the next section and we refer to Algorithm~\ref{algo:getting-bins} for additional details. 
    \item \textbf{Step 2: sampling.} Second, Tabular LIME samples perturbed examples $x_1,\ldots,x_n\in\Reals^d$. 
For each new example, Tabular LIME samples a bin uniformly at random on each axis and then samples according to a truncated Gaussian whose parameters are given as input to the algorithm. 
When no training data is provided, one can also directly provide to Tabular LIME the coordinates of the bins and the mean and variance parameters for the sampling. 
The intuition is to try and mimic the distribution of the data, even though this data may not be available (remember that LIME aims to explain a black-box model $f$). 
We describe the sampling procedure in more details in the next section and we refer to  Algorithm~\ref{algo:sampling-perturbed-examples} for a synthetic view of the sampling procedure.  
\item \textbf{Step 3: surrogate model.} Finally, a surrogate model is trained on the interpretable features, weighted by some positive weights $\pi_i$ depending on the distance between the $x_i$s and $\xi$. 
The final product is a visualization of the most important coefficients if no feature selection mode is selected by the user. 
Algorithm~\ref{algo:tabular-lime} summarizes Tabular LIME, while Figure~\ref{fig:example-default} presents a typical output. 
\end{itemize}

\begin{algorithm}[ht]
\caption{\label{algo:sampling-perturbed-examples}\texttt{Sample}: Sampling a perturbed example}
\begin{algorithmic}[1]
\Require{Bin boundaries $q_{j,p}$, mean parameters $\mu_{j,p}$, variance parameters $\sigma^2_{j,p}$, bin indices of the example to explain,~$\bxi$}
\For{$j=1$ to $d$}
\State $b_j\leftarrow$ \texttt{SampleUniform}($\{1,\ldots,p\}$) \Comment{sample bin index}
\State $(q_\ell,q_u)\leftarrow (q_{j,b_{j}-1},q_{j,b_{j}})$ \Comment{get the bin boundaries}
\State $x_{j}\leftarrow$ \texttt{SampleTruncGaussian}($q_\ell,q_u,\mu_{j,b},\sigma_{j,b}^2$) \Comment{sample a truncated Gaussian}
\State $z_{i,j}\leftarrow \indic{b_{j} = \bxi_j}$ \Comment{mark one if same box}
\EndFor
\State{\Return $x,z$}
\end{algorithmic}
\end{algorithm}

In order to analyze Tabular LIME, we need to be more precise with regards to the operation of the algorithm. 
We now proceed to give more details about the sampling procedure (Section~\ref{section:sampling-procedure}) and the training of the surrogate model (Section~\ref{section:surrogate-model}).

\begin{algorithm}[ht]
	\caption{\label{algo:tabular-lime}\texttt{TabularLIME}: Tabular LIME for regression, default implementation}
	\label{algo:tlime}
	\begin{algorithmic}[1]
		\Require Black-box model $f$, number of new samples $n$, example to explain $\xi$, positive bandwidth $\nu$, number of bins $p$, bin boundaries $q_{j,b}$, means $\mu_{j,b}$, variances $\sigma^2_{j,b}$
		\State $\bxi \leftarrow \texttt{BinIDs}(\xi,q)$ \Comment{get the bin indices of $\xi$}
		\For{$i=1$ to $n$}
		\State{$x_i,z_i \leftarrow $\texttt{Sample}($q,\mu,\sigma,\bxi$)}
		\State $\pi_i\leftarrow \exp{\frac{-\norm{\Indic - z_i}^2}{2\nu^2}}$ \Comment{compute the weight}
		\EndFor
		\State 
		$\betahat \in\Argmin_{\beta\in\Reals^{d+1}} \sum_{i=1}^n \pi_i (f(x_i) - \beta^\top z_i)^2 + \Omega(\beta)$ \Comment{compute the surrogate model}
		\State \Return $\betahat$
	\end{algorithmic}
\end{algorithm}


\subsection{Sampling perturbed examples}
\label{section:sampling-procedure}

In this section, we explain exactly how the sampling of perturbed examples introduced in Algorithm~\ref{algo:sampling-perturbed-examples} operates. 

\paragraph{Discretization. }
The first step in Tabular LIME is to create \emph{interpretable features} along each dimension $j\in\{1,\ldots,d\}$. 
This is achieved by dividing the input space into hyper-rectangular bins. 
The intuition for this is added interpretability: instead of knowing that a coordinate is important for the prediction, Tabular LIME gives a range of values for the coordinate that is important for the prediction. 
More precisely, each dimension is divided into $p\geq 2$ bins, where~$p$ is constant across all dimensions $1\leq j\leq d$. 
This is the default behavior of Tabular LIME (with $p=4$), although it can happen that $p_j<p$ if there are not enough data along an axis. 
We will assume throughout the paper that there are always enough data points if a training set is used, even though the current analysis can be extended to $p_j$ varying across dimensions.   
Note that if $p=1$, there are no bins: $z_{ij}=1$ for any $i,j$ and the surrogate model is just the empirical mean of the $f(x_i)$s.  

The boundaries of the bins are an input of Tabular LIME (Algorithm~\ref{algo:tabular-lime}). 
For each feature $j\in\{1,\ldots,d\}$, we denote these boundaries by 
\[
q_{j,0} < q_{j,1} < \cdots < q_{j,p}
\, .
\]
The products of these bins form a partition of $\supp \defeq \prod_{j=1}^{d}[q_{j,0},q_{j,p}]$. 
In addition to the bins boundaries, Tabular LIME takes as input some mean and variance parameters for each bin along each dimension. 
We denotes these by $(\mu_{j,b},\sigma_{j,b}^2)\in\Reals\times \Reals_+$. 
As we will see in the next paragraph, these parameters (bin boundaries, means, and standard deviation) are computed from a training set if provided. 

\begin{algorithm}[ht]
	\caption{\label{algo:bin-numbers}\texttt{BinID}: Getting the bin indices of the instance to explain}
	\begin{algorithmic}[1]
		\Require Example to explain $\xi$, bin boundaries $q_{j,b}$
		\For{$j=1$ to $d$}
		\For{$b=1$ to $p$}
		\If{$q_{j,b-1} < \xi_j < q_{j,b}$} \Comment{if $\xi_j$ belongs to bin $b$}
		\State $\bxi_j = b$ \Comment{then $\bxi_j=b$}
		\State \textbf{break}
		\EndIf
		\EndFor
		\EndFor
		\State \Return $\bxi$
	\end{algorithmic}
\end{algorithm}

\paragraph{Training data.}
In the default operation mode of Tabular LIME, a training set $\traindata = \{\zeta_1,\ldots,\zeta_m\}$ 
is given as an input to Tabular LIME. 
Note that this training set is \emph{not necessarily} the training set used to train the black-box model~$f$. 
In that case, the boundaries of the bins are obtained by considering the quantiles of $\traindata$ across each dimension. 
Intuitively, along each dimension~$j$, we split the real line in~$p$ bins such that a proportion $1/p$ of the data falls into each bin along this axis.
More formally, the boundaries are obtained by taking the quantiles of level $b/p$ for $b\in\{0,1,\ldots,p\}$. 
That is, the $q_{j,b}$ are such that 
\[
\forall 1\leq j\leq , \, \forall 1\leq b\leq p,\quad \frac{1}{m}\sum_{i=1}^{m}\indic{\zeta_{i,j}\in [q_{j,b-1},q_{j,b}]} \approx \frac{1}{p}
\, .
\]
In particular, for each $1\leq j\leq d$, $q_{j,0}=\min_{1\leq i\leq m} \zeta_{i,j}$ and $q_{j,p}=\max_{1\leq i\leq m} \zeta_{i,j}$. 

When a training set $\traindata$ is used, $\mu_{j,b}$ (resp. $\sigma_{j,b}$) corresponds to the mean (resp. the standard deviation) of the training data on the bin $b$ along dimension $j$. 
We refer to Figure~\ref{fig:sampling-scheme} for a visual depiction of this process. 

Assuming that the example to explain $\xi$ belongs to $\supp$, then each $\xi_j$ falls into a bin $\bxi_j$ along coordinate~$j$. 
The (straightforward) computation of $\bxi$ is given by Algorithm~\ref{algo:bin-numbers}. 
We will see that $\bxi$ is an important quantity: the $d$-dimensional bin containing $\xi$ plays a special role in our results. 

\begin{figure}
	\begin{center}
		\includegraphics[scale=0.4]{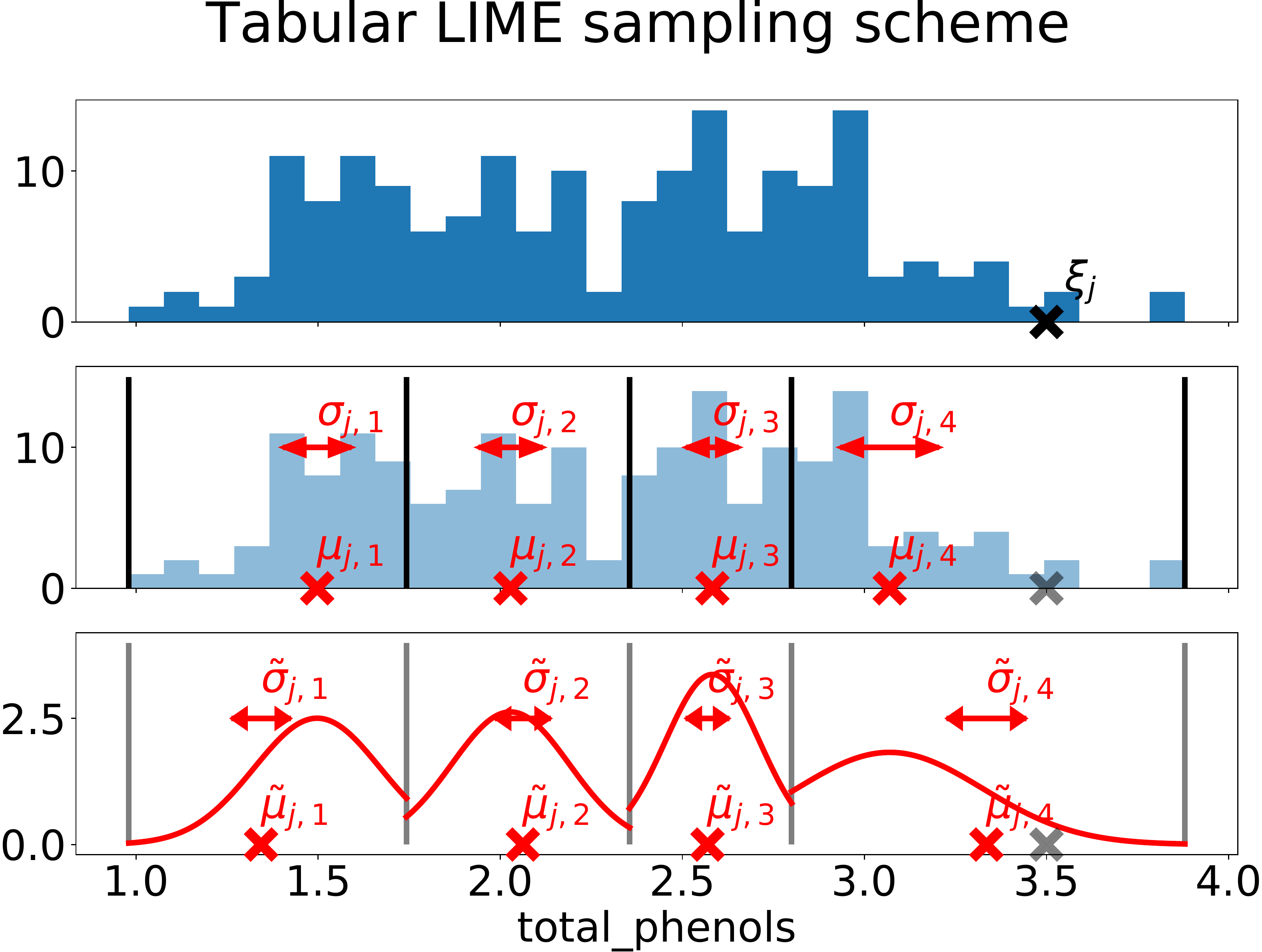}
	\end{center}
	\caption{\label{fig:sampling-scheme}Sampling of perturbed examples.  Data come from the Wine dataset \citep{Cortez_et_al_1998}. \emph{Top panel:} histogram of the values taken by the \texttt{total\_phenol} parameter in the original dataset. The feature value of $\xi$ is denoted by a black cross. \emph{Middle panel:} Tabular LIME computes the quantiles (here $p=4$), then the means and standard deviations for each bin (in red). We see that $\xi_j$ belongs to the fourth bin, therefore $\bxi_j=4$.  \emph{Bottom panel:} choose a bin uniformly at random, then sample according to a truncated Gaussian on the bin with parameters given by the mean and variance on the bin. We report the corresponding probability density function in red, as well as $\mutilde_{j,b}$ and $\sigmatilde_{j,b}$.}
\end{figure}

We note that \citet{Garreau_von_Luxburg_2020} consider the case where $\mu_{j,b}=\mu_j$ and $\sigma_{j,b}=\sigma_j$ for any $b$ as simplifying assumptions. 
In addition, the bins boundaries $q_{j,b}$ were chosen to be the exact Gaussian quantiles. 
We do not make such assumptions here. 

\paragraph{Sampling scheme. }
The next step is the sampling of~$n$ perturbed examples $x_1,\ldots,x_n\in\Reals^d$. 
First, along each dimension, Tabular LIME samples the bin indices of the perturbed samples. 
We write this sample as a matrix $B\in\Reals^{n\times d}$, where $b_{ij}$ corresponds to the bin index of example~$i$ along dimension~$j$. 
In the current implementation of Tabular LIME, the $b_{i,j}$s are i.i.d.\ distributed uniformly on $\{1,\ldots,p\}$. 

The bin indices $b_{i,j}$s are subsequently used in two ways. 
On one hand, Tabular LIME creates the binary features based on these bin indices. 
Formally, we define $z_{i,j}=\indic{b_{i,j}=\bxi_j}$, that is, we mark one if $b_{i,j}$ is the same bin as the one~$\xi_j$ (the $j$th coordinate of the example~$\xi$) falls into. 
We collect these binary features in a matrix $Z\in\Reals^{n\times (d+1)}$ defined as
\[
Z \defeq 
\begin{pmatrix}
1 & z_{1,1} & z_{1,2} & \ldots & z_{1,d} \\
1 & z_{2,1} & z_{2,2} & \ldots & z_{2,d} \\
\vdots & \vdots & \vdots & & \vdots \\
1 & z_{n,1} & z_{n,2} & \ldots & z_{n,d} \\
\end{pmatrix}
\, .
\]
On the other hand, the bin indices are used to sample the new examples $x_1,\ldots,x_n$. 
Let $1\leq i\leq n$, the new sample $x_i$ is sampled independently dimension by dimension in the following way: $x_{i,j}$ is distributed according to a truncated Gaussian random variable with parameters $q_{j,b_{ij}-1}$, $q_{j,b_{i,j}}$, $\mu_{j,b_{i,j}}$, and $\sigma^2_{j,b_{i,j}}$. 
More precisely, $x_{i,j}$ conditionally to $\{b_{i,j}=b\}$ has a density given by
\begin{equation}
\label{eq:tn-density}
\rho_{j,b}(t) \defeq \frac{1}{\sigma_{j,b}\sqrt{2\pi}} \cdot \frac{\exp{\frac{-(t-\mu_{jb})^2}{2\sigma_{j,b}^2}}}{\Phi(r_{j,b}) - \Phi(\ell_{j,b})}\indic{t\in [q_{j,b-1},q_{j,b}]}
\, ,
\end{equation}
where we set $\ell_{j,b} \defeq \frac{q_{j,b-1}-\mu_{j,b}}{\sigma_{j,b}}$ and $r_{j,b}\defeq \frac{q_{j,b}-\mu_{j,b}}{\sigma_{j,b}}$, and $\Phi$ is the cumulative distribution function of a standard Gaussian random variable. 
We denote by $\truncnorm{\mu,\sigma^2,\ell,r}$ the law of a truncated Gaussian random variable with mean parameter $\mu$, scale parameter $\sigma$, left and right boundaries $\ell$ and $r$. 
Note that the means (resp. standard deviations) of these truncated random variables are generally different from the input means (resp. standard deviations). 
We denote by $\mutilde_{j,b}$ (resp. $\sigmatilde_{j,b}$) the mean (resp. standard deviation) of a $\truncnorm{\mu_{j,b},\sigma^2_{j,b},q_{j,b-1},q_{j,b}}$ random variable. 
We refer to Figure~\ref{fig:sampling-scheme} for an illustration.

It is important to understand that \textbf{the sampling of the new examples does not depend on $\xi$}, but rather on the bin indices of $\xi$. 
Therefore, any two given instances to explain will lead to the same sampling scheme provided that they fall into the same $d$-dimensional bin. 
We also note that \textbf{the sampling scheme of LIME is feature-wise independent}. 
In particular, even though features $j$ and $k$ may be highly correlated in the training set, there is no trace of it in the resulting sampling. 

Of course, using truncated Gaussian distributions can be suboptimal in situations where the distribution of the data is more heavy-tailed. 
This is in particular the case in finance, where a larger than expected proportion of the data falls outside a few standard deviations from the mean \citep{Eberlein_Keller_1995}. 
Sampling with heavy-tailed distributions such as \emph{stable distributions} \citep{Nolan_2003} could be more appropriate, though we are not aware of an implementation of Tabular LIME with such sampling on the continuous features at the time of writing.


\subsection{Surrogate model}
\label{section:surrogate-model}

We now focus on the training of the surrogate model. 
As announced in Algorithm~\ref{algo:tabular-lime}, the new samples receive positive weights given by
\begin{equation}
\label{eq:def-weights-default}
\pi_i \defeq \exp{\frac{-\norm{\Indic - z_i}^2}{2\nu^2}} = \exp{\frac{-1}{2\nu^2}\sum_{j=1}^d \indic{b_{i,j} \neq \bxi_j}}
\, ,
\end{equation}
where $\norm{\cdot}$ is the Euclidean norm, and $\nu$ is a positive bandwidth parameter. 
Intuitively, this weighting scheme counts in how many coordinates the bin of the perturbed sample differs from the bin of the example to explain and then applies an exponential scaling. 
If all the bins are the same (the perturbed sample falls into the same hyperrectangle as $\xi$), then the weight is $1$. 
On the other hand, if the perturbed sample is ``far away'' from $\xi$, $\pi_i$ can be quite small (relatively to $\nu$). 
Here, far away means that $x_i$ does not fall into the same bins as $\xi$: it can be the case that $x_i$ is close in Euclidean distance to $\xi$. 
In the default implementation of Tabular LIME, the bandwidth parameter $\nu$ is fixed to $\sqrt{0.75 d}$, but we do not make this assumption and work with arbitrary positive~$\nu$. 
Our main result is actually true for more general weights, see Appendix~\ref{section:general-weights} for their definition. 
In particular, this generalization includes the weights studied in \citet{Garreau_von_Luxburg_2020} as a special case.

The local surrogate model of LIME is then obtained by optimizing a regularized, weighted square loss
\begin{equation}
\label{eq:surrogate-model-general}
\betahat_n \in \Argmin_{\beta\in\Reals^{d+1}} \biggl\{\sum_{i=1}^n \pi_i (f(x_i) - \beta^\top z_i)^2 + \Omega(\beta)\biggr\}
\, ,
\end{equation}
where $z_i$ is the $i$th row of $Z$. 
The coefficients of the surrogate model, collected in $\betahat_n$, are the central output of Tabular LIME and will be our main focus of interest. 
We often refer to the coordinates of $\betahat_n$ as the \emph{interpretable coefficients}. 

\paragraph{Regularization.}
In the default implementation, when no feature selection mechanism is enforced, Problem~\eqref{eq:surrogate-model-general} is solved with $\Omega(\beta)=\lambda\norm{\beta}^2$, where $\norm{\cdot}$ is the Euclidean norm and~$\lambda$ is a positive regularization parameter. 
That is, Tabular LIME is using ridge regression \citep{Hoerl_Kennard_1970} to obtain the surrogate model. 
But the default choice of hyperparameters makes the surrogate models obtained by ridge in this setting indistinguishable from one obtained by  \emph{ordinary least-squares}. 
More precisely, Tabular LIME is using the \texttt{scikit-learn} default implementation of ridge, which has a penalty constant equal to one (that is, $\lambda = 1$). 
For a large choice of bandwidth, the weights $\pi_i$ are close to $1$ and the $\norm{y - ZW\beta}^2$ term in Eq.~\eqref{eq:surrogate-model-general} is $\bigo{n}$ and dominates unless the penalty constant is at least of the same order. 
But the default~$n$ is equal to $5000$ and we investigate the limit in large sample size (meaning that~$n$ is always of order $10^3$ in our experiments), thus there is virtually no difference between taking ordinary least-squares and ridge.
Therefore, even though our analysis is true only for ordinary least-squares ($\lambda = 0$), we recover the results of the default implementation for all reasonable bandwidth choices.  
As a consequence, unless otherwise mentioned, all the experiments in the paper are done with the default regularization ($\lambda=1$). 
For smaller bandwidth parameters, the weights are closer to $0$ and the regularization begins to play a non-negligible role. 
We demonstrate the effects of regularization in  Figure~\ref{fig:regularization-default-linear}. 

\begin{figure}
	\begin{center}
		\includegraphics[scale=0.15]{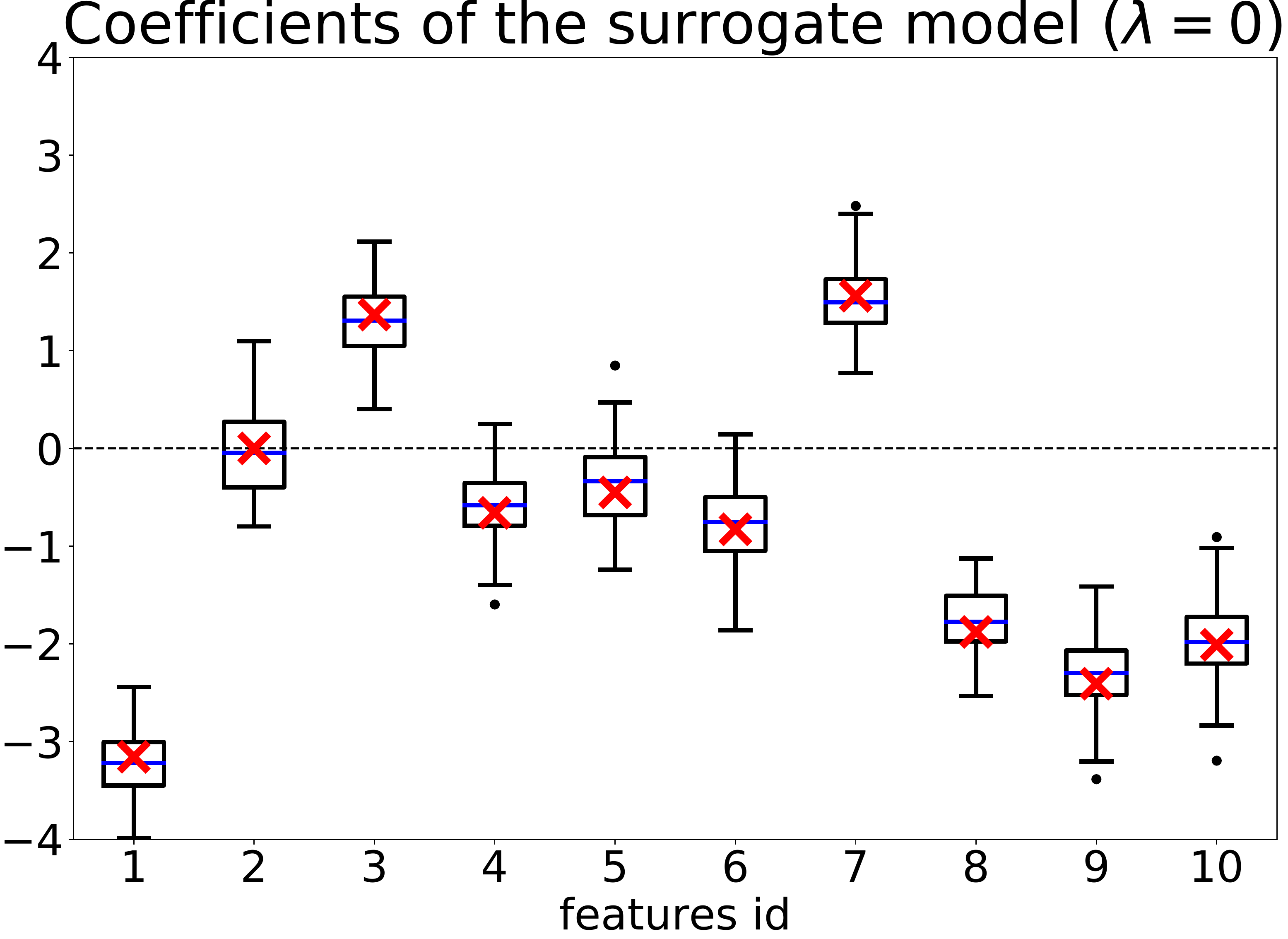}
		\hspace{0.15cm}
		\includegraphics[scale=0.15]{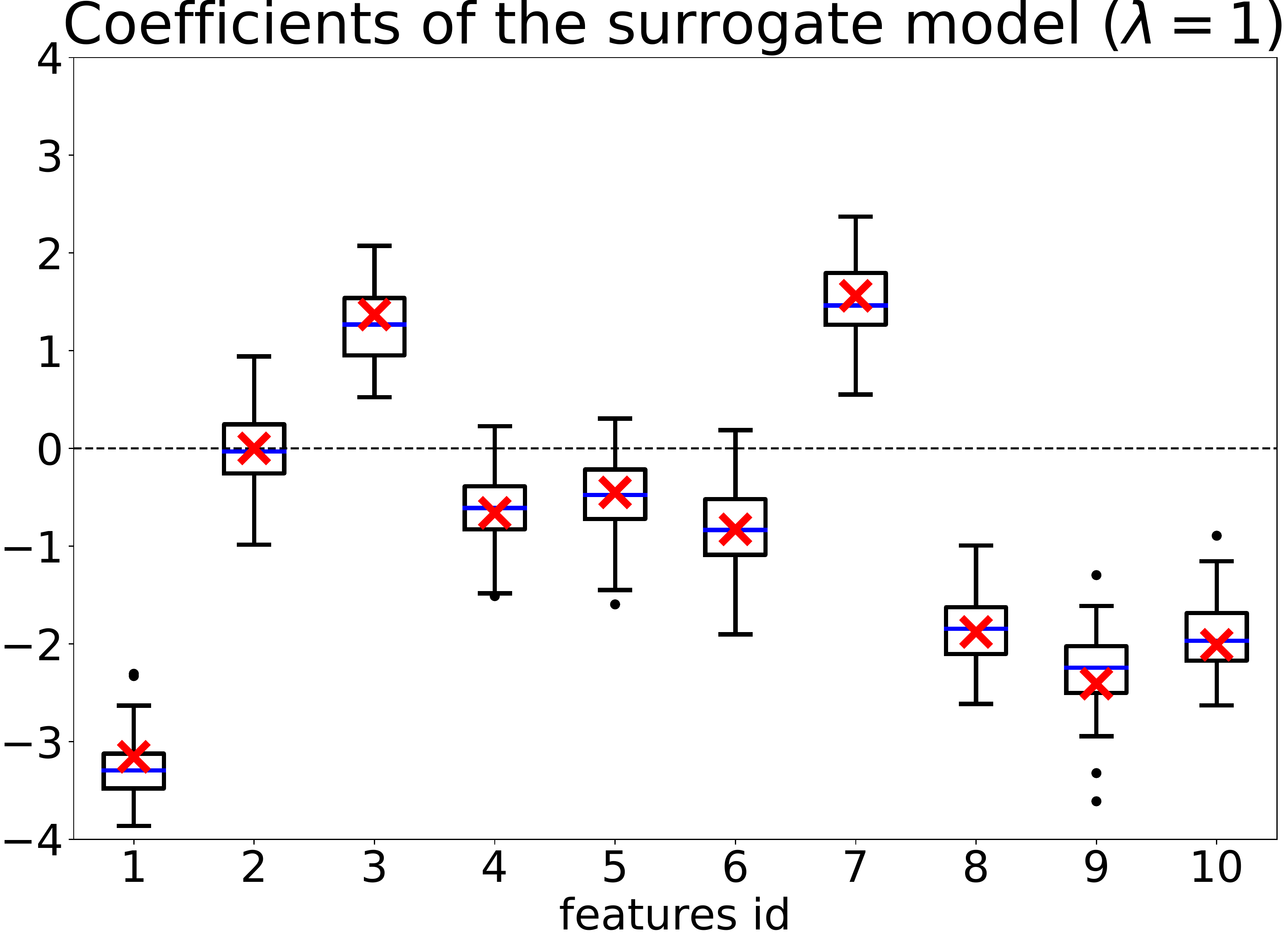}
		\hspace{0.15cm}
		\includegraphics[scale=0.15]{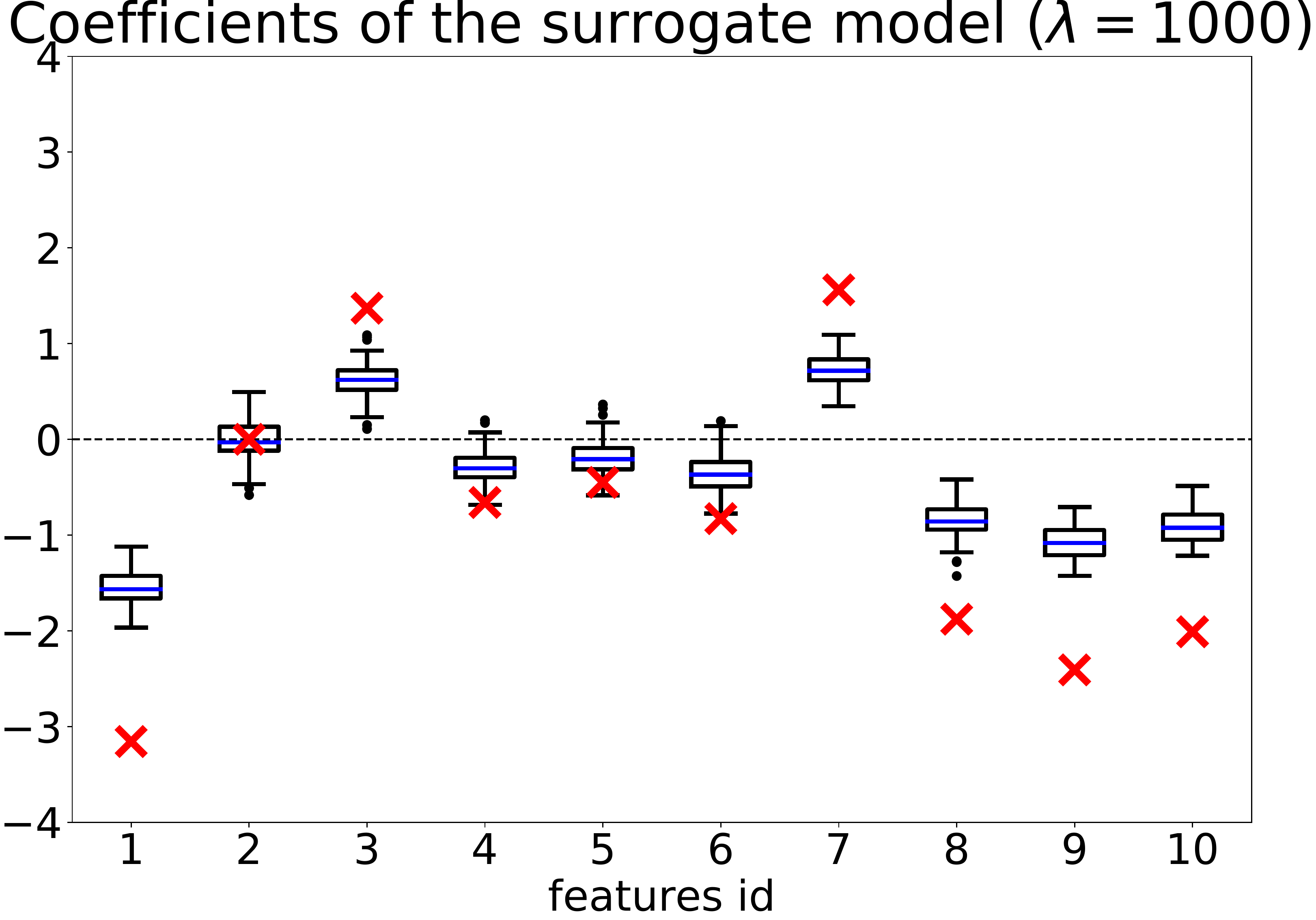}
	\end{center}
	\caption{\label{fig:regularization-default-linear}Effect of the regularization on the surrogate model. The black-box model is linear. We report $100$ runs of Tabular LIME for $5000$ perturbed samples on the same $\xi$ with $\nu=10$. In red, the theoretical predictions given by Corollary~\ref{cor:beta-computation-liner-default-weights}. \emph{Left panel:} no regularization, the surrogate model is trained with ordinary least-squares. This situation corresponds to our analysis. \emph{Middle panel:} $\lambda=1$, default choice in \texttt{scikit-learn}. We take this default choice of regularization in all our experiments. \emph{Right panel:} $\lambda=1000$. When $\lambda$ is of order $n$, the number of perturbed samples, we begin to see the effect of regularization. In effect, the interpretable coefficients are shrinked towards zero, and our theoretical predictions only provide an upper envelope. }
\end{figure}

\paragraph{Presentation of the experimental result. }
In the final step of Tabular LIME, the user is presented with a visualization of the largest coefficients of the surrogate model. 
Note that some feature selection mode can be used before training the surrogate model.  
In that event, the final output of Tabular LIME is not the given of all coefficients of the surrogate model (or setting them to zero). 
We do not consider feature selection in the present work and therefore report all the interpretable coefficients since there is randomness in the ranking of the coefficients due to the randomness of the sampling. 
Note also that because of this randomness in the construction of $\betahat$ \emph{via} the sampling of the perturbed examples $x_1,\ldots,x_n$, we will always report the result of several runs of Tabular LIME on any given example (usually $100$), see Figure~\ref{fig:example-ours}. 

\paragraph{Summary.}
Let us summarize the implementation choices for Tabular LIME that we consider in our analysis. 
The $d$ features are considered to be continuous and are discretized (the choice \texttt{discretize\_continuous=True} is default) along $p$ bins. 
The default choice is $p=4$ (\texttt{discretizer=`quartile'} is default), we will sometimes use another value for $p$. 
The bin boundaries $q_{j,b}$ as well as the location and scale parameters for the sampling $\mu_{j,b}$ and $\sigma_{j,b}$ are arbitrary (computed from the appropriate dataset unless otherwise mentioned).  
We consider default weights given by Eq.~\eqref{eq:def-ew-default-weights}, with prescribed bandwidth $\nu$ (\texttt{kernel\_width=nu} is not default) and the surrogate model is obtained by ridge regression with regularization parameter $\lambda=1$. 
As to feature selection, we do not consider any (\texttt{feature\_selection=`none'}). 
The free parameters of the method are the number of bins on each dimension $p$ and the bandwidth $\nu$: we will focus mostly on them. 

\begin{figure} 
\begin{center}
\includegraphics[scale=0.2]{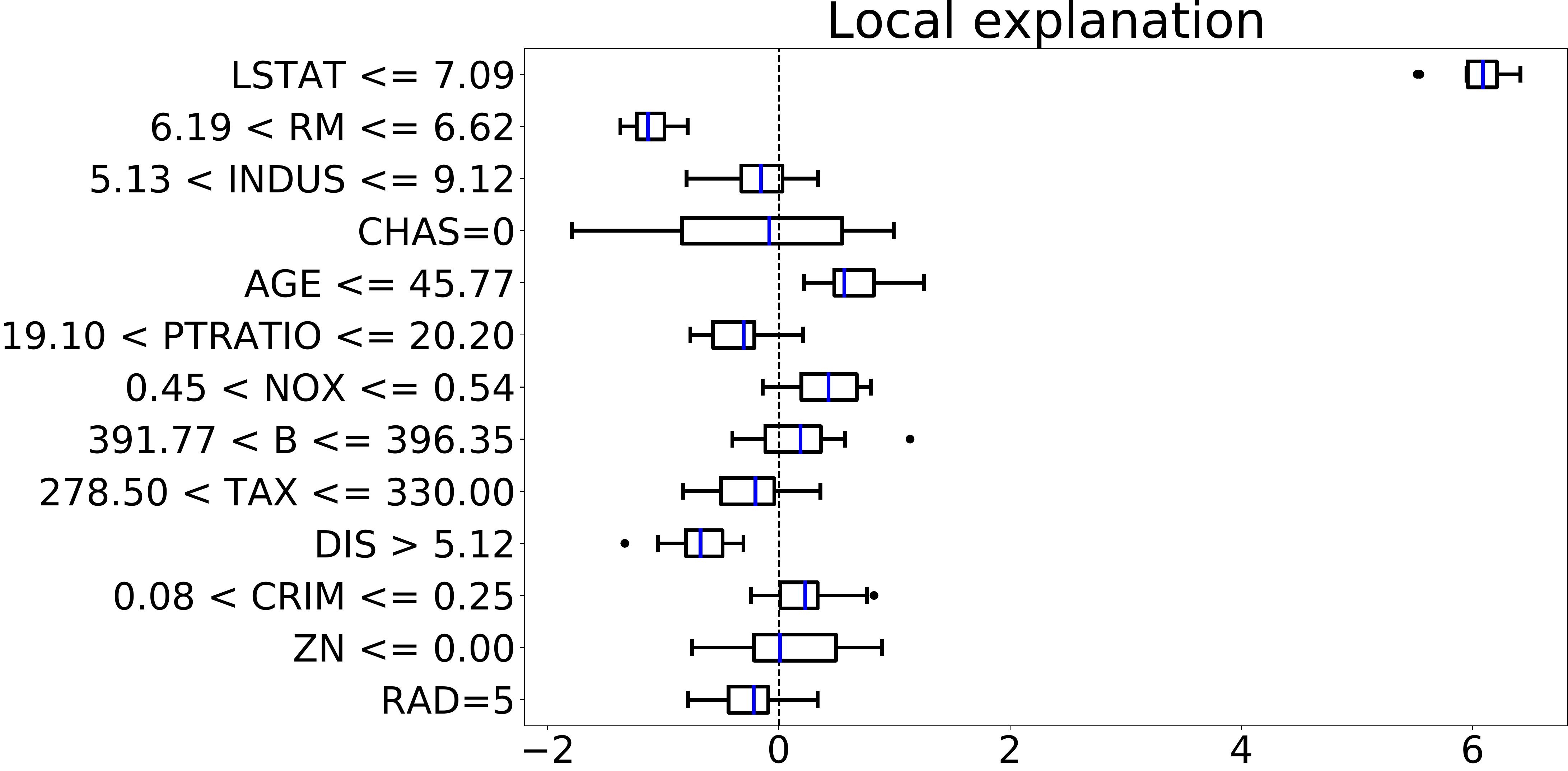}
\end{center}
\caption{\label{fig:example-ours}Example output of Tabular LIME, $100$ repetitions. Since there is randomness in the sampling of the perturbed samples, we run Tabular LIME several time on the same instance and report the whisker boxes associated to these repetitions. In blue the mean over all repetitions. The vertical dotted line marks zero. We also report all the interpretable coefficients values rather than just the top five. This presentation will be standard for the remainder of the article, although we will often omit the bin boundaries when considering synthetic data.}
\end{figure}


\section{Main result: explicit expression for $\betahat_n$ in the large sample size}
\label{section:main-result}

In this section, we state our main result. 
Recall that $\betahat_n$ is the random vector containing the coefficients of the surrogate linear model given by Tabular LIME for a given model $f$ and example $\xi$. 
In a nutshell, when the number of new samples~$n$ is large, $\betahat_n$ concentrates around a vector $\beta$, for which we provide an explicit expression. 
This explicit expression depends on~$f$ and~$\xi$ (\emph{via} $\bxi$), as well as the parameters of Tabular LIME introduced in the previous section (the bandwidth~$\nu$, the number of bins $p$, the bins boundaries~$q$, and the mean and standard deviation on each bin~$\mu$ and~$\sigma$).
Since $\xi$ and the algorithm parameters are fixed, we only emphasize the dependency in $f$ (which will vary from section to section) and we write $\beta^f$. 

Our main assumptions going into this analysis are the following:
\begin{assumption}[Bounded model]
\label{assump:bounded}
The model, $f$, is bounded. In other words, there exists a positive constant $\boundedcst$ such that $\abs{f(x)}\leq \boundedcst$ for all $x\in \supp$, where $\supp=\prod_{j=1}^{d}[q_{j,0},q_{j,p}]$ (see Section~\ref{section:sampling-procedure}).
\end{assumption}
\begin{assumption}[No regularization]
\label{assump:no-regularization}
There is no regularization in the weighted ridge regression problem of Eq.~\eqref{eq:surrogate-model-general}, thus reducing to weighted least-squares. 
In other words, we consider that $\Omega = 0$.
\end{assumption}
We discuss both these assumptions after the statement of the theorem. 

We state our main result in Section~\ref{section:statement-main-result} and present some immediate consequences in Section~\ref{section:consequences-main-result}. 
We then present a brief outline of the proof in Section~\ref{section:proof-outline}. 


\subsection{Explicit expression for $\betahat_n$ in the large sample size}
\label{section:statement-main-result}

In order to make our result precise, in particular to define the vector~$\beta^f$, we need to introduce further notation. 

Recall that $p$ is the fixed number of bins along each dimension and $\nu >0$ is the bandwidth parameter for the weights are the main free parameters of Tabular LIME.
A key normalization constant in our computations is given by 
\begin{equation}
\label{eq:def-little-c}
\littlec \defeq \frac{1}{p} + \left(1-\frac{1}{p}\right)\exps{\frac{-1}{2\nu^2}}
\, .
\end{equation}
We will also denote by $x\in\Reals$ a random variable that has the same law as the $x_i$s (recall that they are i.i.d.\ random variables by construction). 
In the same fashion, we denote by $b\in\{1,\ldots,p\}^d$ the corresponding bin indices, $z\in\{0,1\}^d$ the binary features, and $\pi\in\Reals_+$ the weights. 
These quantities are random variables related to~$x$ in the same way that $b_i$, $z_i$, and $\pi_i$ are related to $x_i$ for $1\leq i\leq n$. 
All expectations in the following are taken with respect to the random variable~$x$. 

We are now armed with enough notation to state our main result, a direct consequence of Theorem~\ref{th:betahat-concentration-general-f-general-weights} and Corollary~\ref{cor:beta-computation-general-f-default-weights} in the Appendix. 

\begin{mytheorem}[Explicit expression for $\betahat_n$ in the large sample size]
	\label{th:betahat-concentration-general-f-default-weights}
	Assume that Tabular LIME operates with the default weights (given by Eq.~\eqref{eq:def-weights-default}).  
Assume that Assumption~\ref{assump:bounded} and~\ref{assump:no-regularization} hold, that is, assume that $f$ is bounded by a constant $\boundedcst$ and take $\Omega=0$ in Eq.~\eqref{eq:surrogate-model-general}. 
	Set $\epsilon \in (0,\boundedcst)$ and $\eta\in [0,1]$. 
For any example $\xi\in \supp$ for which we want to create an explanation, define $\beta^f\in\Reals^{d+1}$ such that, for any $1\leq j\leq d$, 
	\begin{equation}
	\label{eq:coefs-general-f-default-weights}
	\beta^f_j \defeq \littlec^{-d} \left\{\frac{-p\littlec}{p\littlec - 1}\expec{\pi f(x)} + \frac{p^2\littlec^2}{p\littlec - 1}\expec{\pi z_jf(x)} \right\}
	\, ,
	\end{equation}
and 
	\begin{equation}
	\label{eq:intercept-general-f-default-weights}
	\beta^f_0 \defeq \littlec^{-d}\left\{ \left( 1+\frac{d}{p\littlec -1 }\right)\expec{\pi f(x)} - \frac{p\littlec}{p\littlec - 1}\sum_{j=1}^d \expec{\pi z_j f(x)}\right\}
	\, .
\end{equation}
Then, for every 
\[
n\geq \max \biggl\{\frac{2^{12}\boundedcst d^4p^4\exps{\frac{1}{\nu^2}}\log \frac{8d}{\eta}}{\littlec^{2d}\epsilon^2},\frac{2^{15}\boundedcst^2d^7p^8\exps{\frac{2}{\nu^2}} \log \frac{8d}{\eta}}{\littlec^{4d}\epsilon^2}\biggr\}
\, ,
\]
	we have $\smallproba{\smallnorm{\betahat_n - \beta^f} \geq \epsilon} \leq \eta$. 
\end{mytheorem}

In practice, Tabular LIME only relies on the interpretable coefficients for $1\leq j\leq d$. 
We nonetheless provide the intercept for completeness' sake.

\paragraph{Direct consequences.}
Intuitively, Theorem~\ref{th:betahat-concentration-general-f-default-weights} states that:
\begin{itemize}
	\item if the number of perturbed samples is large enough, the explanations provided by Tabular LIME for any $f$ at a given example $\xi$ stabilize around a fixed value, $\beta^f$;
	\item this $\beta^f$ has an explicit expression, simple enough that we can hope to use it in order to answer the question we asked in the introduction.
\end{itemize}
In particular, for reasonably large $n$, \textbf{we can focus on $\beta^f$ in order to gain insight on the explanations} provided by Tabular LIME with default settings. 
This will be our agenda in Section~\ref{section:discussion}, where we will assume specific structures for $f$. 
Theorem~\ref{th:betahat-concentration-general-f-default-weights} also hints that, from a practical standpoint, the empirical values $(\betahat_n)_j$ concentrate around a given value for large $n$. 
Thus one is encouraged to take $n$ as large as possible to avoid instability of the explanations. 

\paragraph{Large bandwidth behavior.}
Let us assume for a moment that $\littlec=1$ and that $\pi=1$ almost surely (we will see in Section~\ref{section:bandwidth} that this happens when $\nu \to +\infty$). 
Notice that, by definition of $z_j$, $\expec{z_jf(x)}=\frac{1}{p}\smallcondexpec{f(x)}{b_j=\bxi_j}$. 
Thus, for any $1\leq j\leq d$, the interpretable coefficients take the simple expression 
\begin{equation}
\label{eq:simple-beta}
\beta_j^f = \frac{p}{p-1}\left(\condexpec{f(x)}{b_j=\bxi_j} - \expec{f(x)}\right)
\, .
\end{equation}
Up to numerical constants, $\beta_j^f=\smallcondexpec{f(x)}{b_j=\bxi_j}$. 
In other words, $\beta_j^f$ is the expected value of the model conditioned to $x_j$ falling into the same bin as $\xi_j$. 
Intuitively, $\beta_j^f$ is high (resp. low) if $f$ takes larger (resp. smaller) than average values on the $d$-dimensional bins above~$\xi_j$. 
The general picture is of course more complicated, since we ignored completely the role of~$\nu$ and~$p$ in this discussion. 

\paragraph{Discussing the assumptions.}
It is remarkable that Theorem~\ref{th:betahat-concentration-general-f-default-weights} is a statement about the default implementation of Tabular LIME as is, the only difference being $\Omega=0$ (Assumption~\ref{assump:no-regularization}).
Moreover, Theorem~\ref{th:betahat-concentration-general-f-default-weights} is true under pretty mild assumptions on the model to explain. 
Essentially, we only require~$f$ to be bounded on the bins (Assumption~\ref{assump:bounded}). 
If these bins are computed from a finite training set $\traindata$, then $\supp$ is compact and we are essentially requiring that $f$ be well-defined on $\supp$, which is virtually always the case for most machine learning models. 
An important assumption that can be easily overlooked is that $\xi$ should lie in $\supp=\prod_j [q_{j,0},q_{j,p}]$ for the theorem to hold. 
In particular, Theorem~\ref{th:betahat-concentration-general-f-default-weights} does not say anything about examples to explain that do not belong to the bins provided to the algorithm. 
We leave the analysis of Tabular LIME explanations for $\xi\notin \supp$ to future work. 

\paragraph{Limitations.}
The strongest limitation of Theorem~\ref{th:betahat-concentration-general-f-default-weights} is the poor concentration for small bandwidths. 
Indeed, for small $\nu$ (say $\nu<1$), according to Theorem~\ref{th:betahat-concentration-general-f-default-weights}, one has to take $n$ polynomial in $d$ in order to witness the convergence. 
This usually means that it is quite hard to use Theorem~\ref{th:betahat-concentration-general-f-default-weights} in practice for $\nu<1$ as soon as the dimension is greater than $10$. 
Moreover, in that case, the effect of regularization start to come into play (see the discussion in Section~\ref{section:surrogate-model}). 
However, the default bandwidth is $\sqrt{0.75d}\gg 1$ as soon as $d\approx 10$, and Theorem~\ref{th:betahat-concentration-general-f-default-weights} is quite satisfactory to study LIME in its normal use.


\subsection{Further consequences: properties of LIME that can be deduced from the explicit expression}
\label{section:consequences-main-result}

We now present some additional consequences of Theorem~\ref{th:betahat-concentration-general-f-default-weights}, which are true without assuming anything on~$f$ other than boundedness. 

\subsubsection{Linearity of explanations}
\label{sec:explanation-linearity}
We first notice that the vector $\beta^f$ \textbf{depends linearly on $f$}. 
Indeed, a careful reading of Eqs.~\eqref{eq:intercept-general-f-default-weights} and \eqref{eq:coefs-general-f-default-weights} reveals that~$\beta^f$ depends on~$f$ only through the expectations $\expec{\pi f(x)}$ and $\expec{\pi z_j f(x)}$. 
Since~$\pi$ and~$z$ do not depend on~$f$, by linearity of the expectation, we see that 
\begin{equation}
\label{eq:linearity-limit-explanation}
\beta^{f+g}=\beta^f+\beta^g
\, .
\end{equation}
Up to noise coming from the sampling of the $n$ perturbed examples and ignoring the effect of regularization, we can conclude that $\betahat_n^{f+g} \approx \betahat_n^f + \betahat_n^g$.
Therefore the \emph{linearity} axiom of \citet{Sundararajan_Najmi_2020} is approximately satisfied. 
We give a simple example of this property in Figure~\ref{fig:linearity-of-explanations}. 

From a theoretical standpoint, the main consequence of Eq.~\eqref{eq:linearity-limit-explanation} is the following:
we will soon specialize Theorem~\ref{th:betahat-concentration-general-f-default-weights} to more explicit models~$f$. 
Quite a number of models can be written in an additive form (think of a generalized additive model, a kernel regressor, or a random forest). 
Linearity will allow us to focus on the \emph{building bricks} of these models (a kernel function or an indicator function).

\begin{figure}[ht!]
    \centering
    \includegraphics[scale=0.12]{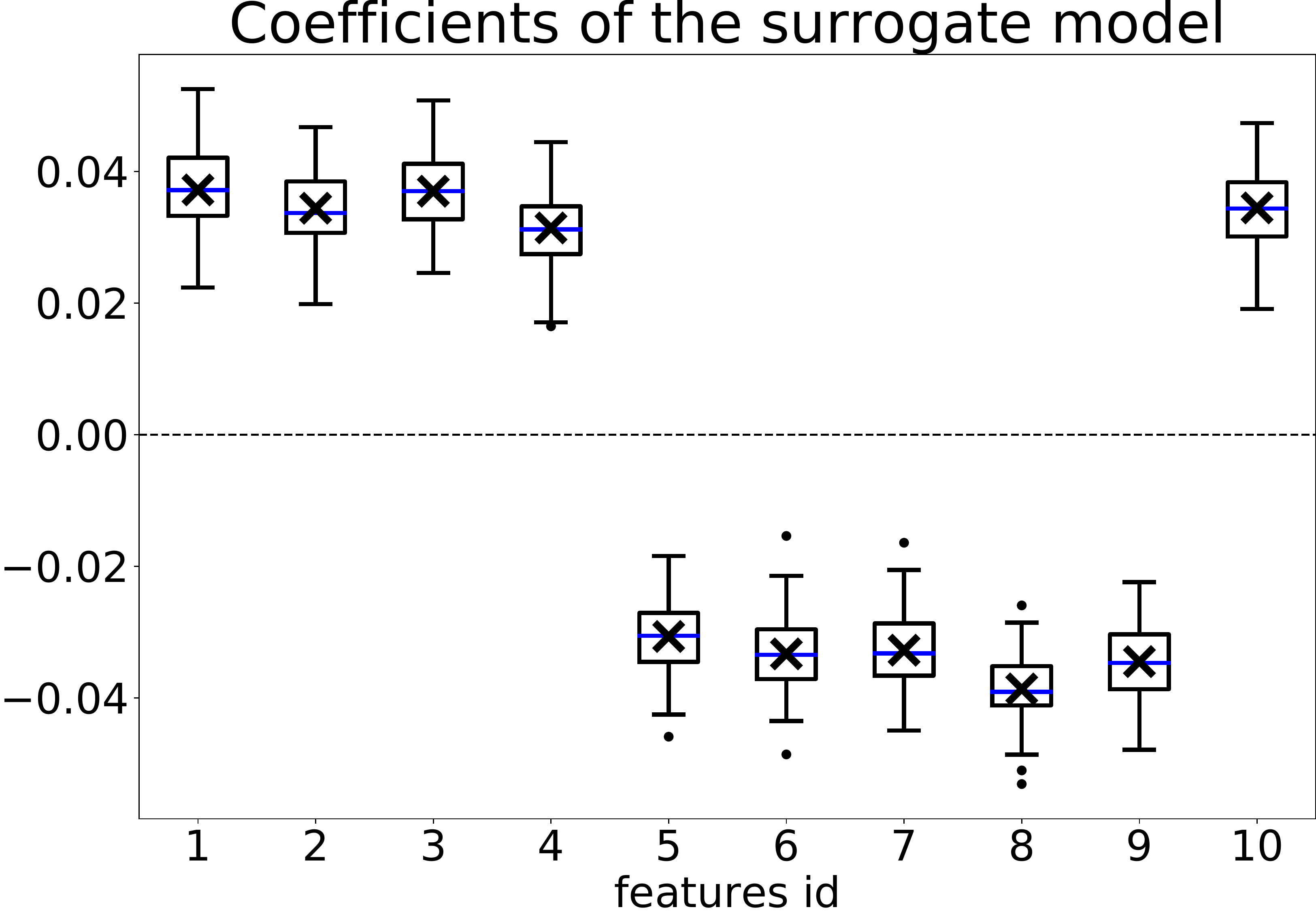} 
    \raisebox{0.87cm}{\includegraphics[scale=1]{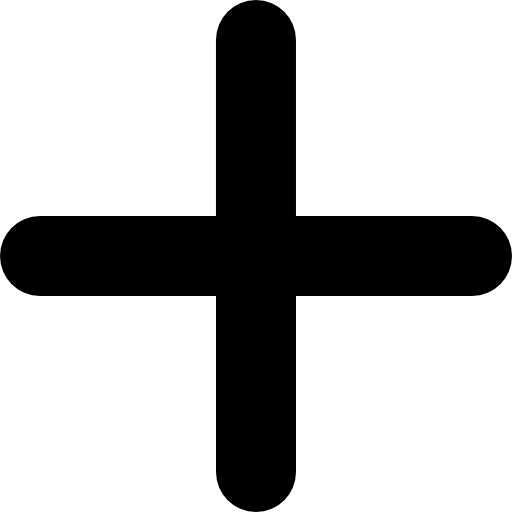}}
    \includegraphics[scale=0.12]{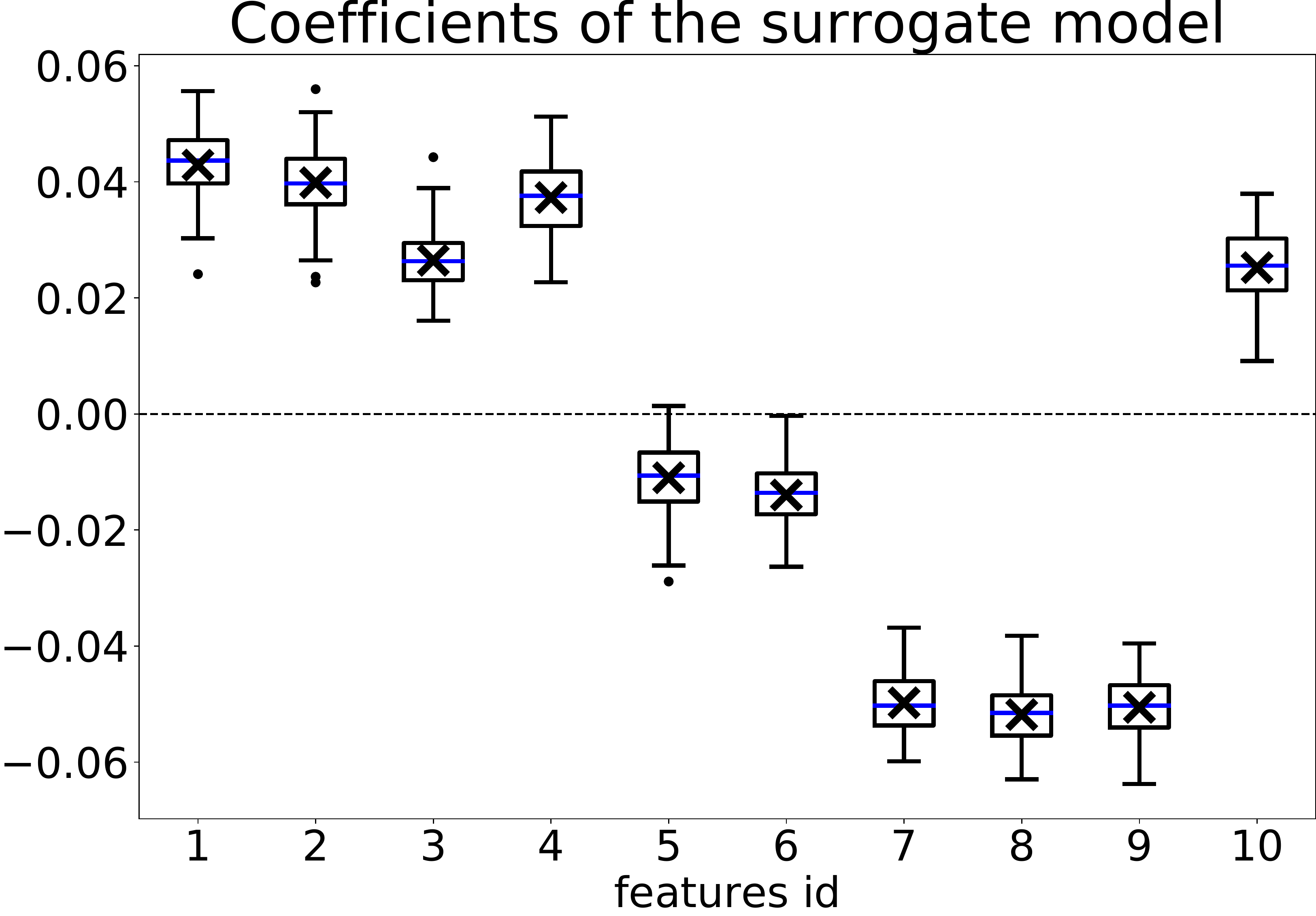}
    \raisebox{0.95cm}{\includegraphics[scale=1]{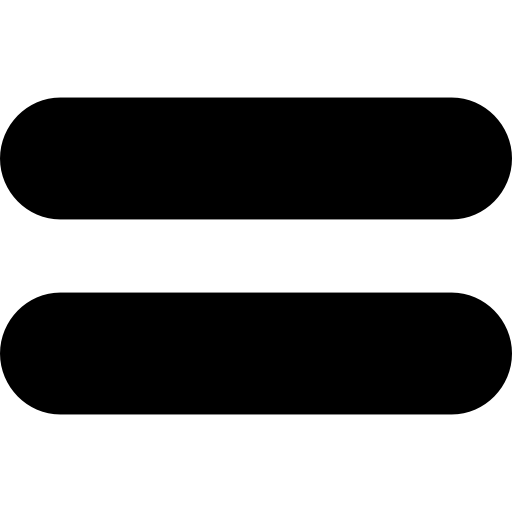}}
    \includegraphics[scale=0.12]{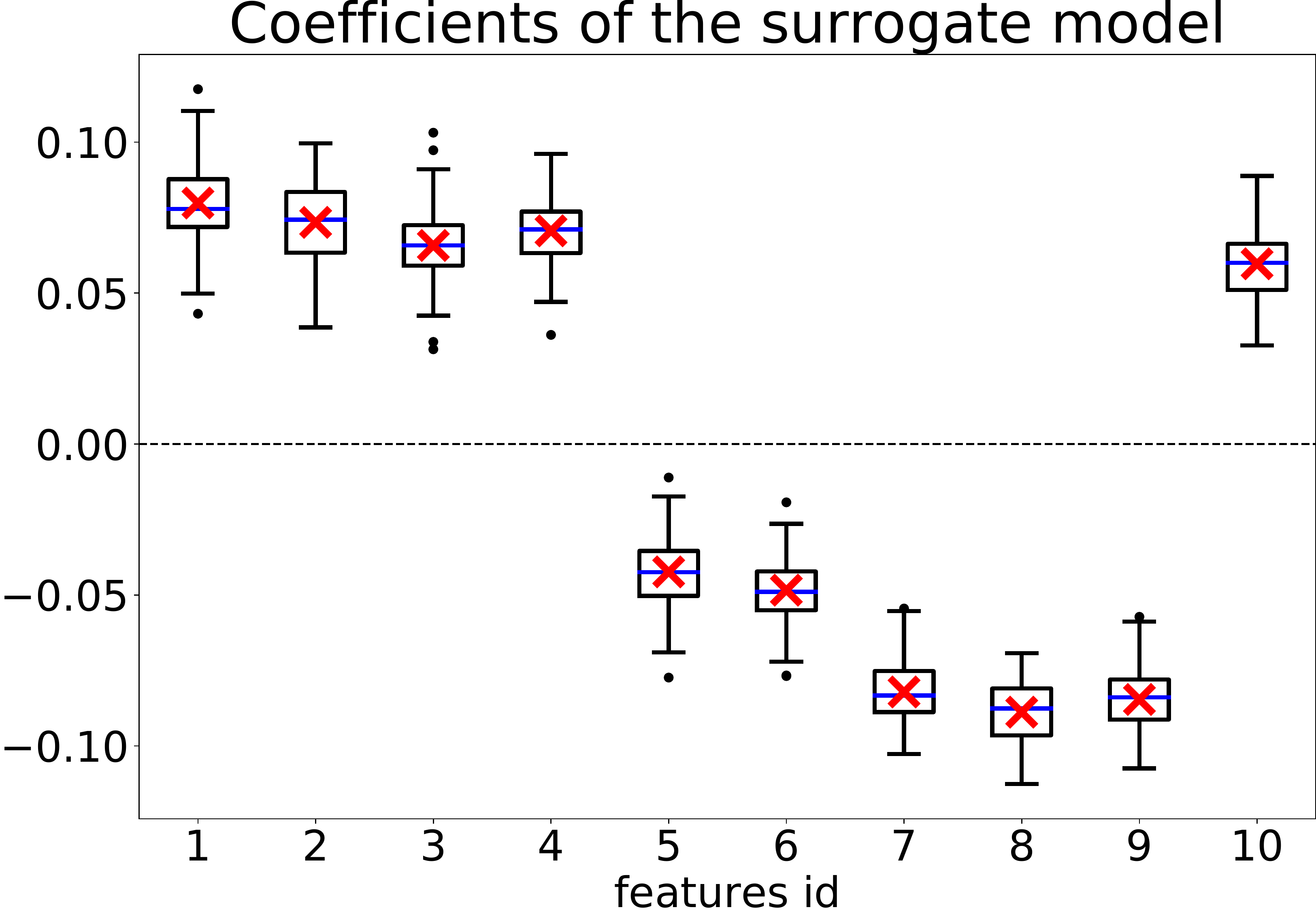}
    \caption{\label{fig:linearity-of-explanations}Linearity of explanations. In this experiment, we compute the explanations for two radial basis kernel functions $f_1$ and $f_2$. The explanation provided by Tabular LIME for $f_1$ (resp. $f_2$) is displayed on the left panel (resp. middle), with empirical mean of the explanations denoted by a black cross. The explanation for $f_1+f_2$ is displayed in the right panel, with the sum of the mean explanations marked in red. As predicted, $\betahat^{f_1} + \betahat^{f_2}=\betahat^{f_1+f_2}$, up to the randomness due to the sampling scheme.}
\end{figure}

From a more practical standpoint, let us assume that our knowledge of $f$ is imperfect. 
More precisely, let us split the function to explain in two parts: (i) the part coming from the black-box model~$f$ by itself, and (ii) the part coming from small perturbations such as numerical errors or measurement noise. 
Linearity allows us to focus on the perturbation part separately, and prove the following:

\begin{myproposition}[Robustness of the explanations given by Tabular LIME]
\label{prop:explanation-stability-default-weights}
Suppose that $\xi\in\supp$. 
Consider $f$ and $g$ two functions that are bounded on $\supp$. 
Then, under the assumptions of Theorem~\ref{th:betahat-concentration-general-f-default-weights}, 
\[
\smallnorm{\beta^f - \beta^g} \leq \frac{\sqrt{d(9d+4p^2)}\exps{\frac{1}{2\nu^2}}}{p-1}\infnorm{f-g}
\, .
\]
\end{myproposition}

As expected, small perturbations of the function to explain do not perturb the explanations too much, which is a desirable property. 
Proposition~\ref{prop:explanation-stability-default-weights} is proven in Appendix~\ref{section:proof-stability-corollary}. 


\subsubsection{Explanations only depend on the bin indices of $\xi$}
\label{sec:bin-indices}

The interpretable coefficients $\beta^f_j$ depend only on the bin indices $\bxi_j$ of the example $\xi$. 
Indeed, Eqs.~\eqref{eq:intercept-general-f-default-weights} and \eqref{eq:coefs-general-f-default-weights} reveal that only the sampling of $x$ depends on the actual coordinates of~$\xi$. 
But if we recall Section~\ref{section:sampling-procedure}, this sampling only depends on the bin indices of~$\xi$. 
Therefore, \textbf{Tabular LIME provides the same explanation for any two instances falling in the same bins along each dimension}, up to some noise coming from the sampling procedure. 
See Figure~\ref{fig:explanations-bins} for an illustration of this phenomenon. 
In a sense, this behavior gives a certain stability to the explanations provided by Tabular LIME: if two examples to explain $\xi$ and $\xi'$ are very close, they are likely to have the same bin indices, and therefore the same $\beta^f$. 
On the other hand, if $\xi$ and $\xi'$ are close but do not have the same bin indices, the explanations are likely to be quite different. 
This could be an explanation for the instability of explanations observed by \citet{Alvarez_Jaakola_2018}. 

An interesting consequence is the lack of faithfulness to the true model. 
Let us fix a hyperrectangle~$R$ given by the discretization and consider all $\xi\in R$. 
In the general case, $f$ is not constant on $R$. 
But on the other hand, since all surrogate models obtained inside $R$ are the same, their values at $\xi$ cannot be the same as $f(\xi)$ for all $\xi$. 
The \emph{local accuracy} property of \citet{Lundberg_Lee_2017} (also called \emph{efficiency} by \citet{Sundararajan_Najmi_2020}) is thus not satisfied.

From a practical point of view, if one believes that the exact value of $\xi$ matters, then increasing $p$ is a solution, leading to thinner bins along each dimension. 
At the limit, taking $p=0$ (no discretization) is also an option. 

\begin{figure}
\begin{center}
\includegraphics[scale=0.18]{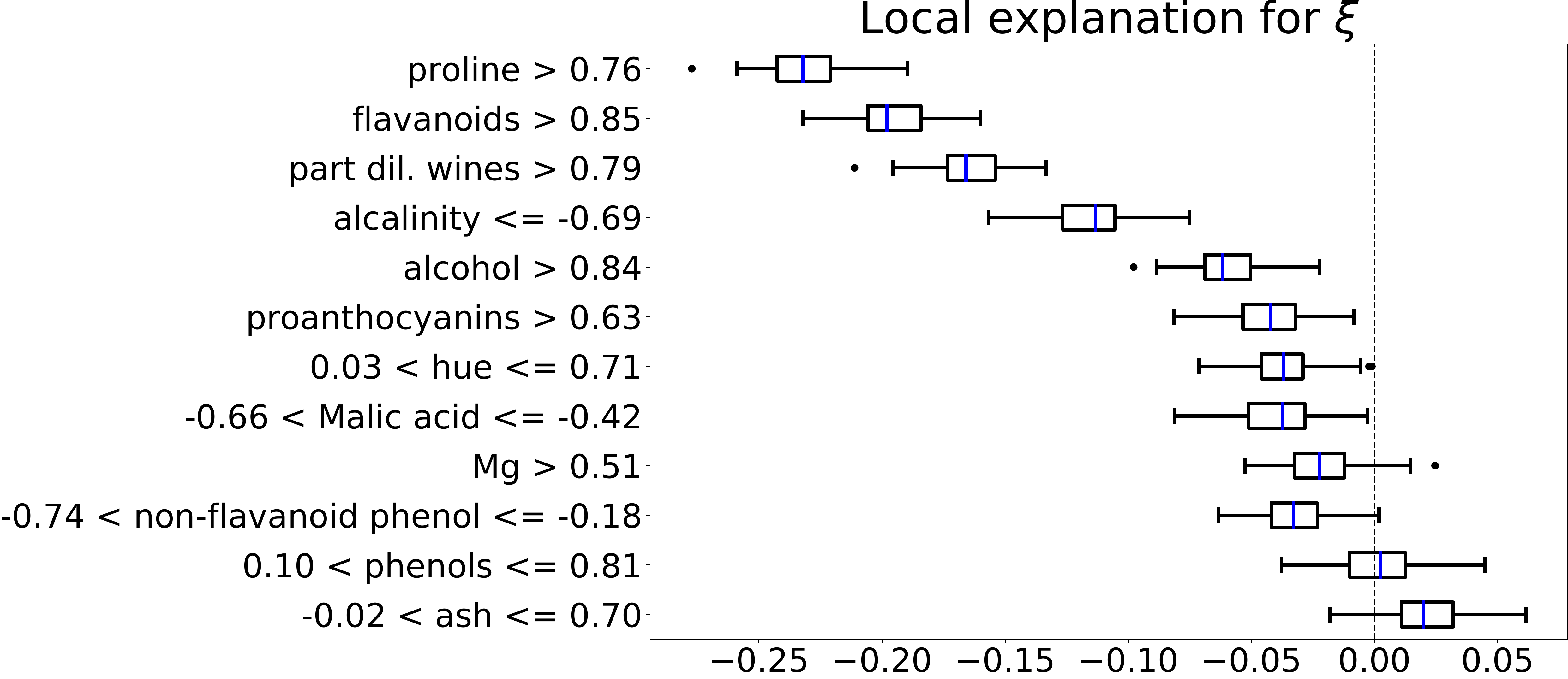}
\hspace{0.1cm}
\includegraphics[scale=0.18]{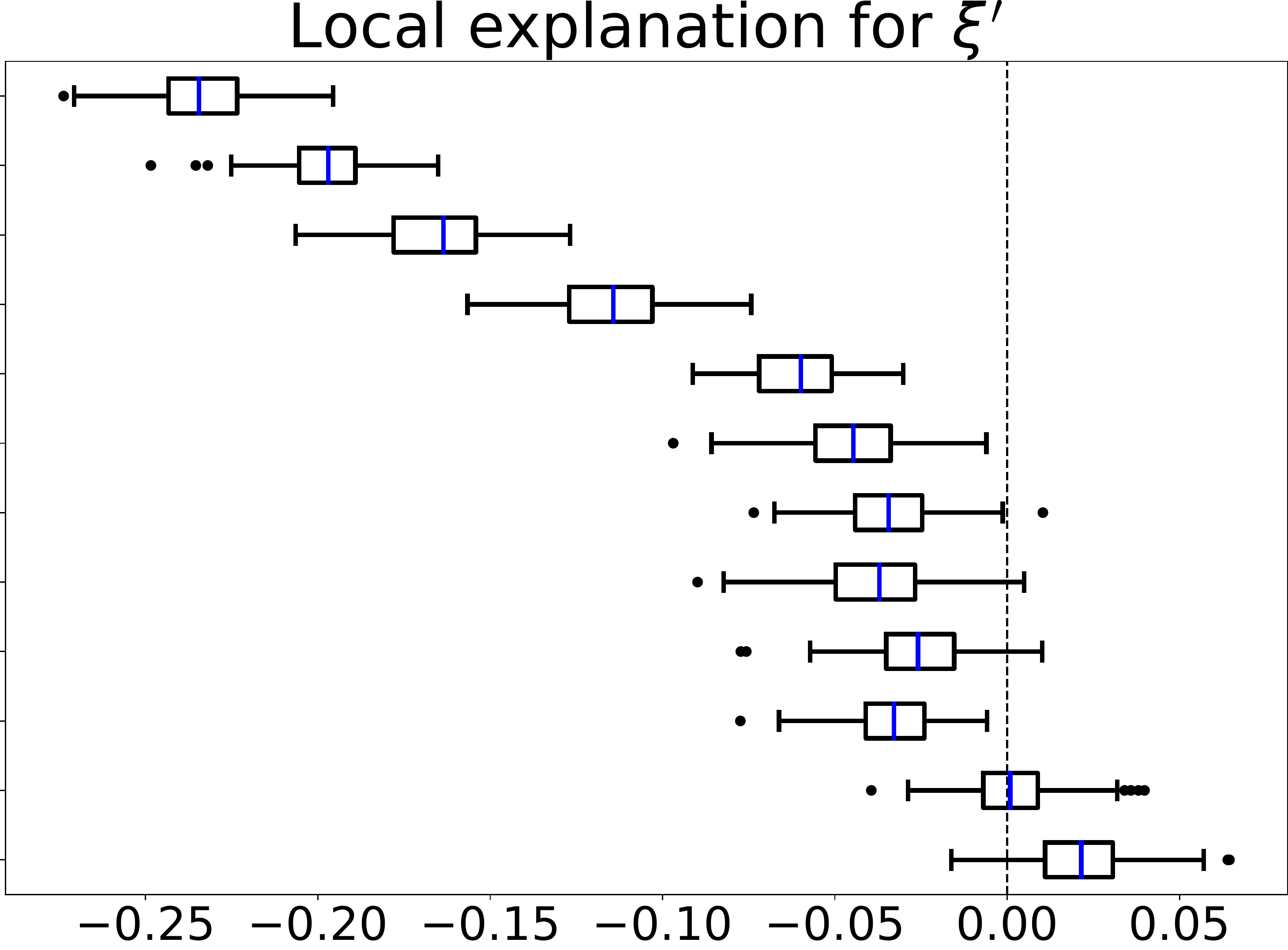}
\end{center}
\caption{\label{fig:explanations-bins}Explanations given by Tabular LIME for a kernel ridge regressor trained on the Wine dataset \citep{Cortez_et_al_1998} with normalized features. \emph{Left panel:} Explanations for a given $\xi$ with $1000$ perturbed samples. \emph{Right panel:} We recovered the bins boundaries as well as the bin indices of $\xi$. We then sampled $\xi'$ in the same $d$-dimensional box as $\xi$. The local explanation for~$\xi'$ are indistinguishable from those for~$\xi$ once randomness of the sampling procedure is taken into account.}
\end{figure}


\subsubsection{Dependency on the bandwidth parameter $\nu$}
\label{section:bandwidth}

\paragraph{Large bandwidth behavior.}
Suppose that the bandwidth is large, that is, $\nu\to +\infty$. 
In that case, it is clear from their definitions that $\littlec\to 1$ and $\pi_i \to 1$ almost surely. 
By dominated convergence, Eq.~\eqref{eq:intercept-general-f-default-weights} and~\eqref{eq:coefs-general-f-default-weights} imply that, for any $1\leq j\leq d$, 
\begin{equation}
\label{eq:coefs-limit-large-bandwidth}
\beta^f_j \longrightarrow \frac{-p}{p - 1}\expec{f(x)} + \frac{p^2}{p - 1}\expec{ z_{j}f(x)}
\, .
\end{equation}
The intercept satisfies
\begin{equation}
	\label{eq:intercept-limit-large-bandwidth}
	\beta^f_0 \longrightarrow \left(1 + \frac{d}{p - 1}\right)\expec{f(x)} -\frac{p}{p - 1}\sum_{j=1}^{d}\expec{z_{j}f(x)}
	\, .
\end{equation}
We show this convergence phenomenon in Figure~\ref{fig:bandwidth-cancellation}. 
In cases where the bandwidth choice of the default implementation $\nu=\sqrt{0.75 d}$ is large, this approximation is well-satisfied. 
In fact, the bandwidth parameter then becomes redundant: it is equivalent to give weight~$1$ to every perturbed sample. 

\paragraph{Small bandwidth behavior.}
On the other hand, when $\nu\to 0$, it is straightforward to show that $\beta^f_j\to 0$ for any $1\leq j\leq d$. 
In between these two extremes, the behavior of the interpretable coefficients is not trivial. 
In particular, the sign of interpretable coefficients can change when the bandwidth varies. 
{\bf This is a worrying phenomenon: for some values of the bandwidth $\nu$, the explanation provided by Tabular LIME is negative, while it becomes positive for other choices.} 
In the first case, the trusting user of Tabular LIME would grant a positive influence for the parameter, in contrast to a negative influence in the second case. 
This is not only a theoretical worry: if the values of these interpretable coefficients are large enough, they may very well be among the top coefficients that are usually relevant for the user. 
We demonstrate this effect in Figure~\ref{fig:bandwidth-cancellation} on a subset of the interpretable coefficients, the remainder can be found in Section~\ref{sec:additional-experiments} of the Appendix.

\paragraph{Practical considerations.}
In practice, the previous discussion suggests to take large values of the bandwidth. 
By doing so, we make sure that we are far away from the small bandwidth regime where all coefficients vanish and where the non-linear behavior seems to occur. 
If we come back to Eq.~\eqref{eq:def-weights-default}, this means taking $\nu^2$ comparable to $\norm{\Indic - z_i}^2$. 
We think that the default choice is satisfying, since $\nu^2=\bigo{d}$, which is the maximum number of terms in $\norm{\Indic - z_i}^2$. 

\begin{figure}
	\begin{center}
\includegraphics[scale=0.3]{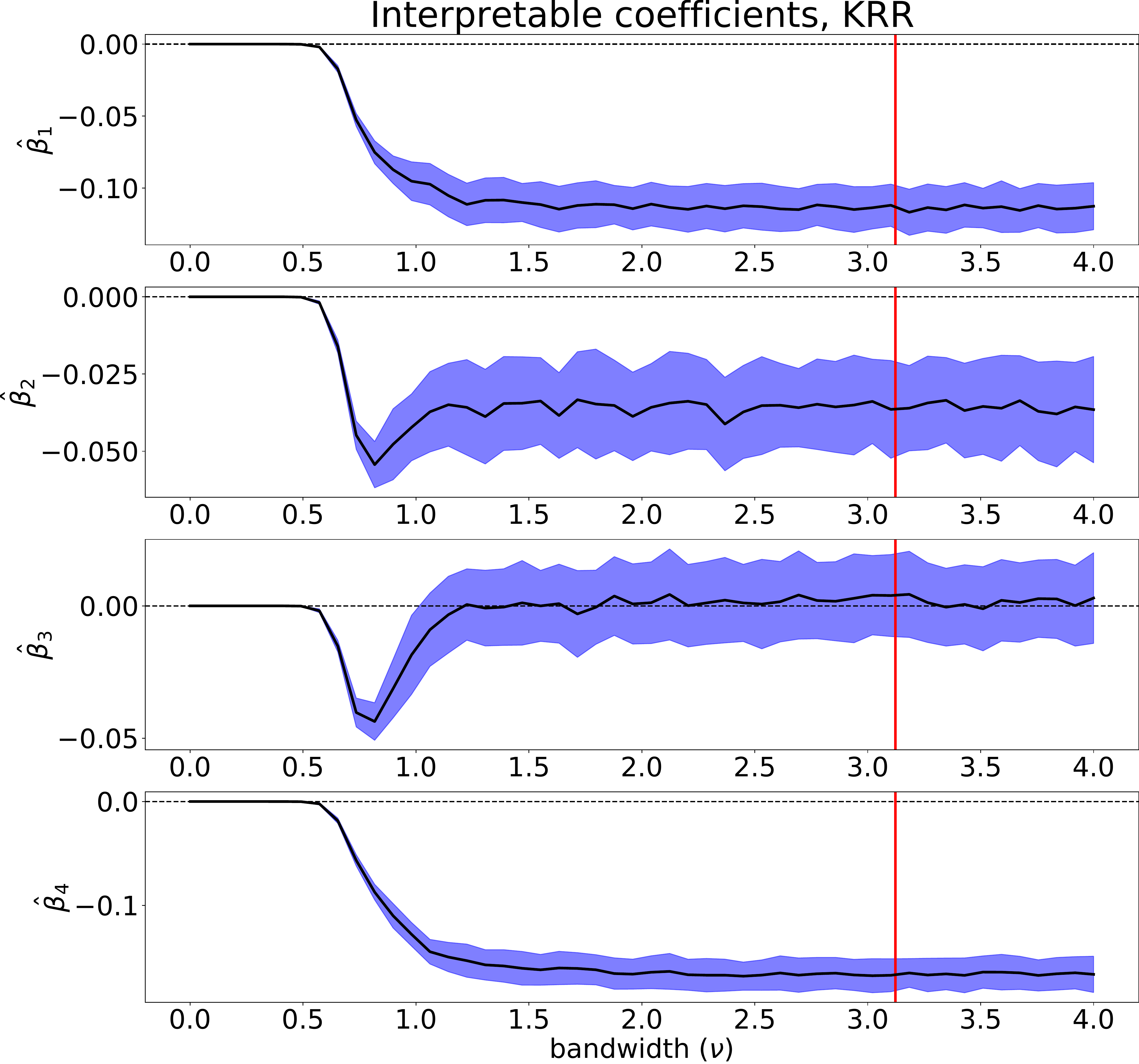}
	\end{center}
	\caption{\label{fig:bandwidth-cancellation}Tabular LIME on a kernel regressor trained on the Wine dataset, $100$ repetitions, $1000$ samples, plotting the interpretable coefficient for the first four features for varying bandwidth. 
The vertical red line marks the default choice for $\nu$, given by $\sqrt{0.75 d}$. 
We see that $\beta^f_j=0$ when $\nu=0$ as predicted, while $\betahat^f_j$ stabilizes when $\nu\to +\infty$. We also note that the sign of the interpretable coefficient can change as a function of the bandwidth (see $\betahat_3$). }
\end{figure}


\subsection{Outline of the proof of Theorem~\ref{th:betahat-concentration-general-f-default-weights}}
\label{section:proof-outline}

Before turning to specializations of Theorem~\ref{th:betahat-concentration-general-f-default-weights} to specific classes of models, we provide a short outline of the proof (the complete proof is provided in appendix). 

Since we restrict our analysis to $\Omega=0$,  Eq.~\eqref{eq:surrogate-model-general} becomes a simple weighted least-squares problem. 
Let us collect the weights $\pi_i$ in a diagonal matrix $W\in\Reals^{n\times n}$ such that $W_{i,i}=\pi_i$ for all $i\in [n]$. 
We also define $y\in\Reals^n$ coordinate-wise by $y_i\defeq f(x_i)$. 
Then the solution of Eq.~\eqref{eq:surrogate-model-general} is given in closed-form by
\[
\betahat_n = (Z^\top W Z)^{-1}Z^\top W y
\, .
\]
Let us set $\Sigmahat_n \defeq \frac{1}{n}Z^\top W Z$ and $\Gammahat_n \defeq \frac{1}{n}Z^\top W y$ and notice that $\betahat_n = \Sigmahat_n^{-1}\Gammahat_n$. 
Elementary computations show that both $\Sigmahat_n$ and $\Gammahat_n$ can be written as empirical averages of expressions depending on the perturbed samples~$x_i$. 
Since the sampling of the $x_i$s is i.i.d., the weak law of large numbers guarantees that 
\[
\Sigmahat_n \cvproba \Sigma 
\quad \text{and}\quad 
\Gammahat_n \cvproba \Gamma^f
\, ,
\]
where we defined $\Sigma\defeq \smallexpec{\Sigmahat_n}$ and $\Gamma^f\defeq \smallexpec{\Gammahat_n}$. 
These quantities are given by Proposition~\ref{prop:sigma-computation} and Eq.~\eqref{eq:computation-gamma-general}. 
Since $\Sigma$ is invertible in closed-form for a large class of weights (Proposition~\ref{prop:computation-inverse-sigma}), we can set $\beta^f \defeq \Sigma^{-1}\Gamma^f$ and proceed to show that $\betahat_n$ is close from $\beta^f$ in probability, whose expression is given by Corollary~\ref{cor:beta-computation-general-f-default-weights}.  
A prerequisite to the concentration of $\betahat_n$ are the concentration of $\Sigmahat_n$ (Proposition~\ref{prop:sigmahat-concentration-general-weights}) and the concentration of $\Gammahat_n$ (Proposition~\ref{prop:gammahat-concentration}). 
Together with the control of $\opnorm{\Sigma^{-1}}$ (Proposition~\ref{prop:upper-bound-operator-norm}), a binding lemma (Lemma~\ref{lemma:binding-lemma}) allows us to put everything together in Appendix~\ref{section:concentration-betahat}.


\section{Special cases of Theorem~\ref{th:betahat-concentration-general-f-default-weights}}
\label{section:discussion}

We now specialize Theorem~\ref{th:betahat-concentration-general-f-default-weights} in three situations where $f$ has some additional structure, allowing us to go further in the computation of $\beta^f$. 
Namely, we will assume that~$f$ is an \textbf{additive function} over the coordinates (Section~\ref{section:additive}), \textbf{multiplicative} along the coordinates (Section~\ref{section:multiplicative}), and finally a function \textbf{depending only on a subset of the coordinates} (Section~\ref{section:sparse}). 
Our goal is to show how to use Theorem~\ref{th:betahat-concentration-general-f-default-weights} in specific cases and how to gain insights into the explanations provided by Tabular LIME in these cases. 

Before being able to complete this program, we need to introduce some additional notation. 
In order to use Theorem~\ref{th:betahat-concentration-general-f-default-weights} for models with additional structure, we will need to make the expectations computations more explicit. 
Recall that~$x$ corresponds to the sampling of the perturbed examples (see Section~\ref{section:sampling-procedure}). 
Often, we will need to compute the expectation of the product of the weight and a function of~$x$ conditionally to $x$ belonging to a certain bin along dimension $j$. 
For this reason, we define
\begin{equation}
\label{eq:def-ew-default-weights}
\forall 1\leq j\leq d, \enspace \forall 1\leq b\leq p, \quad \ew_{j,b}^\psi \defeq \condexpec{\exps{\frac{-(1-z_j)^2}{2\nu^2}}\psi(x_j)}{b_j = b}
\, ,
\end{equation}
where $\psi : \Reals \to\Reals$ is any integrable function. 
A typical use case of this definition will be $\psi(t) = 1$ for all $t\in\Reals$, in which case we shorten $\ew_{j,b}^1$ to $\ew_{j,b}$. 
Another frequent use will be when $\psi$ is the identity mapping, in which event we will write $\ew_{j,b}^\times$ instead of $\ew_{j,b}^\id$. 

Even though this can look cumbersome, this notation is quite helpful. 
The reason is quite simple: whenever we need to compute an expectation with respect to $x$, we will use the law of total expectation with respect to the events $\{b_j = b\}$ for $b\in\{1,\ldots,p\}$. 
The idea is to ``cut'' the expectation depending on which bins $x$ falls into on each dimension. 

For any $1\leq j\leq d$, we also define the constants
\begin{equation}
\label{eq:def-little-cj-psi}
\littlec_j^\psi \defeq \frac{1}{p}\sum_{b=1}^{p} \ew_{j,b}^\psi
\, ,
\end{equation}
the average value of the $\ew_{j,b}^\psi$ coefficients over all the bins for a given feature.  
When $\psi$ is identically equal to one, we also shorten $\littlec_j^\psi$ to $\littlec_{j}$ and  set $\bigc\defeq \prod_{j=1}^d \littlec_j$. 
They can be seen as a generalization of $\littlec$ and $\bigc$, the normalization constants encountered in Theorem~\ref{th:betahat-concentration-general-f-default-weights}.  

\paragraph{First computations. }
In the default implementation of Tabular LIME, if $b=\bxi_j$, then $z_j=1$, and $0$ otherwise. 
Thus the computation of $\ew_{j,b}$ is straightforward in this case:
\[
\ew_{j,b} = 
\begin{cases}
1 &\text{ if }b = \bxi_j \\
\exps{\frac{-1}{2\nu^2}} & \text{ otherwise.}
\end{cases}
\]
The expression of $\littlec_{j}\defeq \littlec_j^1$ is also quite simple.
In particular, $\littlec_j$ does not depend on $j$ and we find
\begin{equation}
\label{eq:little-cj-special-case}
\forall 1\leq j\leq d,\quad \littlec_j = \littlec \defeq \frac{1}{p} + \left(1-\frac{1}{p}\right)\exps{\frac{-1}{2\nu^2}}
\, ,
\end{equation}
recovering the expression of $\littlec$ given in Section~\ref{section:main-result}.  
As a consequence, note that $\bigc = \littlec^d$. 

We now have all the required tools to specialize Theorem~\ref{th:betahat-concentration-general-f-default-weights} in practical cases.


\subsection{General additive models, including linear functions}
\label{section:additive}

In this section, we consider functions that can be written as a sum of functions where each function depends only on one coordinate. 
Namely, we make the following assumption on $f$:
\begin{assumption}[Additive function]
	\label{ass:additive}
	We say that $f:\Reals^d\to\Reals$ is \emph{additive} if there exist arbitrary functions $f_1,\ldots,f_d : \Reals^d \to \Reals$ such that, for any $x\in\Reals^d$, 
	\[
	f(x) = \sum_{j=1}^{d} f_j(x_j)
	\, .
	\]
\end{assumption}
General additive models generalize \emph{linear models}, where $f_j: x\mapsto f_j\cdot x$ for any $1\leq j\leq d$.  
They were popularized by \citet{Stone_1985} and \citet{Hastie_Tibshirani_1990}. 
We refer to Chapter~9 in \citet{Hastie_Tibshirani_Friedman_2001} for an introduction to general additive models, and in particular Section 9.1.1 regarding the training thereof. 
If $f$ is a general additive model, we can specialize Theorem~\ref{th:betahat-concentration-general-f-default-weights}, and examine the explanation provided by Tabular LIME in this case. 
From now on, rather than giving the concentration result (which remains unchanged), we focus directly on $\beta^f$. 
Recall that the $\ew_{j,b}^{f_j}$ are the conditional expectation of $f_j$ on bin $b$ along dimension $j$ (Eq.~\eqref{eq:def-ew-default-weights}) and that $\littlec_{j}$ and $\littlec_{j}^{f_j}$ are normalization constants (Eqs.~\eqref{eq:def-little-cj-psi}). 

\begin{myproposition}[Computation of $\beta^f$ for an additive function]
	\label{prop:beta-computation-additive-f-default-weights}
	Assume that $f$ satisfies Assumption~\ref{ass:additive} and that the assumptions of Theorem~\ref{th:betahat-concentration-general-f-default-weights} are satisfied (in particular, each $f_j$ is bounded on $[q_{j,0},q_{j,p}]$). 
	Set $\xi\in \supp$. 
	Then Theorem~\ref{th:betahat-concentration-general-f-default-weights} holds and, for any $1\leq j\leq d$, 
	\[
	\beta^f_j = \frac{p\littlec}{p\littlec - 1} \left(\ew_{j,\bxi_j}^{f_j} - \frac{\littlec_j^{f_j}}{\littlec}\right)
	\, .
	\]
Moreover, the intercept is given by
	\[
	\beta^f_0 = \frac{1}{p\littlec-1}\sum_{k=1}^{d} \sum_{b\neq \bxi_k}\ew_{k,b}^{f_k}
	\, .
	\]

\end{myproposition}

The proof of Proposition~\ref{prop:beta-computation-additive-f-default-weights} is a direct consequence of Theorem~\ref{th:betahat-concentration-general-f-default-weights} and Proposition~\ref{prop:beta-computation-additive-f-general-weights}. 
Proposition~\ref{prop:beta-computation-additive-f-default-weights} has several interesting consequences. 

\paragraph{Splitting the coordinates. }
A careful reading of the expression of $\beta^f$ in Proposition~\ref{prop:beta-computation-additive-f-default-weights} reveals that the $j$-th interpretable coefficient $\beta^f_j$ depends only on $f_j$, the part of the function depending on the $j$-th coordinate. 
In other words, \textbf{Tabular LIME splits the explanations coordinate by coordinate for general additive models.}
This property is desirable in our opinion.
Indeed, since our model~$f$ depends on the $j$-th coordinate only through the function $f_j$, then $f_j$ alone should be involved on the part of the explanation which is concerned by~$j$. 

\paragraph{Dummy features. }
Suppose for a moment that~$f$ is additive but does not depend on coordinate~$j$ at all. 
That is, $f_j(x)=\kappa$ for any $x$, where $\kappa$ is a constant. 
Then, by linearity of the conditional expectation, $\ew_{j,b}^{f_j}=\kappa\ew_{j,b}$ for any $b$.
By definition of the normalization constant $\littlec_j^{f_j}$, we have $\littlec_j^{f_j}=\kappa\littlec$. 
Therefore
\[
\ew_{j,\bxi_j}^{f_j} - \frac{\littlec_j^{f_j}}{\littlec} = 0
\, ,
\]
and we deduce immediately that $\beta^f_j=0$. 
In other words, unused coordinates are ignored. 
\textbf{For an additive $f$, the dummy axiom \citep{Sundararajan_Najmi_2020} is satisfied.} 
We show that this property also holds for more general weights in the Appendix (see Proposition~\ref{prop:beta-computation-additive-f-general-weights}), and for more general $f$ in Section~\ref{section:sparse}. 

\subsubsection{Linear functions}

We can be even more precise in the case of linear functions. 
In this case, the functions $f_j$ are defined as $f_j(x)=\lambda_j x$ with $\lambda_j\in\Reals$ for any $1\leq j\leq d$ some real coefficients.
Recall that, for any $1\leq j\leq d$ and $1\leq b\leq p$, we defined $\mutilde_{j,b}$ as the mean of the random variable $\truncnorm{\mu_{j,b},\sigma_{j,b},q_{j,b-1},q_{j,b}}$. 

\begin{mycorollary}[Computation of $\beta^f$, linear $f$]
\label{cor:beta-computation-liner-default-weights}
Assume that $f$ is linear. 
Let us set $\xi\in\supp$.  
Then Theorem~\ref{th:betahat-concentration-general-f-default-weights} holds and, for any $1\leq j\leq d$,
\begin{equation}
\label{eq:beta-linear-default-ls}
\beta^f_j = \lambda_j \cdot \frac{1}{p-1}\sum_{b = 1}^p (\mutilde_{j,\bxi_j} - \mutilde_{j,b})
\, .
\end{equation}
Moreover, if one defines 
\begin{equation}
\label{eq:def:muttilde}
\muttilde_j \defeq \frac{\mutilde_{j,\bxi_j} + \sum_{b\neq \bxi_j}  \exps{\frac{-1}{2\nu^2}}\mutilde_{j,b}}{1 + (p-1)\exps{\frac{-1}{2\nu^2}} }
\, ,
\end{equation}
the weighted average of the $\mutilde_{j,b}$ across dimension $j$, 
the intercept is given by 
\[
\beta_0^f = f(\muttilde) - \frac{1}{p\littlec - 1}\sum_{j=1}^{d}  (\mutilde_{j,\bxi_j} - \muttilde_j) \lambda_j
\, .
\]
\end{mycorollary}

\paragraph{Linearity in $\lambda_j$.}
Intuitively, Corollary~\ref{cor:beta-computation-liner-default-weights} tells us that for any $1\leq j\leq d$, $\beta_j^f$ is \emph{proportional} to $\lambda_j$, the coefficient of the linear model corresponding to this feature.  
\textbf{For a linear model, the interpretable coefficient along dimension $j$ is proportional to the corresponding coefficient of the linear model.} 
This result is analogous to Theorem~3.1 in \citet{Garreau_von_Luxburg_2020}. 
Such property is also satisfied by kernel SHAP (Corollary~1 in \citet{Lundberg_Lee_2017}), for which $\beta_j^f = \lambda_j\cdot (\xi_j-\expec{x_j})$ in our notation. 
This is often considered as a nice property of an interpretability method: since a linear model is already interpretable to some extent, the interpretable version thereof should coincide up to constants. 
In the language of \citet{Sundararajan_Najmi_2020}, this is called the \emph{proportionality} axiom. 
We demonstrate in Figure~\ref{fig:linear_default_ls} the accuracy of the theoretical predictions of Corollary~\ref{cor:beta-computation-liner-default-weights}. 

\begin{figure}
	\begin{center}
		\includegraphics[scale=0.3]{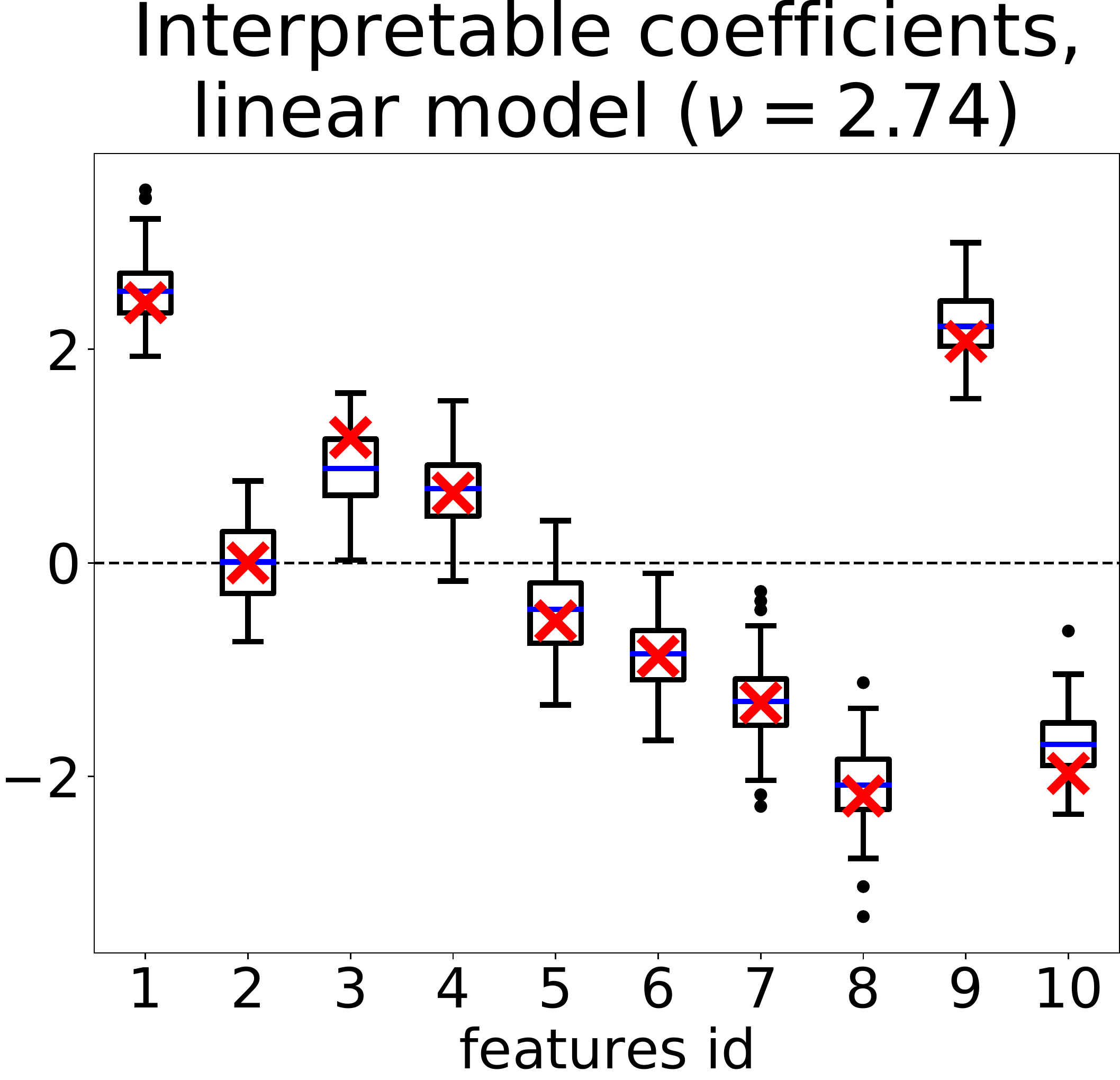}\hspace{1cm}
		\includegraphics[scale=0.3]{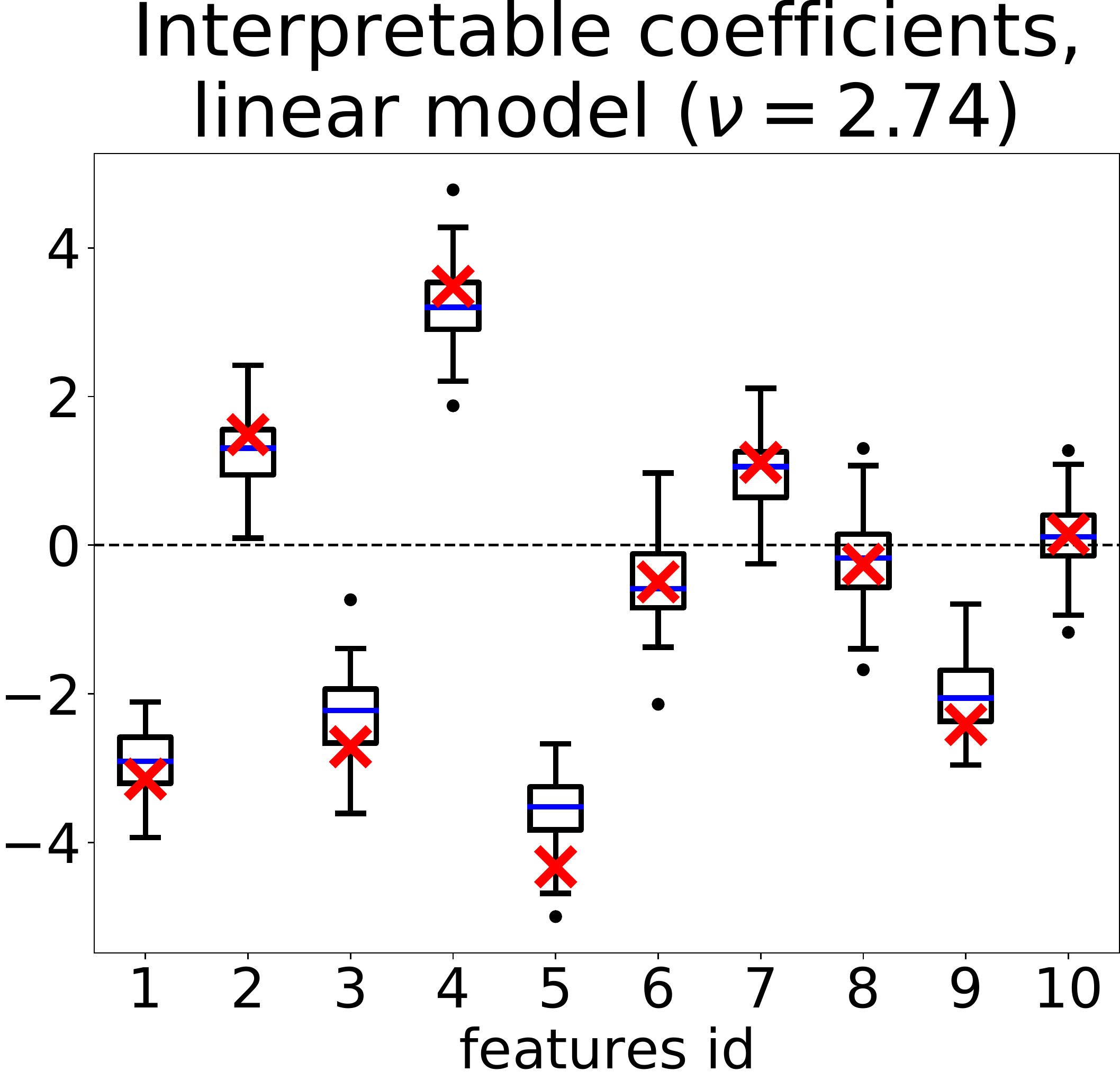}
	\end{center}
	\caption{\label{fig:linear_default_ls}In this set of experiments, we ran Tabular LIME with default settings ($n=5\cdot 10^3$ perturbed samples, $\nu=\sqrt{0.75 d}$) on a linear $f$ with arbitrary coefficients. The ambient dimension is $10$ in the top panel, $20$ in the bottom panel. In black, the whisker boxes corresponding to $10^2$ repetitions of this experiment. In blue, the experimental mean. In red, the values given by Corollary~\ref{cor:beta-computation-liner-default-weights}.  We can see that the theoretical predictions of Eq.~\eqref{eq:beta-linear-default-ls} match the experimental results.}
\end{figure}

\paragraph{Cancellation phenomenon.}
Now let let us take a closer look at the multiplicative coefficient between $\beta_j^f$ and $\lambda_j$, given by Eq.~\eqref{eq:beta-linear-default-ls}. 
For the purpose of this discussion, let us define 
\[
\cprop_j\defeq \frac{1}{p-1}\sum_{b = 1}^p (\mutilde_{j,\bxi_j} - \mutilde_{j,b})
\, .
\]
We notice that $\cprop_j$ depends only on the discretization scheme, the means of the truncated Gaussians $\mutilde_{j,b}$, and the index of the bin containing $\xi$ along feature $j$. 
More precisely, one can see this $\cprop_j$ as the average of the differences between the mean of $x_j$ on the bin containing $\xi_j$ and the mean of $x_j$ on all other bins. 
We refer to Figure~\ref{fig:linear-coefficient-explanation} for a visual help. 

\begin{figure}
	\begin{center}
		\includegraphics[scale=0.3]{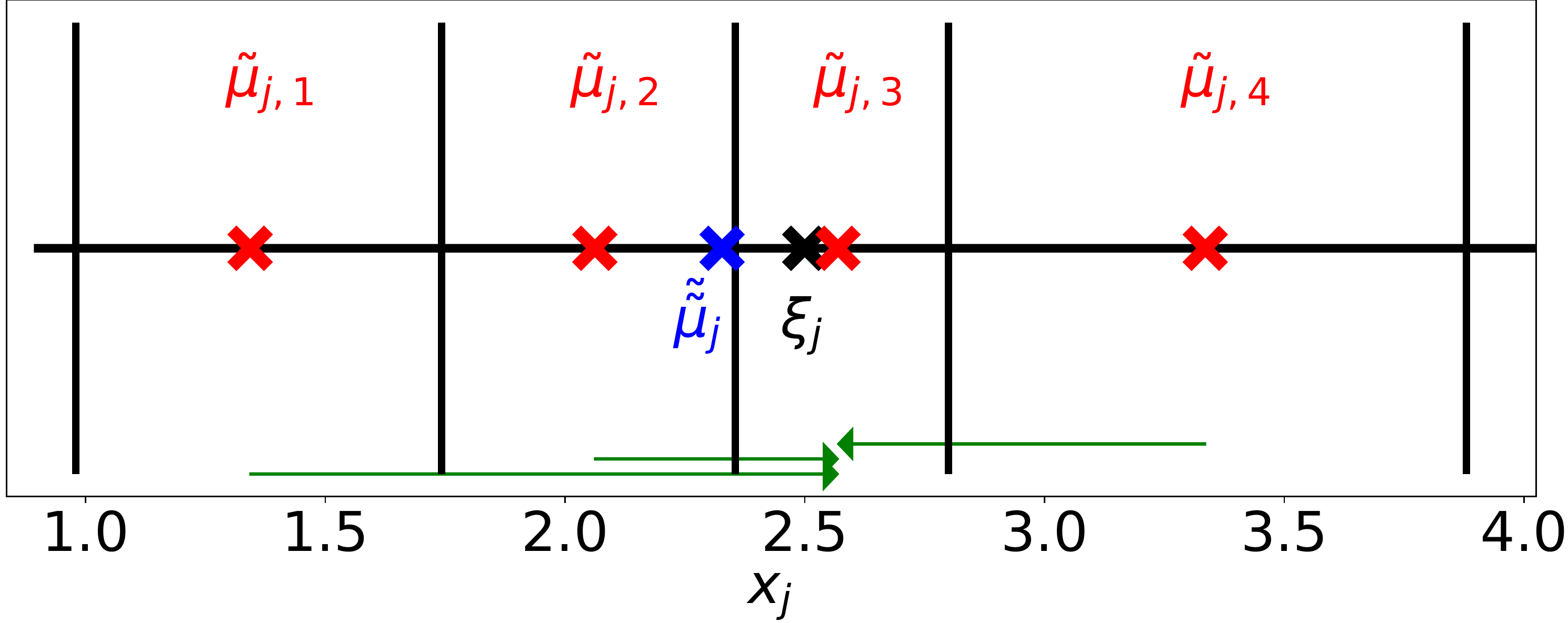} \\
		\vspace{0.5cm}
		\includegraphics[scale=0.3]{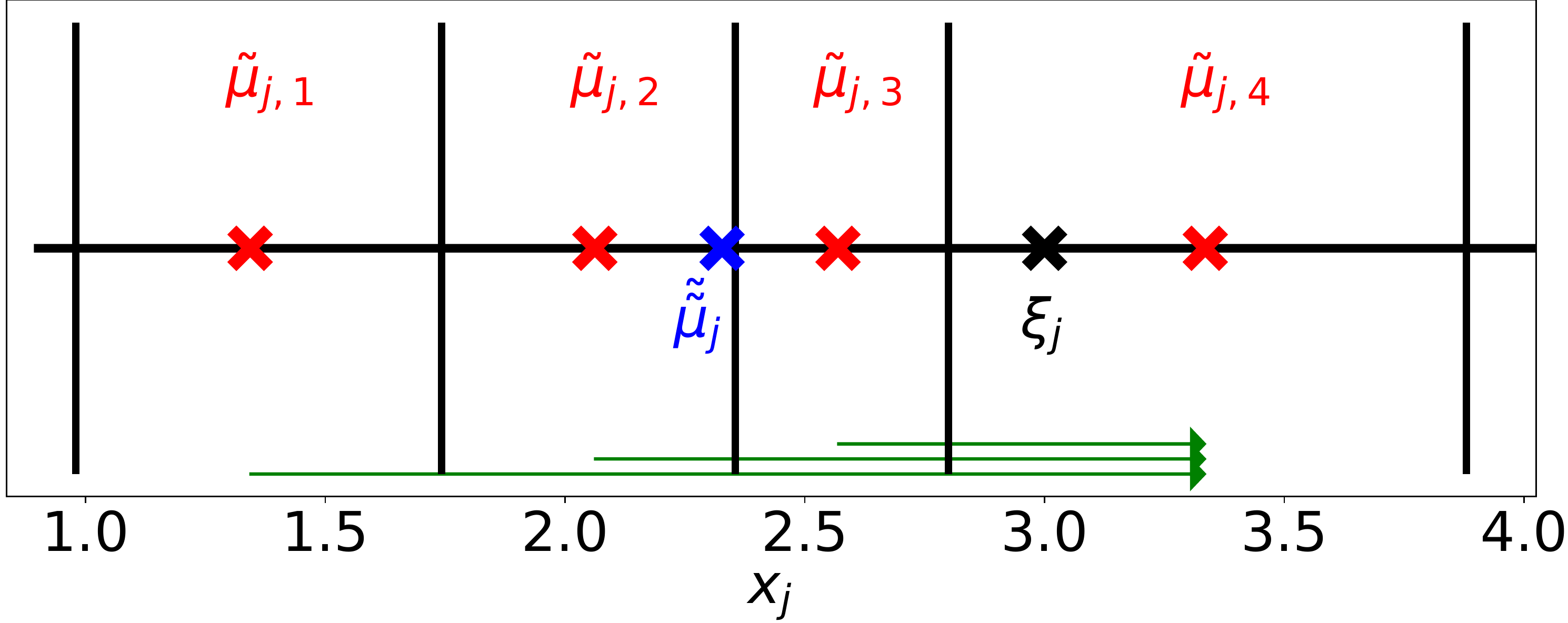}
	\end{center}
	\caption{\label{fig:linear-coefficient-explanation}Computation of $\cprop_j$. Recall that the $\mutilde_{j,b}$ (red crosses) are the averages of $x_j$ on each bin $b$. The bin index of $\xi$ along dimension $j$ is given by $\bxi_j$. The sum of the $\mutilde_{j,b}-\mutilde_{j,\bxi_j}$ is represented by the green arrows, $\cprop_j$ is the average. We report $\muttilde$ in blue (Eq.~\eqref{eq:def:muttilde}). In the top panel, since the bin containing $\xi_j$ is in a central position, the proportionality coefficient is small, whereas it is much larger in the bottom panel. }
\end{figure}

A first observation is that nothing prevents $\cprop_j$ to be zero, leading to \textbf{cancellation of the $j$th interpretable component in the explanation whatever the value of $\lambda_j$ may be.}  
In particular, a modification of the number of bins~$p$ can achieve this cancellation: 
when the bin containing $\xi_j$ along dimension $j$ is in a central position, since the bins are typically of the same size and $\mutilde_{j,b}$ takes a central position in the bin, the values taken by $\cprop_j$ can be very small (see Figure~\ref{fig:linear-coefficient-explanation}). 
More precisely, if 
\[
\mutilde_{j,\bxi_j} = \frac{1}{p-1}\sum_{b\neq \bxi_j} \mutilde_{j,b}
\, ,
\]
then $\beta_j^f=0$. 
Intuitively, along dimension $j$, if the means of the perturbed sample distribution on the bins balance the mean on the bin containing $\xi$, then $\beta^f_j$ vanishes. 
We demonstrate this effect in Figure~\ref{fig:cancellation-linear-default-ls}.
In this specific example, we  considered a uniformly distributed training set on $[-10,10]$. 
Thus the $\mu_{j,b}$ are evenly distributed across $[-10,10]$, as well as the means of the truncated Gaussians $\mutilde_{j,b}$. 
For any $\xi$ in a central position, the following happens:
\begin{itemize}
	\item if $p$ is even, then $\mutilde_{j,\bxi_j}$ and $\frac{1}{p-1}\sum_{b\neq \bxi_j} \mutilde_{j,b}$ are far away (left panel of Figure~\ref{fig:cancellation-linear-default-ls});
	\item if $p$ is odd, then $\mutilde_{j,\bxi_j}$ is in a central position and approximately equal to $\frac{1}{p-1}\sum_{b\neq \bxi_j} \mutilde_{j,b}$. The corresponding coefficient vanishes (right panel of Figure~\ref{fig:cancellation-linear-default-ls}), whatever the value of $\lambda_j$ may be. 
\end{itemize}
We do not see this cancellation phenomenon as a good property: we do not want to erase meaningful feature from the explanation simply by changing the discretization. 

\begin{figure}
	\begin{center}
		\includegraphics[scale=0.3]{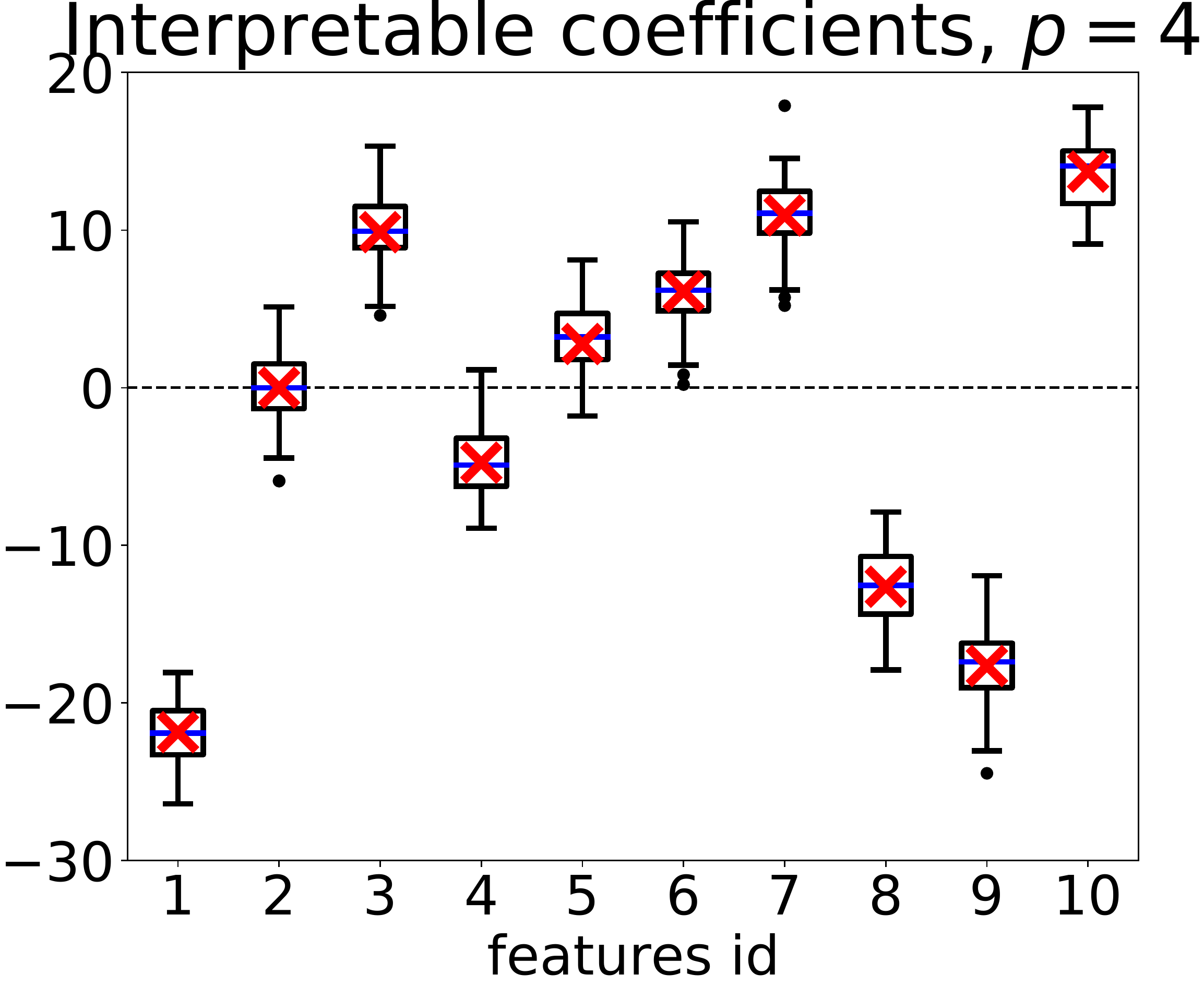}
		\includegraphics[scale=0.3]{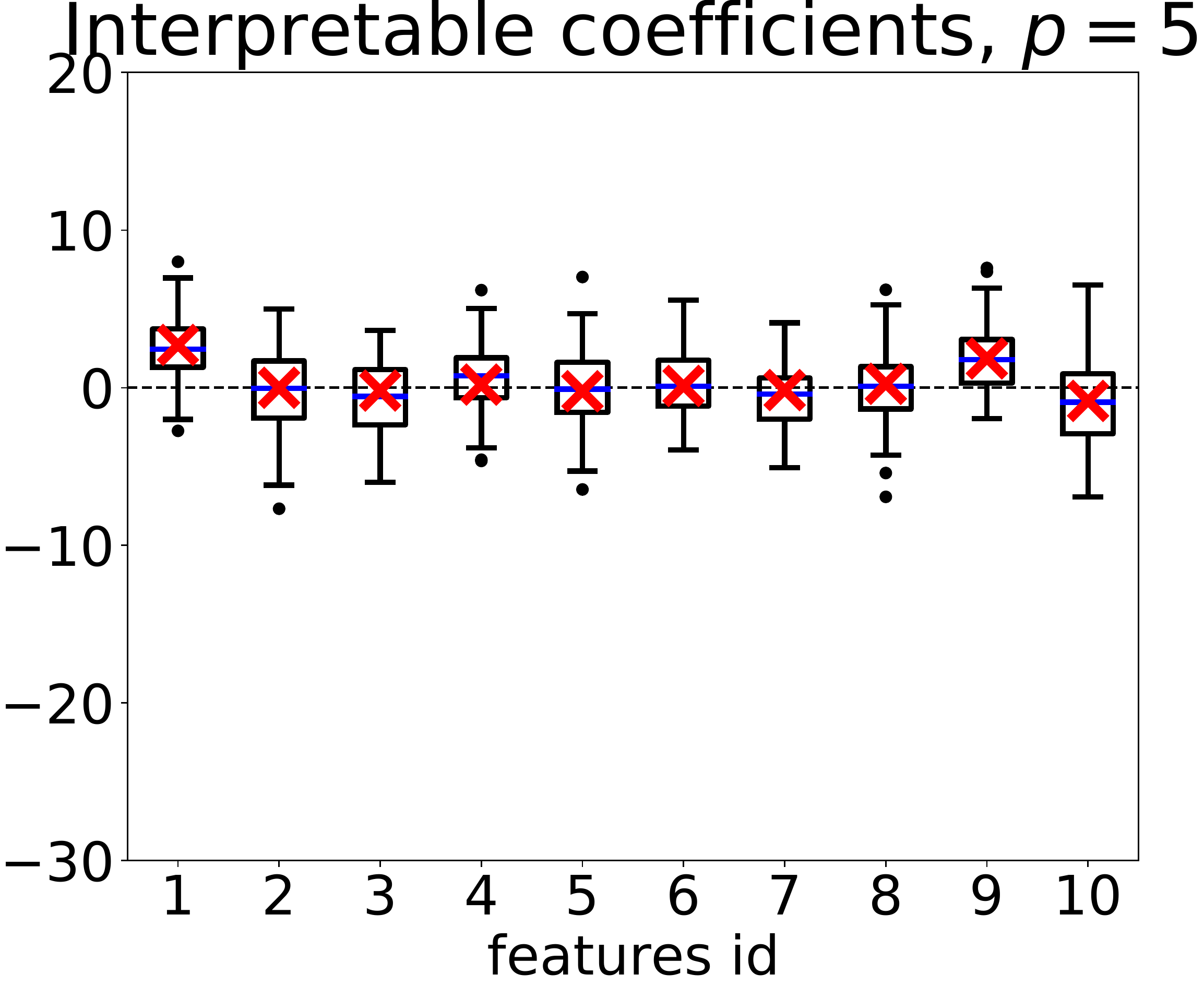}
	\end{center}
	\caption{\label{fig:cancellation-linear-default-ls}Cancellation phenomenon. Explanation given by Tabular LIME for the same linear model and example in dimension $10$. The training set is sampled independently according to a $\uniform{[-10,10]}$ distribution along each coordinate. We choose $n=5\cdot 10^3$, and $\nu=\sqrt{0.75 d}$. In the left panel, we took $p=4$ bins, and $p=5$ bins in the right panel. In red the theoretical values given by Proposition~\ref{prop:beta-computation-additive-f-default-weights}, in black the experimental values ($10^2$ repetitions). We can see that, surprisingly, choosing a different number of bins along each dimension ($5$ instead of $4$) sets the values of \emph{all} interpretable coefficients near zero, both in theory and in practice. }
\end{figure}

It is interesting to note, however, that it is not possible to cancel out the interpretable coefficient by a clever choice of bandwidth in Eq.~\eqref{eq:beta-linear-default-ls}, contrarily to what is done in \citet{Garreau_von_Luxburg_2020}. 
This seems due to the specific setting considered in \citet{Garreau_von_Luxburg_2020}, especially choosing $\mu_{j,b}$ and $\sigma_{j,b}$ not depending on $b$. 
In fact, in the present case, the magnitude of the explanations does not depend on the bandwidth at all. 

Let us conclude with the inspection of $\cprop_j$ by considering $p$ and $\nu$ fixed and letting $\xi_j$ vary. 
In that case, when the boundary of a bin on feature $j$ is crossed, the value of the proportionality coefficient may change sign. 
This can provoke non-monotone behaviors, and therefore the \emph{demand monotonicity} axiom of \citet{Sundararajan_Najmi_2020} is not satisfied. 

\medskip

We conclude this section by presenting the proof of Corollary~\ref{cor:beta-computation-liner-default-weights}, in order to show how to use Theorem~\ref{th:betahat-concentration-general-f-default-weights} in specific cases. 
It is a direct application of Proposition~\ref{prop:beta-computation-additive-f-default-weights}. 

\begin{proof}
Reading Proposition~\ref{prop:beta-computation-additive-f-default-weights}, we need to compute $\ew_{j,b}^{f_j}$ for all $1\leq j\leq d$ and $1\leq b\leq p$. 
By linearity, all that is required is the computation of $\ew_{j,b}^\times$, which is equal to $\ew_{j,b}^\times=\mutilde_{j,b}$ if $b=\bxi_j$ and $\exps{\frac{-1}{2\nu^2}}\mutilde_{j,b}$ otherwise by definition of $\mutilde$. 
We deduce that $\littlec_j^\times = \mutilde_{j,\bxi_j} + \sum_{b\neq \bxi_j} \mutilde_{j,b}$, and $\muttilde_j = \frac{\littlec_j^{f_j}}{\littlec}$ by definition of $\muttilde$ (Eq.~\eqref{eq:def:muttilde}). 
Following Proposition~\ref{prop:beta-computation-additive-f-default-weights}, we obtain
\[
\beta^f_0 = f(\muttilde) - \frac{1}{p\littlec - 1}\sum_{j=1}^{d}  (\mutilde_{j,\bxi_j} - \muttilde_j) \lambda_j
\, ,
\]
and, for any $1\leq j\leq d$, 
\[
\beta^f_j = \frac{p\littlec}{p\littlec - 1}(\mutilde_{j,\bxi_j} - \muttilde_j) \lambda_j
\, .
\]
We simplify further this last display using 
\begin{align*}
\muttilde_j - \mutilde_{j,\bxi_j} &= \frac{\mutilde_{j,\bxi_j} + \sum_{b\neq \bxi_j} \exps{\frac{-1}{2\nu^2}} \mutilde_{j,b}}{1 + (p-1)\exps{\frac{-1}{2\nu^2}}} - \mutilde_{j,\bxi_j} \tag{definition of $\muttilde$} \\
&= \frac{\sum_{b\neq \bxi_j} \exps{\frac{-1}{2\nu^2}} (\mutilde_{j,b} - \mutilde_{j,\bxi_j})}{1+(p-1)\exps{\frac{-1}{2\nu^2}}}
\, .
\end{align*}
The result follows after some algebra. 
\end{proof}


\subsection{Multiplicative models, including indicator functions and Gaussian kernels}
\label{section:multiplicative}

In this section, we turn to the study of functions $f$ that can be written as a product of functions where each term depends only on one coordinate. 
Namely, we now make the following assumption on $f$:
\begin{assumption}[Multiplicative function]
	\label{ass:multiplicative}
	We say that $f$ is \emph{multiplicative} if there exist functions $f_1,\ldots,f_d:\Reals\to\Reals$ such that
	\[
	\forall x\in\Reals^d, \quad f(x) = \prod_{j=1}^{d} f_j(x_j)
	\, .
	\]
\end{assumption}

In this case, as promised, we also can be more explicit in the statement of Theorem~\ref{th:betahat-concentration-general-f-default-weights}. 
Before stating our result, let us recall that we set $\ew_{j,b}^{f_j}$ to be the weighted expectation of $f_j$ on bin $b$ along dimension $j$ (Eq.~\eqref{eq:def-ew-default-weights}) and that $\littlec_{j}$ and $\littlec_{j}^{f_j}$ are normalization constants (Eqs.~\eqref{eq:def-little-cj-psi}). 

\begin{myproposition}[Computation of $\beta^f$ for a multiplicative function]
	\label{prop:beta-computation-multiplicative-default}
	Assume that~$f$ satisfies Assumption~\ref{ass:multiplicative} and suppose that the assumptions of Theorem~\ref{th:betahat-concentration-general-f-default-weights} hold (in particular, for each $1\leq j\leq d$, $f_j$ is bounded on $[q_{j,0},q_{j,p}]$). 
	Set $\xi\in \supp$. 
	Then Theorem~\ref{th:betahat-concentration-general-f-default-weights} holds and, for any $1\leq j\leq d$,
\[
\beta^f_j = \frac{\prod_{k=1}^{d}\littlec_k^{f_k}}{\bigc} \cdot \frac{p\littlec}{p\littlec - 1} \left( \ew_{j,\bxi_j}^{f_j}\cdot \frac{\littlec}{\littlec_j^{f_j}} - 1\right)
\, .
\]
Moreover, the intercept is given by
	\[
	\beta^f_0 = \frac{\prod_{k=1}^{d}\littlec_k^{f_k}}{\bigc} \left\{1 + \sum_{j=1}^{d} \frac{1}{p\littlec-1} \left(1 - \ew_{j,\bxi_j}^{f_j}\cdot \frac{\littlec}{\littlec_j^{f_j}}\right)\right\}
	\, .
	\]
\end{myproposition}

Proposition~\ref{prop:beta-computation-multiplicative-default} is a consequence of Lemma~\ref{lemma:beta-computation-multiplicative-f-general-weights} in the Appendix (which is true for more general weights). 

As in the additive case, we can see that (i) $f_j$ only comes into play for the value of $\beta_j^f$, up to a multiplicative constant common to all interpretable coefficients; (ii) unused coordinates are ignored in the explanation. 
Let us give a bit more details regarding the second point and let us assume that for a certain index $j$, the function $f$ does not depend on the $j$th coordinate. 
In other words, $f_j(x)=\kappa$ for any $x\in\Reals$ where $\kappa$ is a constant. 
Then, by definition, $\ew_{j,\bxi_j}^{f_j} = \kappa\ew_{j,\bxi_j}$ and $\littlec_j^{f_j}=\kappa\littlec_j$. 
It follows from Proposition~\ref{prop:beta-computation-multiplicative-default} that $\beta^f_j=0$. 
For a multiplicative $f$, Tabular LIME provably ignores unused coordinates.
This is also true for more general weights (see Lemma~\ref{lemma:beta-computation-multiplicative-f-general-weights} in the Appendix). 

We now give two fundamental examples in which multiplicative functions are the building block of the model: tree regressors when with splits parrallel to the axis and kernel methods using the Gaussian kernel.


\subsubsection{Indicator functions and partition-based models}

Partition-based models split the feature space and then fit a simple model on each element of the partition.
Simply put, once trained, such models outputs predictions according to
\begin{equation}
\label{eq:def-cart}
f(x) = \sum_{A\in\partition} f(A) \indic{x\in A}
\, ,
\end{equation}
where we set $f(A)$ the value of $f$ on~$A$. 
In particular, these partition-based rules include tree-based methods (see Section~9.2 in \citet{Hastie_Tibshirani_Friedman_2001} for an introduction). 
Random trees are the building bricks of very popular regression algorithms such as random forests \citep{Breiman_2001}, which are considered as one of the most successful general-purpose algorithms in modern-times \citep{Biau_Scornet_2016}. 
In this section, we want to answer the following question: 
Do the explanations provided by Tabular LIME make sense when $f$ is a partition-based model? 
We will focus on, CART trees \citep{Breiman_et_al_1984}, but most of the discussion remains true if the elements of $\partition$ are \emph{rectangles}.

Without loss of generality, we will assume from now on that \emph{$A$ is included into the bins of Tabular LIME onto each dimension.} 
Indeed, instead of $\partition$, one can always consider the partition formed by the intersection between $\partition$ and the LIME grid. 
Formally, for any $1\leq j\leq d$, if we call $A_j$ the projection of $A$ in the $j$-th dimension, then there exist $1\leq b\leq p$ such that
\[
A_j \subseteq [q_{j,b-1},q_{j,b}] \eqdef \Abar_j
\, .
\]
Thus for any $A\in\partition$, there exists a rectangle $\Abar$ such that the $\Abar_j$ are the bins. 

By linearity, if we want to understand how Tabular LIME produces interpretable coefficients for $f$, we only need to understand how Tabular LIME produces interpretable coefficients for the function $\indic{A} : x\mapsto \indic{x\in A}$ for any $A\in\partition$. 
The explanation for $f$ will simply be the (weighted) sum of the explanations. 

We now make the fundamental observation that \textbf{$\indic{A}$ is a multiplicative function whenever $A$ is a rectangles with faces parallel to the axes.}  
Indeed, in this case we can write $A=\prod_{j=1}^d A_j$, and therefore $\indic{x\in A} = \prod_{j=1}^{d} \indic{x_j\in A_j}$. 
Hence we can apply Proposition~\ref{prop:beta-computation-multiplicative-default} to the function $\indic{A}$. 
Before giving our result for indicator functions, we need to introduce two important concepts. 
\begin{mydefinition}[Relative importance]
\label{def:relative-importance}
Let $A\subseteq \Reals^d$ be a rectangle such that $A$ is included in a $d$-dimensional bin $\Abar$ given to LIME. 
In other words, we assume that $A=\prod_{j=1}^d[s_j,t_j]$ and that, for each $1\leq j\leq d$, there exist a $b$ such that $[s_j,t_j]\subseteq [q_{j,b-1},q_{j,b}]$. 
We define the \emph{relative importance} of $A$ by
\[
\relimp{A} \defeq 
\prod_{j=1}^d \frac{\Phi\left(\frac{t_j -\mu_{j,b}}{\sigma_{j,b}}\right) - \Phi\left(\frac{s_j-\mu_{j,b}}{\sigma_{j,b}}\right)}{\Phi\left(\frac{q_{j,b}-\mu_{j,b}}{\sigma_{j,b}}\right) - \Phi\left(\frac{q_{j,b-1}-\mu_{j,b}}{\sigma_{j,b}}\right)}
\, .
\]
\end{mydefinition}
Intuitively, $\relimp{A}$ is the mass given to $A$ by the truncated Gaussian measure normalized by the volume of $\Abar$. 
One can check that a first order approximation of $\relimp{A}$ is given by $\volume{A}/\volume{\Abar}$ by setting $\Phi(x)\approx \frac{1}{2}+\frac{1}{\sqrt{2\pi}}x$. 
If $A$ is very small with respect to $\Abar$ or is located in a low density region of $\Abar$, then $\relimp{A} = 0$. 
On the contrary, if $A$ is large with respect to $\Abar$ or is located in a high density region of $\Abar$, then $\relimp{A}=1$. 

We now present a notion of distance between $\xi$ and $A$. 
\begin{mydefinition}[Bin distance]
\label{def:bin-distance}
Let $A\subseteq\Reals^d$ and $\Abar$ be a rectangle of the LIME grid such that $A\subseteq \Abar$. 
Let $\xi$ be a point in $\supp$. 
We define the \emph{bin distance} between $\xi$ and $A$ by
\[
\dist{\xi}{A} \defeq \card{\{k \in \{1,\ldots,d\}\text{ such that } \xi_k \notin \Abar_k\}}
\, ,
\]
where $\card{E}$ denotes the cardinality of the finite set $E$. 
\end{mydefinition}
To put it simply, $\dist{\xi}{A}$ counts in how many directions the projection of $\xi$ does not fall into the projection of $\Abar$. 
If $\xi$ is in the same $d$-dimensional bin as $A$, then $\dist{\xi}{A}=0$. 
On the contrary, if $\xi$ is not aligned with $A$ in any given direction, then $\dist{\xi}{A}=d$. 
Note that $\dist{\xi}{A}$ can be high without the Euclidean distance between $\xi$ and $A$ being high. 

We are now able to give the expression of $\beta^A$ which is short for $\beta^{\indic{A}}$. 
\begin{myproposition}[Computation of $\beta^f$ for an indicator function]
\label{prop:beta-computation-indicator}
Let $A$ be a rectangle with faces parallel to the axes such that $A\subseteq \Abar$, where $\Abar$ is a $d$-dimensional bin given as input to Tabular LIME. 
Then Theorem~\ref{th:betahat-concentration-general-f-default-weights} holds for $\indic{A}$ and for any $1\leq j\leq d$,
\begin{itemize}
\item for all $j\in [d]$ such that $\xi_j\in\Abar_j$, 
\[
\beta_j^A = \frac{1}{p^{d-1}\littlec^{d-1}} \relimp{A}\exp{\frac{-\dist{\xi}{A}}{2\nu^2}}
\, ;
\]
\item for all $1\leq j\leq d$ such that $\xi_j\notin\Abar_j$,
\[
\beta_j^A = 
\frac{-1}{(p-1)p^{d-1}\littlec^{d-1}}  \relimp{A}\exp{\frac{-\dist{\xi}{A}+1}{2\nu^2}}
\, .
\]
\end{itemize} 
The intercept is given by 
\[
\beta_0^A = \frac{1}{p^d\littlec^d} \relimp{A}\exp{\frac{-\dist{\xi}{A}}{2\nu^2}} - \frac{1}{p\littlec} \sum_{j=1}^d \beta_j^A
\, .
\]
\end{myproposition}

As a direct consequence of Proposition~\ref{prop:beta-computation-indicator}, we can see that the interpretable coefficients are small when $\xi$ (the example to explain) is far away from $A$ (the support of the function to explain).  
We believe that this is a desirable property: we want the explanations to be local, and in this case the function is locally completely flat. 
Therefore we would like to attribute $0$ to each feature (except in the bin containing the hyperrectangle $A$).  
Note however that there is a notable exception: \textbf{whenever $\xi_j$ is aligned with $A$ along coordinate $j$, then $\beta_j^A$ can be large even though $\xi$ is far away from $A$.} 
We refer to Figure~\ref{fig:2D-indicator} for an illustration of this phenomenon. 
In particular, we can visualize artifacts on the bins that are aligned with $A$ along the axes.

\begin{figure}
	\begin{center}
		\includegraphics[scale=0.3]{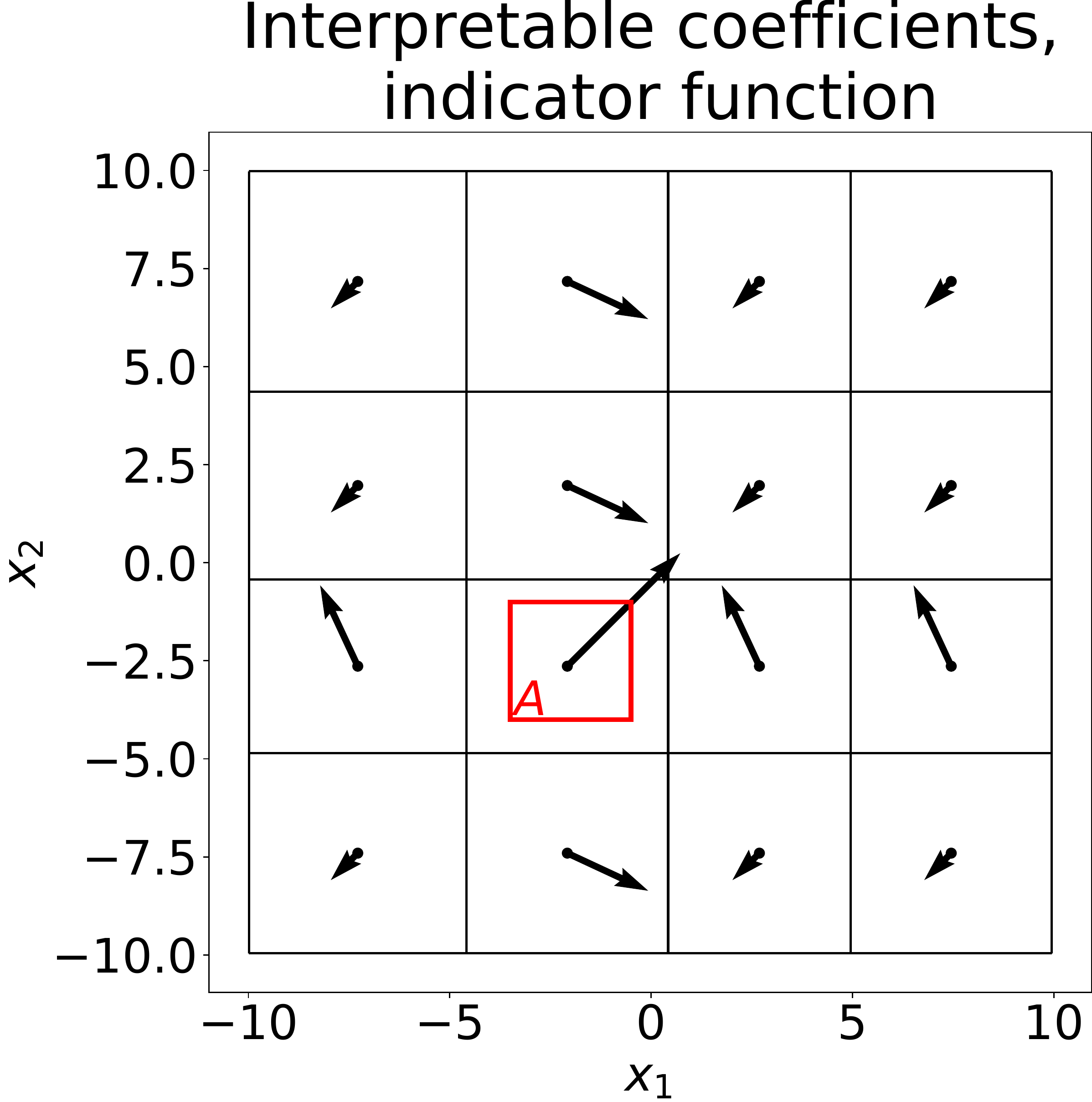}
		\includegraphics[scale=0.3]{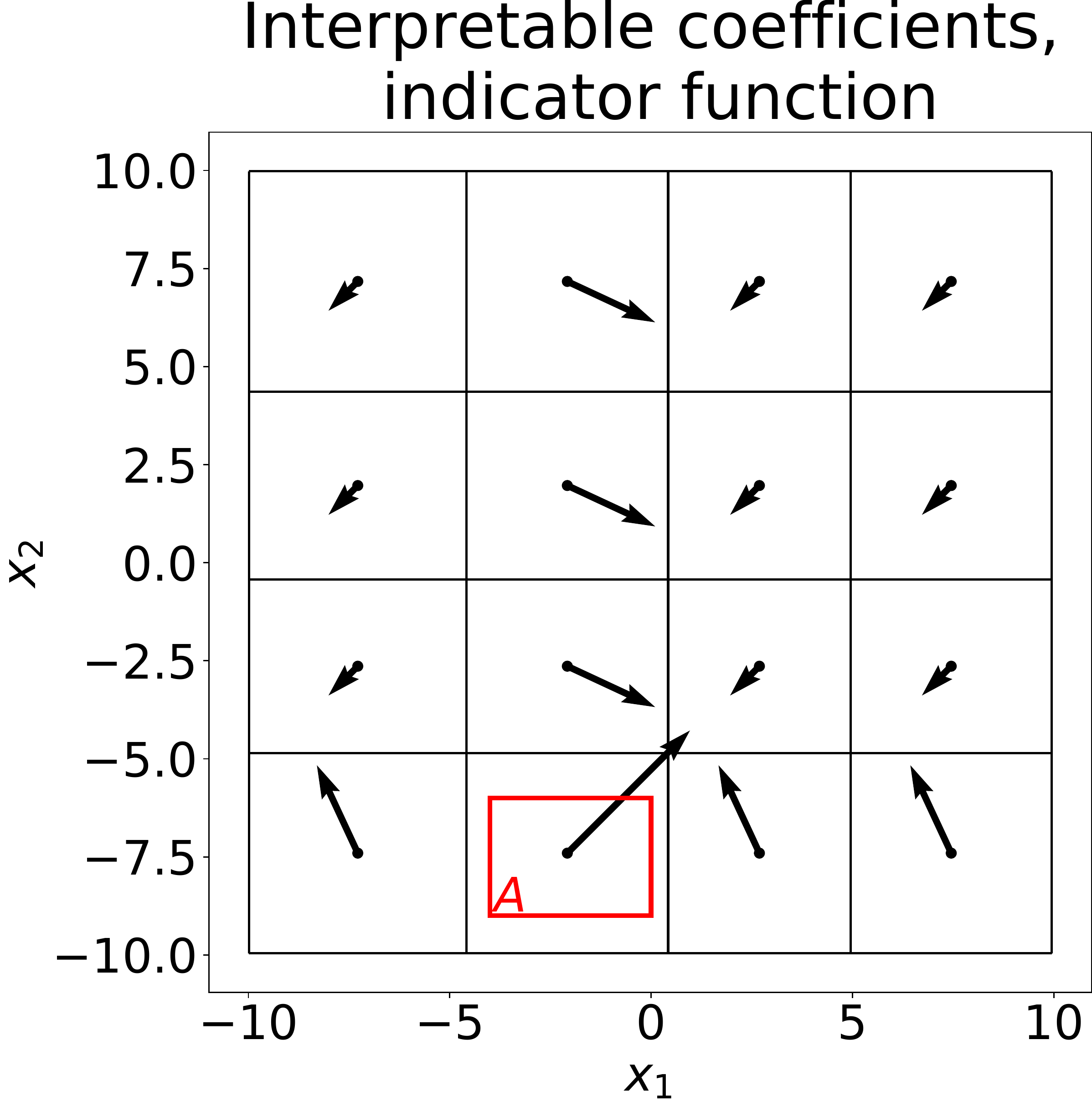}
	\end{center}
	\caption{\label{fig:2D-indicator}Explanations provided by Tabular LIME (default settings) for $\indic{A}$ in dimension $2$. In red, the rectangle $A$. 
	The training data is a uniform sample on $[-10,10]^2$. 
The grid corresponds to the discretization of the input space. 
	For each $2$-dimensional bin, we compute the explanations given by Tabular LIME at the central point. 
	The black arrows correspond to the vectors $(\beta^A_1,\beta^A_2)$ (not to scale): for instance, a large arrow pointing north-east means that both $\beta^A_1$ and $\beta^A_2$ take large positive values at $\xi$, the base-point of the arrow.   
}
\end{figure}

%
\begin{mycorollary}[Computation of $\beta^f$ for partition-based model]
\label{cor:beta-computation-tree}
Let $\xi\in\supp$. 
Let $f$ be a partition-based model defined as in Eq.~\eqref{eq:def-cart}. 
Assume that each $A\in\partition$ is a rectangle with faces parallel to the axes and is included in a $d$-dimensional bin of Tabular LIME. 
Then, for any $1\leq j\leq d$, 
\begin{equation}
\label{eq:beta-tree}
\beta_j^f = \frac{1}{p^{d-1}c^{d-1}}\sum_{\substack{A\in\partition \\ \xi_j\in \Abar_j}} f(A) \relimp{A}\exps{\frac{-\dist{\xi}{A}}{2\nu^2}}
 - \frac{1}{(p-1)p^{d-1}\littlec^{d-1}}\sum_{\substack{A\in\partition \\ \xi_j\notin \Abar_j}} f(A)\relimp{A}\exps{\frac{-\dist{\xi}{A}+1}{2\nu^2}}
 \, .
\end{equation}
\end{mycorollary}

Essentially, the explanation for $\xi$ is the difference between two weighted averages: the average value of $f$ ``above'' $\xi_j$ to the left and the average value of $f$ everywhere else to the right. 
The weights of each term are the product of the relative importance of the rectangle and an exponential decay driven by the bin distance between $\xi$ and $A$. 
This is maybe more visible when $\nu$ is large. 
In that case, Eq.~\eqref{eq:beta-tree} becomes
\[
\beta_j^f = \frac{1}{p^{d-1}}\sum_{\substack{A\in\partition \\ \xi_j\in \Abar_j}} f(A) \relimp{A} - \frac{1}{(p-1)p^{d-1}}\sum_{\substack{A\in\partition \\ \xi_j\notin \Abar_j}} f(A)\relimp{A}
 \, .
\]
In the last display, $p^{d-1}$ counts the number of $d$-dimensional bins that are above $\xi_j$, while $(p-1)p^{d-1}$ counts all the others bins. 
We demonstrate the accuracy of our theoretical predictions in Figure~\ref{fig:tree-explanation} in the particular case of a CART tree. 

\begin{figure}
    \centering
    \includegraphics[scale=0.30]{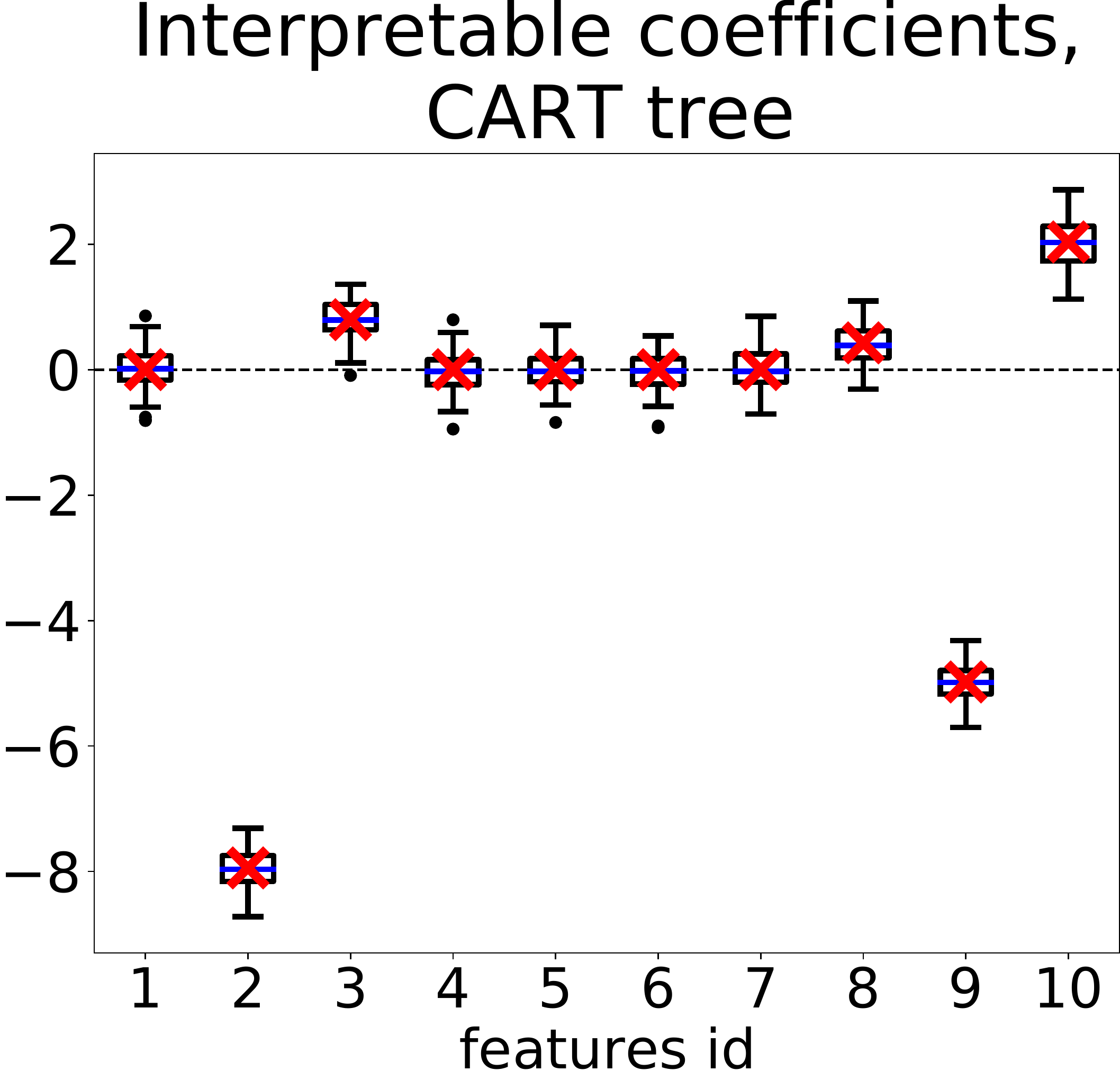}\hspace{1cm}
    \includegraphics[scale=0.30]{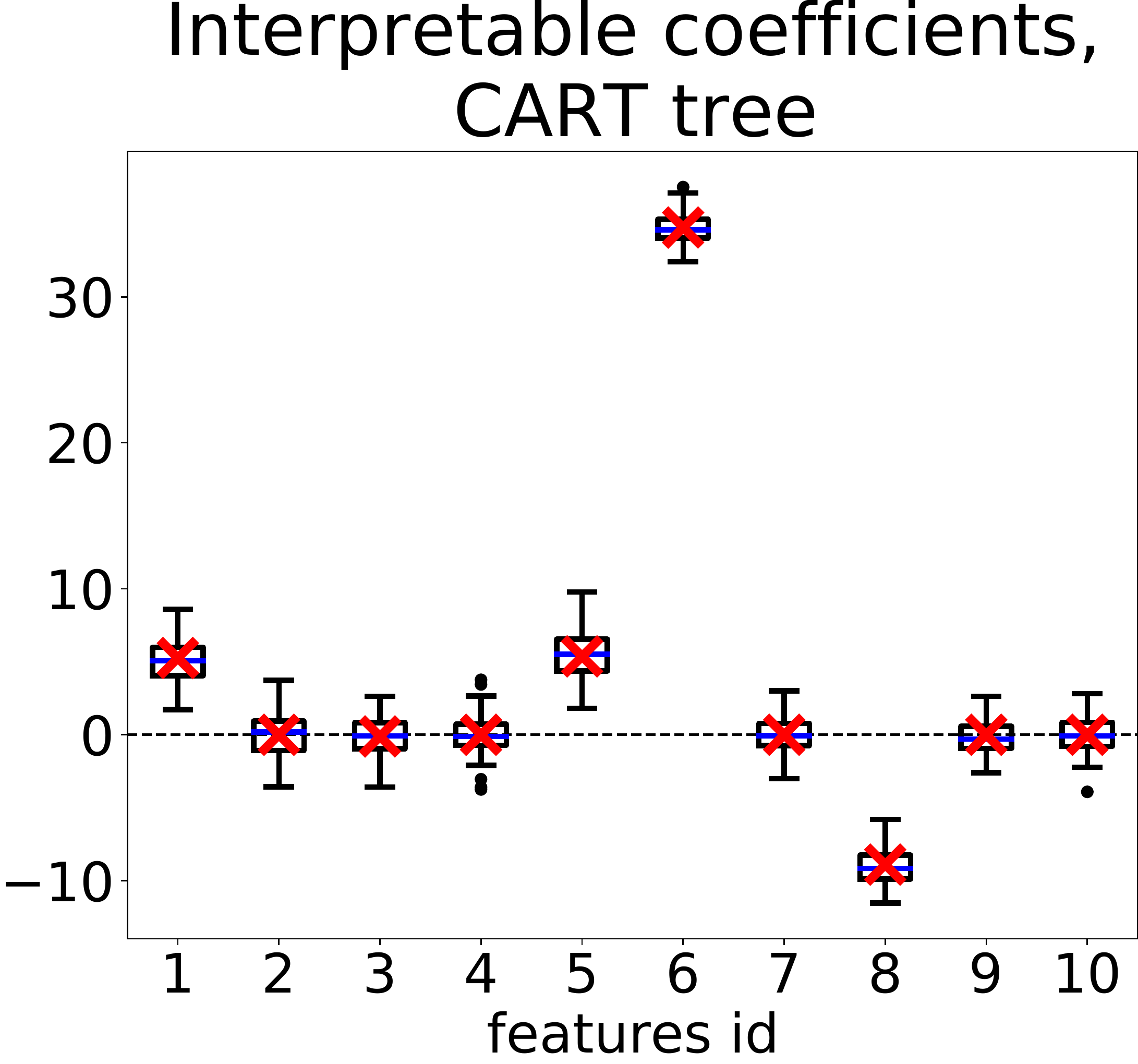} 
    \caption{\label{fig:tree-explanation}Interpretable coefficients given by Tabular LIME with default settings for a CART tree (depth $3$, training data is uniform in $[-10,10]^d$). The tree was fitted on the function $x\mapsto \sum_j x_j$ in the left panel, $x\mapsto \sum_{j=1}^{d}x_j^2$ in the right panel. We repeated the experiment $100$ times. The theoretical predictions (red crosses) are obtained by retrieving the partition associated to the tree and using Corollary~\ref{cor:beta-computation-tree}. }
\end{figure}

Perhaps the most striking feature of Eq.~\eqref{eq:beta-tree} is the positive influence of rectangles that are aligned with $\xi$. 
Indeed, when $\dist{\xi}{A}$ is small with respect to $\nu^2$, then the $f(A)\relimp{A}\exps{-\dist{\xi}{A}/(2\nu^2)}$ terms have a huge influence in Eq.~\eqref{eq:beta-tree}. 
Even though these rectangles can be far away from $\xi$: $\dist{\xi}{A}$ only counts how many directions do not align $\xi$ and~$\Abar$. 
To put it plainly, \textbf{Tabular LIME can give positive influence to $j$ whereas the function to explain is totally flat in the vicinity of $\xi$} if there is a bump in $f$ aligned with $\xi$ in this direction. 
We demonstrate this effect in Figure~\ref{fig:tree-artifacts}. 
We do not think that this is a desired behavior. 
It is important to notice, however, that these artifacts are less visible in high dimension, since such alignments are less common. 
Moreover, for complicated models, we surmise that the local variability of the model predominates. 
Nevertheless, from a practical point of view, taking a distance function depending not only on the $z_j$ in the definition of the weights may alleviate this effect. 

\begin{figure}
    \centering
    \includegraphics[scale=0.2]{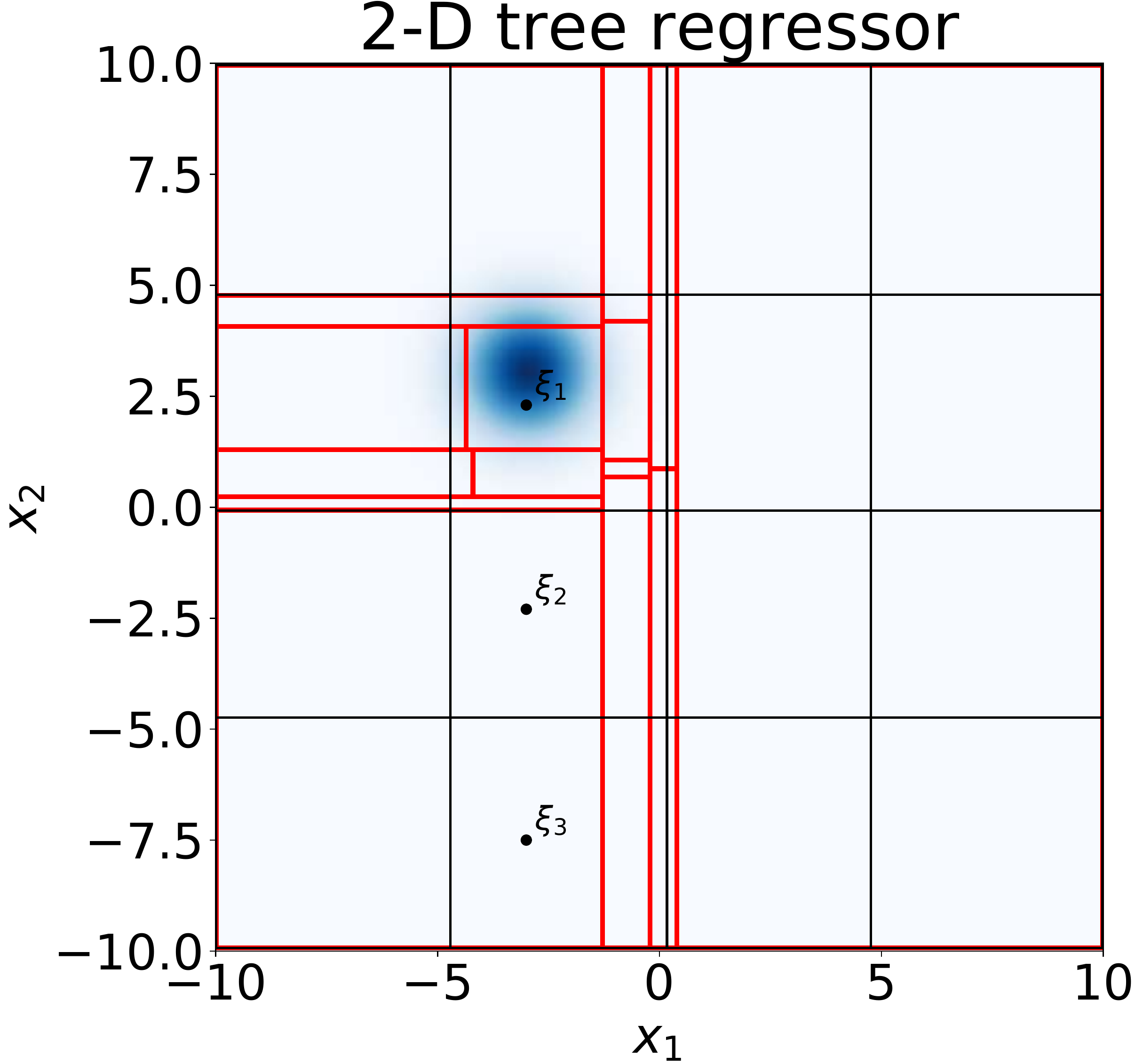}
    \includegraphics[scale=0.2]{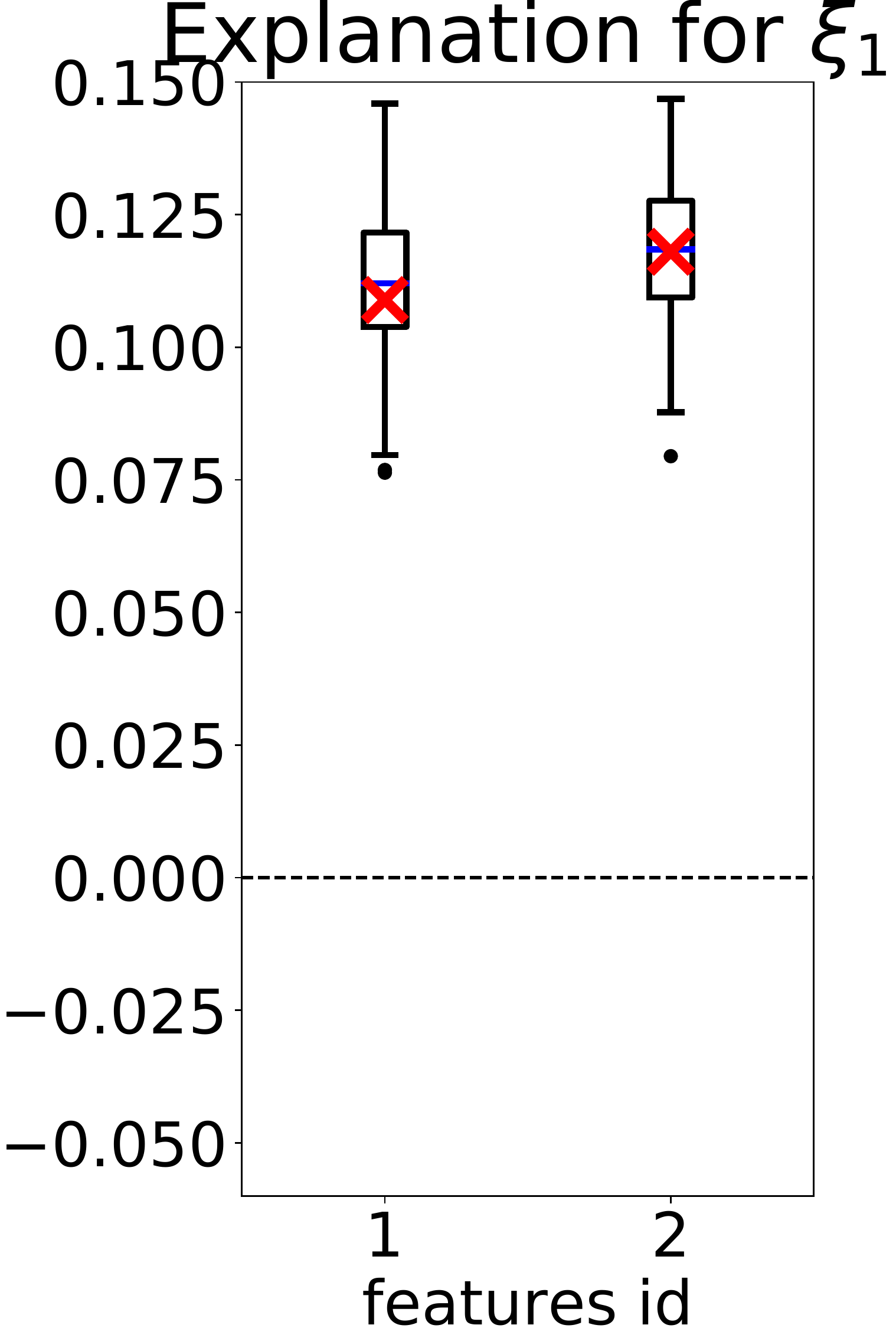}
    \includegraphics[scale=0.2]{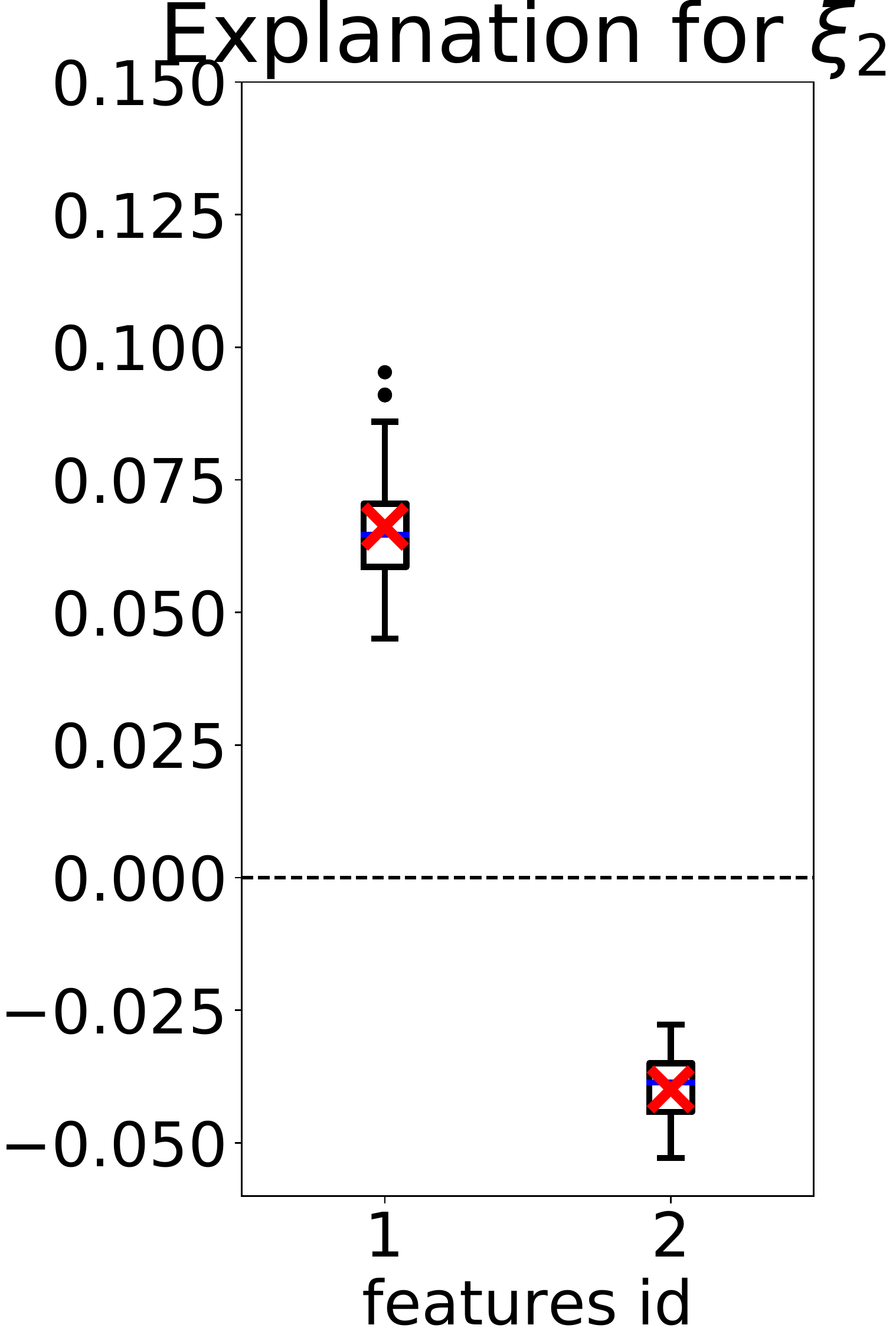}
    \includegraphics[scale=0.2]{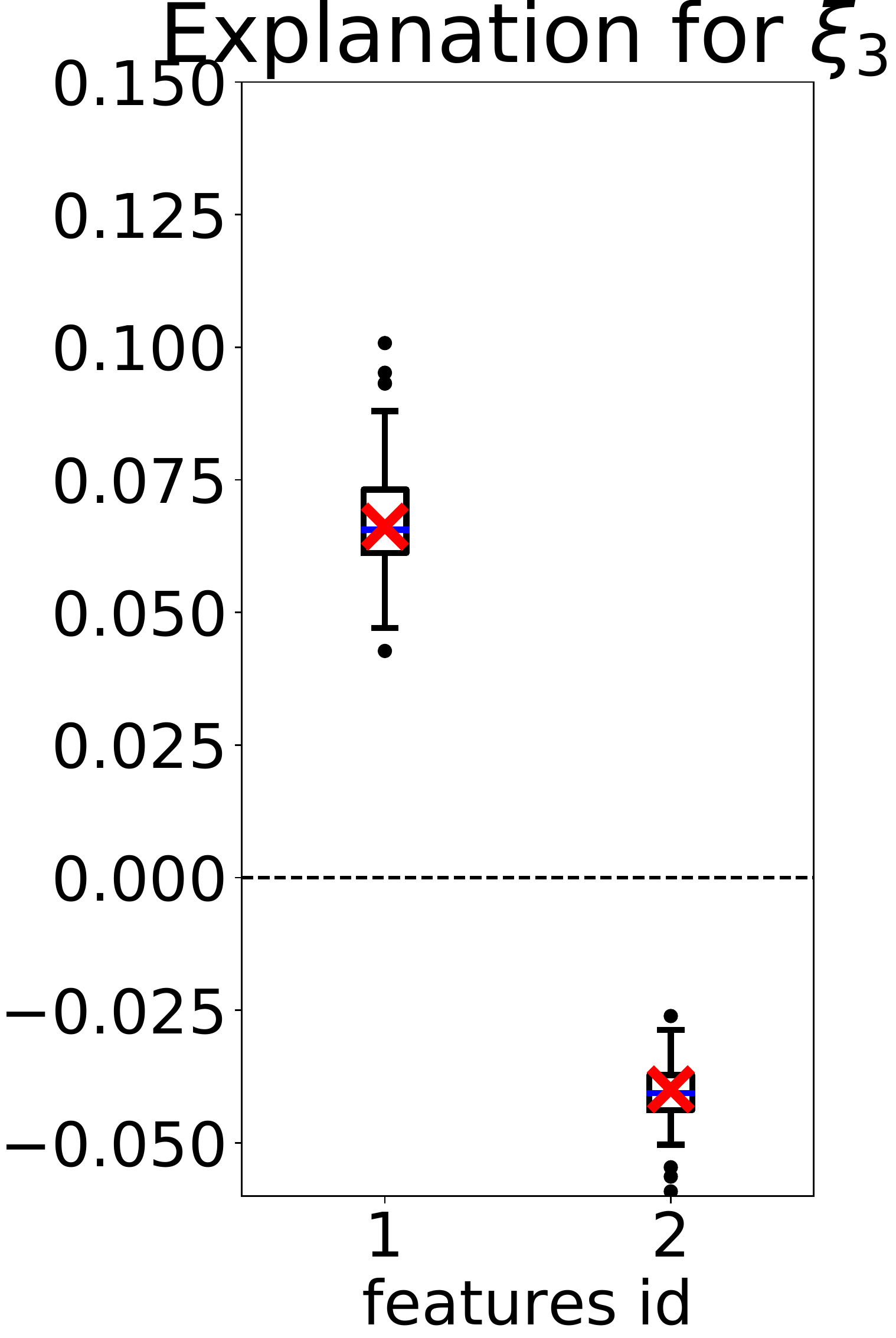}
    \caption{\label{fig:tree-artifacts}Alignments create artifacts. \emph{Left panel:} We trained a $2$-dimensional tree regressor $f$ on a bump function (in blue) with uniform training data on $[-10,10]^2$. The partition associated to $f$ is in red, the partition given to Tabular LIME in black. \emph{Right panels:} explanations given by Tabular LIME for $\xi_i$, $i\in\{1,2,3\}$, theory in red. As predicted, the explanations are high for $\xi_1$: the weighted mean of $f$ is higher near $\xi_1$ than everywhere else. However, the explanation for $\xi_2$ and $\xi_3$ remains high in the first coordinate. Despite being far away, the bin distance is the same due to the alignment. }
\end{figure}

We want to conclude this section by noticing that Corollary~\ref{cor:beta-computation-tree} also applies to a random forest regressor build upon CART trees: 
one only needs to consider the partition obtained by intersecting all the associated tree partitions. 


\subsubsection{Radial basis function kernel}

An important class of examples is given by kernel regressors, a family of predictors containing, for instance, kernel support vector machines (see \citet{Hastie_Tibshirani_Friedman_2001} Section~12.3 and \citet{Schoelkopf_Smola_2001}). 
In all these cases, $f$, the model to explain can be written in the form
\[
f(x) = \sum_{i=1}^{m} \alpha_i k_{\gamma_i}(x,\zeta_i)
\, ,
\]
where $k$ is a positive semi-definite kernel, $\gamma_i$ is a positive scale  parameter, $\alpha\in\Reals^m$ are some real coefficients, and the $\zeta_i$s are support points. 
By linearity of the explanation, we can focus on the $\beta^f$  with $f = k_{\gamma_i}(\cdot,\zeta_i)$. 
For one of the most intensively used kernel, the \emph{Gaussian kernel} (also called radial basis function kernel),  $x\mapsto k_{\gamma}(x,\zeta)$ is \emph{multiplicative}. 
Therefore, we can compute in closed-form the associated interpretable coefficients.
Although, it is more challenging to obtain a statement as readable as Proposition~\ref{prop:beta-computation-indicator}, and we only give the computation of the $\ew_{j,b}$ in this case. 
Namely, we focus on $f(x) = \exp{\frac{-\norm{x-\zeta}^2}{2\gamma^2}}$ for some positive $\gamma$ and a fixed $\zeta\in\Reals^d$. 

\begin{myproposition}[Gaussian kernel computations]
\label{prop:ews-computation-gaussian-kernel}
For any $1\leq j\leq d$, define $k_j\defeq \exp{\frac{-(\cdot -\zeta_j)^2}{2\gamma^2}}$. 
Let us set 
\[
\mtilde_{j,b} \defeq \frac{\gamma^2 \mu_{j,b} + \sigma_{j,b}^2 \zeta_j}{\gamma^2 + \sigma_{j,b}^2}
\, ,
\quad 
\text{and} 
\quad 
\stilde_{j,b}^2 \defeq \frac{\sigma_{j,b}^2\gamma^2}{\sigma_{j,b}^2 + \gamma^2}
\, .
\]
Then, for any $1\leq j\leq d$ and $1\leq b\leq p$,  $\ew_{j,b}=\ek_{j,b}$ if $b=\bxi_j$ and $\exps{\frac{-1}{2\nu^2}}\ek_{j,b}$ otherwise, where we defined 
\begin{equation}
\label{eq:def-ek}
\ek_{j,b} \defeq  \frac{\stilde_{j,b}}{\sigma_{j,b}} \frac{\Phi\left(\frac{q_{j,b}-\mtilde_{j,b}}{\stilde_{j,b}}\right) - \Phi\left(\frac{q_{j,b-1}-\mtilde_{j,b}}{\stilde_{j,b}}\right)}{\Phi\left(\frac{q_{j,b}-\mu_{j,b}}{\sigma_{j,b}}\right) - \Phi\left(\frac{q_{j,b-1}-\mu_{j,b}}{\sigma_{j,b}}\right)} \exps{\frac{-(\mu_{j,b}-\zeta_j)^2}{2(\gamma^2+\sigma_{j,b}^2)}}
\, .
\end{equation}
\end{myproposition}

\begin{proof}
We see that $f$ is multiplicative, with $f_j(x)=k_j$. 
Therefore we can use Proposition~\ref{prop:beta-computation-multiplicative-default} to compute the associated interpretable coefficients. 
Computing $\ew_{j,b}$ requires to integrate $f_j$ with respect to the Gaussian measure. 
We can split the square as
\[
\frac{(x-\mu_{j,b})^2}{2\sigma_{j,b}^2} + \frac{(x-\zeta_j)^2}{2\gamma^2} = \frac{(x-\mtilde_{j,b})^2}{2\stilde_{j,b}^2}
\, ,
\]
Let $x\sim\truncnorm{\mu_{j,b},\sigma_{j,b}^2,q_{j,b-1},q_{j,b}}$. 
We deduce that 
\begin{equation*}
\expec{\exp{\frac{-(x-\zeta_j)^2}{2\gamma^2}}} =  \ek_{j,b}
\, .
\end{equation*}
From the previous display we can deduce the value of $\ew_{j,b}^{k_j}$ for any $1\leq j\leq d$ and $1\leq b\leq p$ and conclude the proof. 
\end{proof}

The value of the $\littlec_j^{f_j}$ coefficients follows, which gives us a closed-formula for $\beta^f_j$. 
Figure~\ref{fig:gaussian-kernel-explanation} demonstrates how our theoretical predictions match practice in dimension $10$.

Without giving the explicit formula for $\beta^f$, we can see that, as in the previous section, for a given $j$, $\beta_j^f$ is relatively larger than the other interpretable coefficients if 
\[
\ek_{j,\bxi_j} \gg \frac{1}{p-1}\sum_{b\neq \bxi_j}\ek_{j,b}
\, .
\]
Now, in the typical situation, $\gamma \ll \sigma_{j,b}$ for any $1\leq j\leq d$ and $1\leq b\leq p$. 
Indeed, the bins are usually quite large (containing approximately $1/p$-th of the training data in any given direction), whereas $\gamma$ should be rather small in order to encompass the small-scale variations of the data. 
In this case, whenever $\mu_{j,b}$ is far away from $\zeta_j$ with respect to $\gamma$, the exponential term in Eq.~\eqref{eq:def-ek} vanishes, and $\ek_{j,b}\approx 0$. 
Therefore, we end up in a situation very similar to the indicator function case treated in the previous section, where \textbf{the only large $\ek_{j,b}$ coefficients are the one for which the associated bin contains $\zeta_j$.}  
This has similar consequences: the explanations are rather small when $\xi$ is far away from $\zeta$. 
With the same exception: when $\xi_j$ is ``aligned'' with $\zeta_j$, $\ek_{j,b}$ is large and this yields artifacts in the explanation, which we demonstrate in Figure~\ref{fig:vector-field}. 
Again, we do not think that this is a desirable behavior. 
One would prefer to see $\beta^f_j$ pointing in the direction of $\zeta_j$, or at the very least small $\beta^f_j$ since the function is very flat when far from $\zeta$ with respect to $\nu$ (at least in the Gaussian kernel case). 

\begin{figure}
	\begin{center}
\includegraphics[scale=0.29]{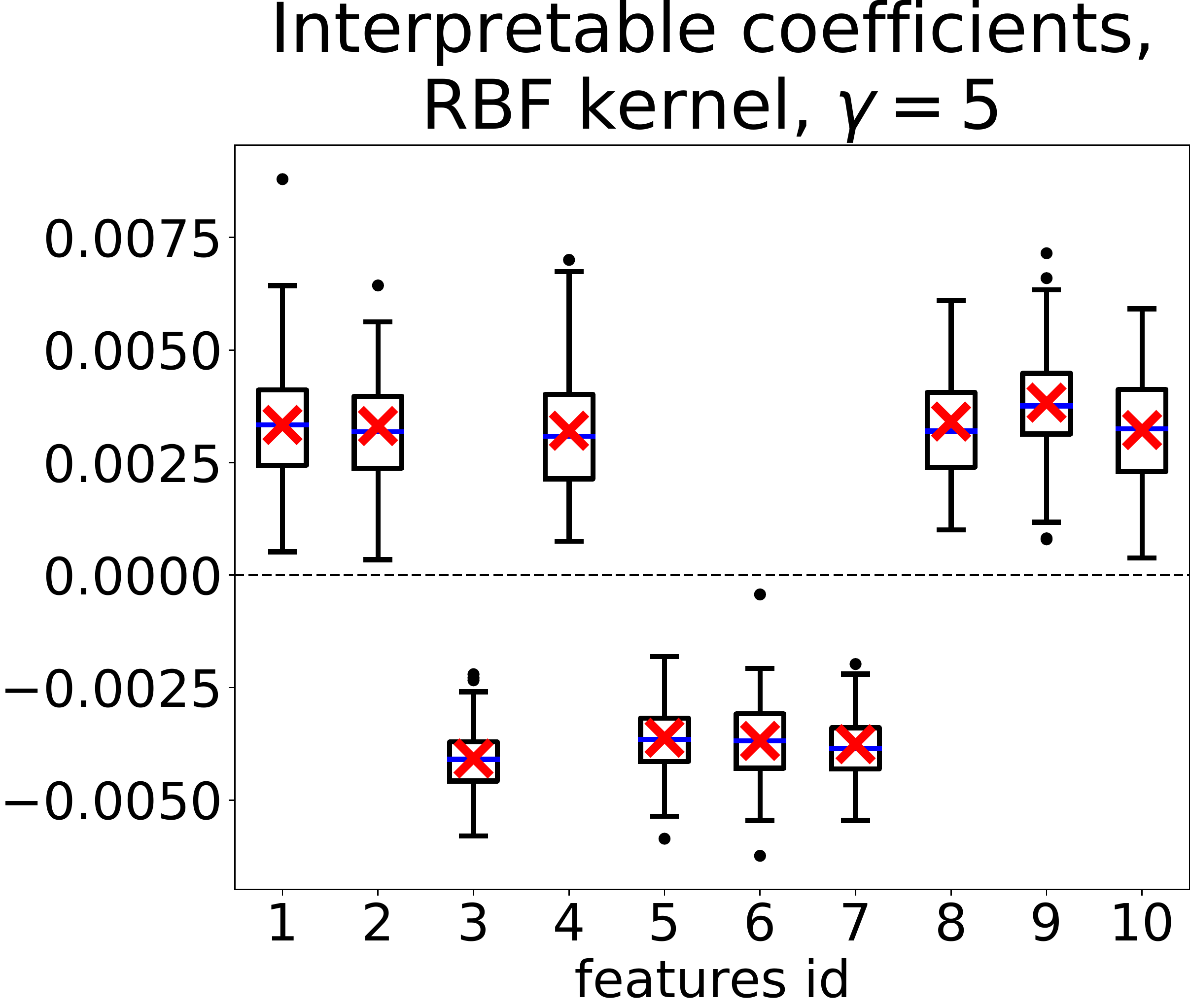}\hspace{1cm}
\includegraphics[scale=0.29]{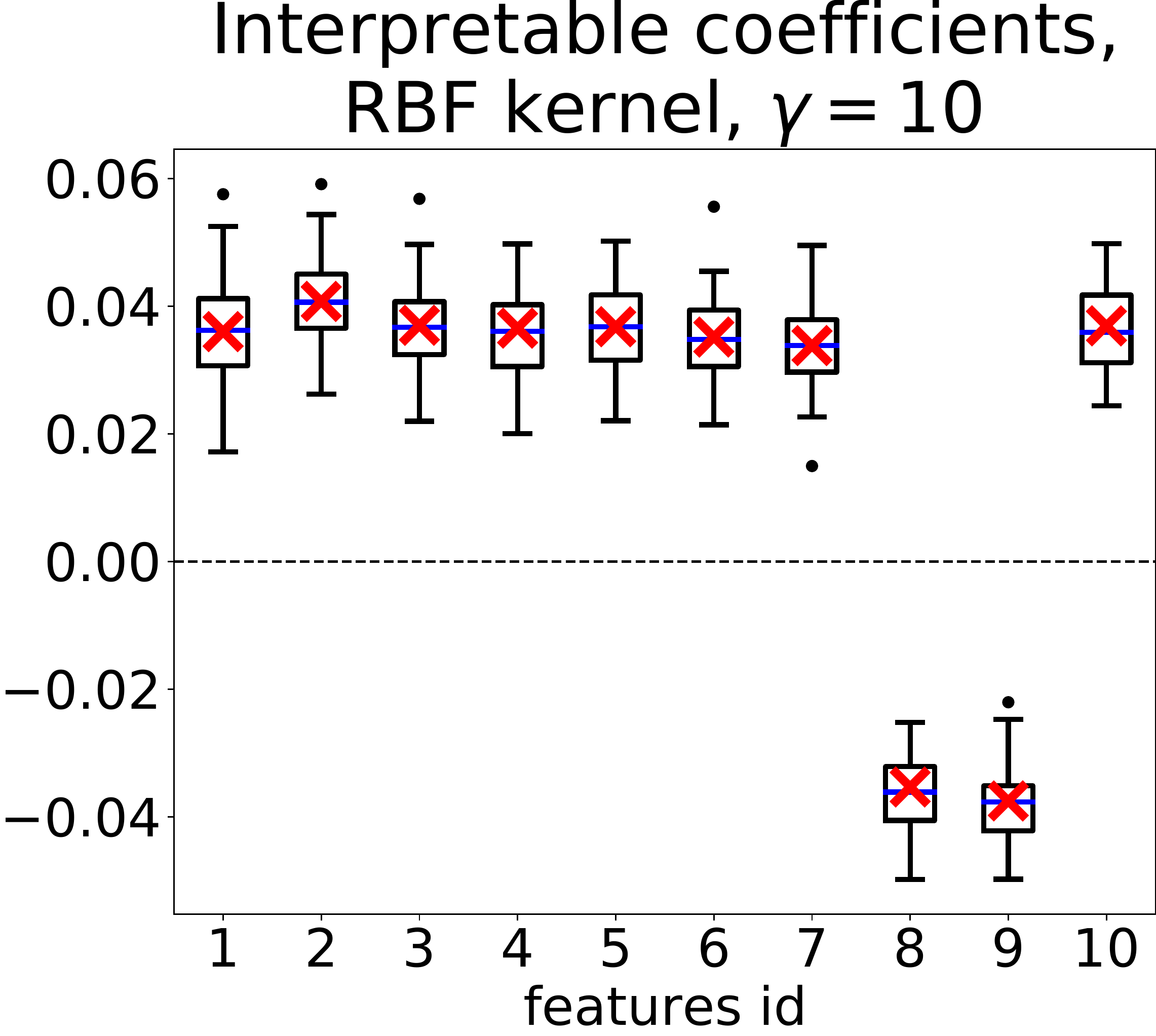}
	\end{center}
	\caption{\label{fig:gaussian-kernel-explanation}Theory vs practice for $f$ given by a Gaussian kernel with bandwidth parameter $\gamma=10$, in dimension $10$ (top panel) and $20$ (bottom panel). We see that our theoretical predictions (red crosses) match perfectly the values of the interpretable coefficients given by Tabular LIME with default settings ($100$ repetitions, black whisker boxes).}
\end{figure}

\begin{figure}
	\begin{center}
		\includegraphics[scale=0.28]{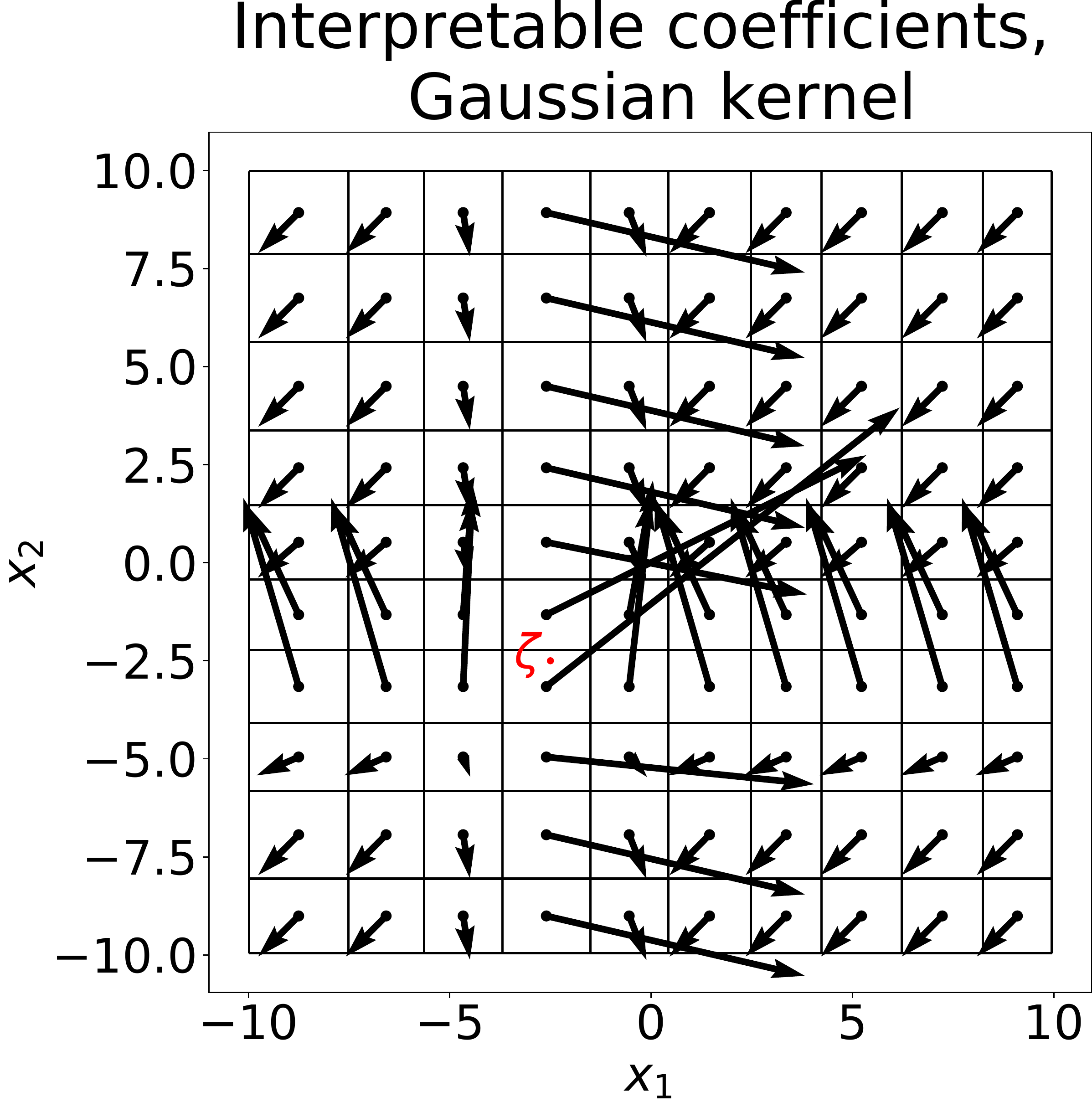}\hspace{1cm}
		\includegraphics[scale=0.28]{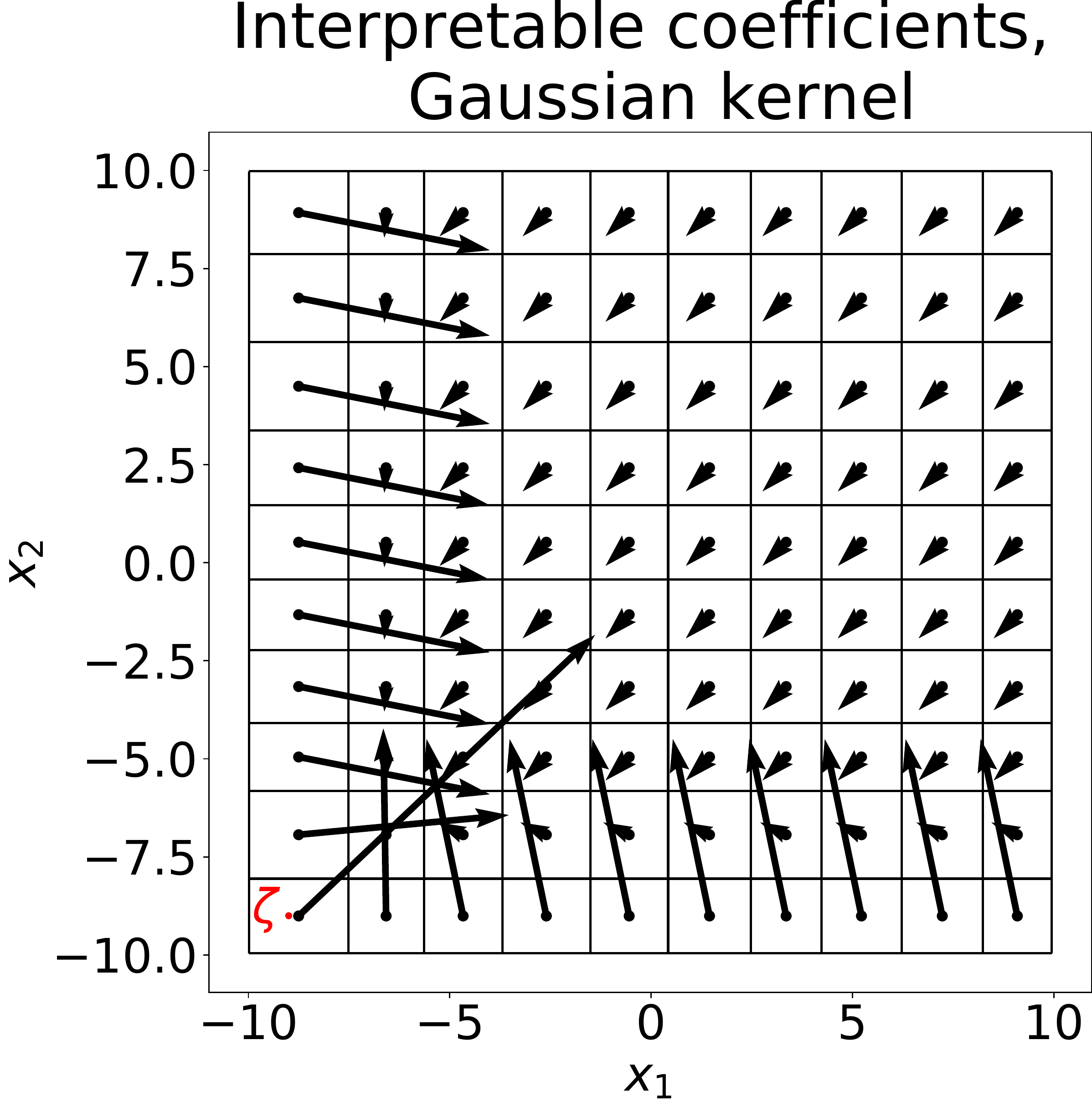}
	\end{center}
	\caption{\label{fig:vector-field}Explanations for a Gaussian kernel in dimension $2$ ($\gamma=1$). We consider uniformly distributed training data on $[-10,10]^2$ with $p=10$ bins along each coordinate. 
We plot the explanation for each center of a $2$-dimensional bin. 
		As predicted by the theory, the explanations are very small if $\xi$ is far away from $\zeta$, excepted when $\xi_j$ falls into the same bin as $\zeta_j$. }
\end{figure}


\subsection{Dummy features}
\label{section:sparse}

As a last application of Theorem~\ref{th:betahat-concentration-general-f-default-weights}, we return to the case where one or more coordinates are unused by the model.

\begin{myproposition}[Dummy features]
\label{prop:ignoring-unused-coordinates}
Assume that $f$ satisfies Assumption~\ref{assump:bounded}. 
	Assume further that there exists a mapping $g:\Reals^s\to\Reals$ such that
	\[
	\forall x\in\Reals^d, \quad f(x) = g(x_{j_1},\ldots,x_{j_s})
	\, ,
	\]
	where $S\defeq \{j_1,\ldots,j_s\}$ is a subset of $\{1,\ldots,d\}$ of cardinality $s$. 
	Let $S$ be the set of indices relevant for $f$ and $\Sbar \defeq \{1,\ldots,d\}\setminus S$. 
	Then Theorem~\ref{th:betahat-concentration-general-f-default-weights} holds with $\beta_j^f=0$ for any $j\in\Sbar$. 
\end{myproposition}

Proposition~\ref{prop:ignoring-unused-coordinates} encompasses a very desirable trait of any feature attribution method: if a coordinate is not used by the model at all, then this coordinate has no influence on the prediction given by the model and should receive weight $0$ in the explanation. 
As mentioned earlier, this is the \emph{dummy} axiom in \citet{Sundararajan_Najmi_2020}. 
We illustrate Proposition~\ref{prop:ignoring-unused-coordinates} in Figure~\ref{fig:ignore-default-non-linear-ls}. 

\begin{figure}
	\begin{center}
		\includegraphics[scale=0.3]{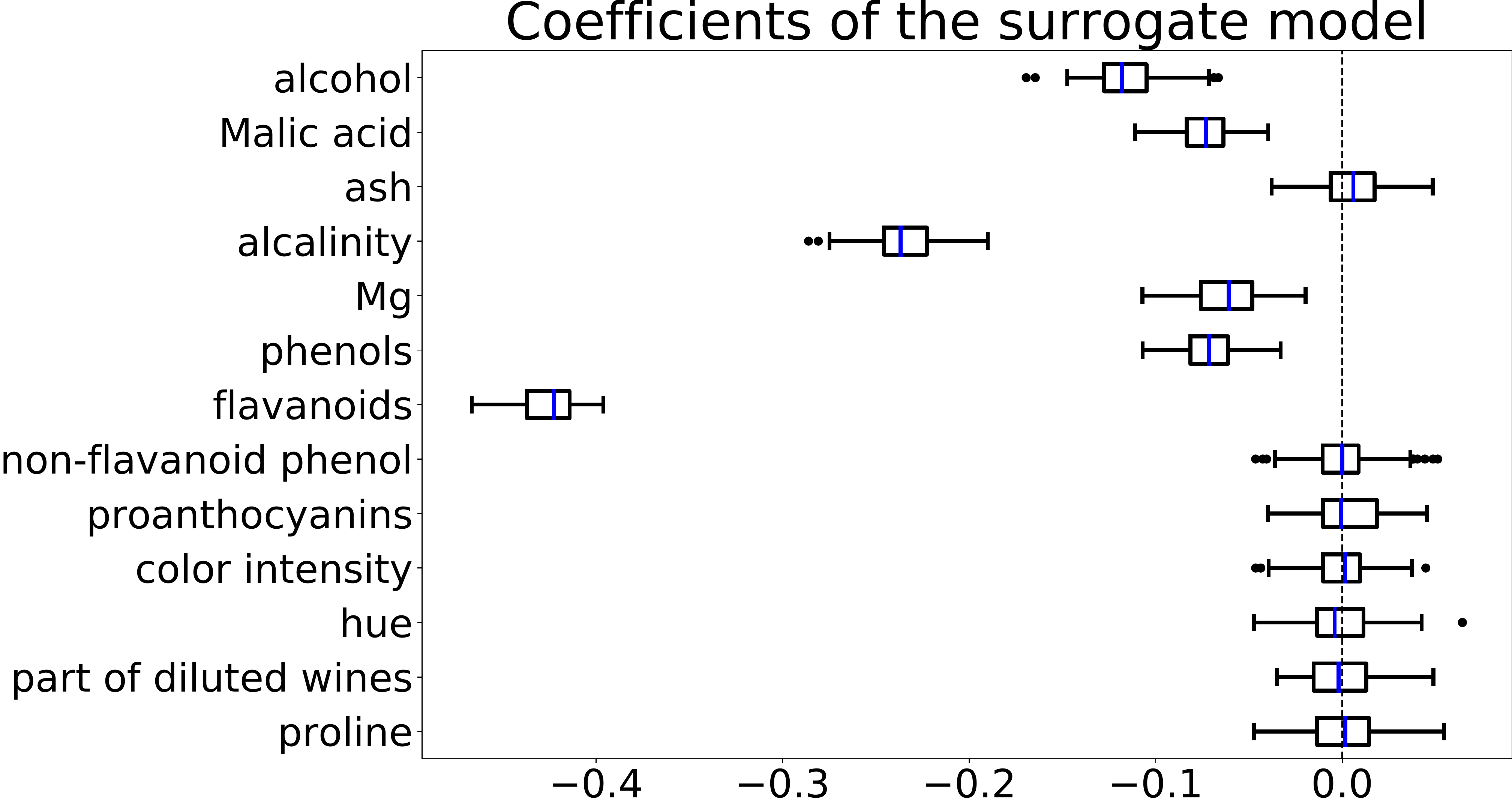}
	\end{center}
\caption{\label{fig:ignore-default-non-linear-ls}Dummy features. In this experiment, we run Tabular LIME  on a kernel ridge regressor (Gaussian kernel with bandwidth set to $5$) not using the last $6$ features. The data is Wine quality. As predicted by Proposition~\ref{prop:ignoring-unused-coordinates}, the interpretable coefficients for the last six coordinates are zero up to noise from the sampling: LIME ignore these coordinates in the explanation.}
\end{figure}

While Proposition~\ref{prop:ignoring-unused-coordinates} is true for more general weights (see Appendix~\ref{section:ignoring-unused-coordinates-general-weights}), we provide here a proof for the default weights. 
Our goal is, again, to demonstrate how Theorem~\ref{th:betahat-concentration-general-f-default-weights} can be used to get interesting statements on the explanations provided by Tabular LIME when simple structural assumptions are made on $f$.

\begin{proof}
We want to specialize Theorem~\ref{th:betahat-concentration-general-f-default-weights} in the specific case where $f$ does not depend on a subset of coordinates. 
As we have seen before, the main challenge in using Theorem~\ref{th:betahat-concentration-general-f-default-weights} is to compute the expectations $\expec{\pi f(x)}$ and $\expec{\pi z_jf(x)}$. 
The main idea of the proof is to regroup in these expectation computations the parts depending on $j\in S$ and those depending on $j\notin S$. 
We will see that some cancellations happen afterwards.

Let us turn first to the computation of $\expec{\pi f(x)}$. 
By definition of the weights (Eq.~\eqref{eq:def-weights-default}), we have
\[
\expec{\pi f(x)} = \expec{\prod_{k=1}^{d}\exps{\frac{-(1-z_k)^2}{2\nu^2}} f(x)} 
\, .
\]
Using successively the properties of $f$ and the independence of the $x_j$s, we can rewrite the previous display as
\[
\prod_{k\in\Sbar}\expec{\exps{\frac{-(1-z_k)^2}{2\nu^2}}} \cdot \expec{\prod_{k\in S}\exps{\frac{-(1-z_k)^2}{2\nu^2}} g(x_{j_1},\ldots,x_{j_s})}
\, .
\]
We recognize the definition of the normalization constant $\littlec$. 
Setting 
\[
G \defeq \expec{\prod_{k\in S}\exps{\frac{-(1-z_k)^2}{2\nu^2}} g(x_{j_1},\ldots,x_{j_s})}
\, ,
\] 
we have proved that
\begin{equation}
\label{eq:aux-sparse-1}
\expec{\pi f(x)} = \littlec^{d-s} \cdot G
\, .
\end{equation}

Let $j\in \{1\ldots,d\}\setminus S$. 
The computation of $\expec{\pi z_j f(x)}$ is similar. 
Indeed, by definition of the weights, we can write
\[
\expec{\pi z_j f(x)} = \expec{\prod_{k=1}^{d}\exps{\frac{-(1-z_k)^2}{2\nu^2}} z_{j}f(x)}
\, .
\]
Again, since the $x_j$s are independent and $f$ satisfies Assumption~\ref{ass:multiplicative}, we rewrite the previous display as
\[
\prod_{k\in\Sbar \setminus\{j\}} \expec{\exps{\frac{-(1-z_k)^2}{2\nu^2}}} \cdot \expec{\exps{\frac{-(1-z_j)^2}{2\nu^2}}z_{j}}\cdot \expec{\prod_{k\in S} \exps{\frac{-(1-z_k)^2}{2\nu^2}} g(x_{j_1},\ldots,x_{j_s})}
\, .
\]
We have proved that
\begin{equation}
\label{eq:aux-sparse-2}
\expec{\pi z_j f(x)} = \frac{\littlec^{d-s-1} G }{p}
\, .
\end{equation}

Finally, according to Theorem~\ref{th:betahat-concentration-general-f-default-weights}, 
\[
\beta^f_j = \bigc^{-1}\frac{p\littlec}{p\littlec-1}\biggl(-\expec{\pi f(x)} + p\littlec \expec{\pi z_{j}f(x)} \biggr) 
\, .
\]
Plugging Eqs.~\eqref{eq:aux-sparse-1} and~\eqref{eq:aux-sparse-2} in the previous display, we obtain that 
\[
\beta^f_j = \bigc^{-1}\frac{p\littlec}{p\littlec-1}\left(-\littlec^{d-s} \cdot G + p\littlec \frac{\littlec^{d-s-1} G }{p} \right) = 0
\, .
\]
\end{proof}


\section{Conclusion: strengths and weaknesses of Tabular LIME}
\label{section:conclusion}

In this paper, we provided a thorough analysis of Tabular LIME. 
In particular, we show that, in the large sample size, the interpretable coefficients provided by Tabular LIME can be obtained in an explicit way as a function of the algorithm parameters and some expectation computations related to the black-box model. 
Our experiments show that our theoretical analysis yields predictions that are very close to empirical results for the default implementation. 
This allowed us to provide very precise insights on the inner working of Tabular LIME, revealing some desirable behavior (linear models are explained linearly, unused coordinates are provably ignored, and kernel functions yield flat explanations far away from the bump), and some not so desirable properties (artifacts appear when explaining indicator and kernel functions). 

We believe that the present work paves the way for a better theoretical understanding of Tabular LIME in numerous simple cases. 
Using the machinery presented here, one can check how a particular model interacts with Tabular LIME, at the price of reasonable computations. 
It is then possible to check whether some basic sanity checks are satisfied, helping us to decide whether to use Tabular LIME in this case. 
We also hope that the insights presented here can help us to design better interpretability methods by fixing the flaws of Tabular LIME. 
For instance, our analysis is valid for a large class of weights: for a given class of models, it could be the case that choosing non-default weights is more adequate. 




\acks{The authors want to thank Sebastian Bordt, Romaric Gaudel, Michael Lohaus, and Martin Pawelcyk for constructive discussions.}


%

\vskip 0.2in
\bibliography{../lime}

\newpage

\appendix


\begin{center}
{\Huge Appendix}
\end{center}


In this Appendix we collect all proofs and additional results. 
Appendix~\ref{section:general-weights} generalizes the notation to weights generalizing the default weights.  
In Appendix~\ref{section:study-of-sigma}, \ref{section:gamma}, and \ref{section:beta} we study successively $\Sigma$, $\Gamma^f$, and $\beta^f$. 
The proof of Proposition~\ref{prop:explanation-stability-default-weights} is presented in Appendix~\ref{section:proof-stability-corollary}, and the proof of Proposition~\ref{prop:beta-computation-indicator} is given in Appendix~\ref{section:proof-computation-beta-indicator}. 
We turn to the concentration of $\Sigmahat_n$ and $\Gammahat_n$ in Appendix~\ref{section:concentration-sigmahat} and~\ref{section:concentration-gammahat}. 
Our main result is proved in Appendix~\ref{section:concentration-betahat}. 
An extension of Proposition~\ref{prop:ignoring-unused-coordinates} for general weights is presented in Appendix~\ref{section:ignoring-unused-coordinates-general-weights}. 
Finally, technical results are collected in Appendix~\ref{section:technical-results}. 

\section{General weights}
\label{section:general-weights}

As discussed throughout the main paper, our analysis holds for more general weights than the default weighting scheme. 
Indeed, it will become clear in our analysis that the proof scheme works provided that the weight $\pi$ associated to $x$ has some multiplicative structure. 
More precisely, from now on we assume that
\begin{equation}
\label{eq:definition-weights}
\pi(x) \defeq \exp{\frac{-1}{2\nu^2}\sum_{j=1}^{d}(\tau_j(\xi_j) - \tau_j(x_{j}))^2}
\, ,
\end{equation}
where $\nu >0$ is, as before, the bandwidth parameter, and for any $1\leq j\leq d$, $\tau_j : \Reals \to \Reals$ is an arbitrary fixed function that can depend on the input of the algorithm. 
We will refer to these weights as \emph{general weights} in the following. 
We make the following assumption on the $\tau_j$s:
\begin{assumption}
	\label{ass:bounded-weights}
	For any $1\leq j\leq d$ and any $x,y\in[0,1]$, we have
	\[
	\abs{\tau_j(x) - \tau_j(y)} \leq 1
	\, .
	\]
\end{assumption}
This assumption is mainly needed to control the spectrum of $\Sigma$. 

General weights generalize two important examples, which we describe below. 

\begin{myexample}[Weights in the default implementation]
	In the default implementation of LIME, for a given example $x_i$, we have defined $\pi_i=\exp{-\norm{\Indic - z_i}^2/(2\nu^2)}$ in Eq~\eqref{eq:def-weights-default}. 
	This definition of the weights amounts to taking 
	\[
	\tau_j(x) = \indic{x\in\bigl[q_{j,\bxi_j-1},q_{j,\bxi_j}\bigr]}
	\, .
	\]
	By definition of the $\tau_j$s in this case, Assumption~\ref{ass:bounded-weights} is satisfied. 
\end{myexample}

\begin{myexample}[Smooth weights]
	In \citet{Garreau_von_Luxburg_2020}, the weights were defined as $\pi_i = \exp{-\norm{\xi - x_i}^2/(2\nu^2)}$. 
	This definition of the weights corresponds to taking $\tau_j(x)=x$ for any $1\leq j\leq d$. 
	If the data is bounded (say by $1$), then the boundedness of the $\tau_j$s is satisfied in this case. 
\end{myexample}


\section{The study of $\Sigma$}
\label{section:study-of-sigma}

In this section, we study the matrix $\Sigma$ for general weights satisfying Eq.~\eqref{eq:definition-weights}. 
We begin by generalizing the notation $\ew_{j,b}^\psi$ introduced at the beginning of Section~\ref{section:discussion} to general weights. 
For any $\psi:\Reals\to\Reals$, we set
\begin{equation}
\label{eq:def-ew-general-weights}
\ew_{j,b}^\psi \defeq \condexpec{\psi(x_{ij})\exp{\frac{-1}{2\nu^2}(\tau_j(\xi_j) - \tau_j(x_{ij}))^2}}{b_{ij} = b}
\, .
\end{equation}
As before, when $\psi = 1$, we just write $\ew_{j,b}$ instead of $\ew_{j,b}^1$, and when $\psi=\id$, we write $\ew_{j,b}^\times$. 
We also extend the definition of the normalization constants
\[
\littlec_j^\psi \defeq \frac{1}{p}\sum_{b=1}^{p} \ew_{j,b}^\psi
\, ,
\]
and define $\littlec_j \defeq \littlec_{j}^1$. 
Finally, we also extend the definition of
\begin{equation} 
\label{eq:normalization-constant}
\bigc\defeq \prod_{j=1}^d \littlec_j
\, .
\end{equation}
As pointed out in Section~\ref{section:discussion}, the constant $\bigc$ is ubiquitous in our computations. 
We can give an exact expression for $\bigc$ in many cases, see Example~\ref{ex:basic-computations-old-analysis}. 
In any case, since the $\ew_{j,b}$s belong to $[0,1]$, the same holds for the $\littlec_j$s and for $\bigc$. 

\begin{myremark}
	\label{remark:bandwidth-ew}
	Whenever $\psi$ is regular enough, the behavior of these coefficients in the small and large bandwidth is quite straightforward. 
	Namely:
	\begin{itemize}
		\item if $\nu\to 0$, then $\ew_{j,b}^\psi\to 0$. As a consequence, $\littlec_j^\psi\to 0$ as well;
		\item if $\nu\to +\infty$, then $\ew_{j,b}^\psi \to \condexpec{\psi(x_{ij})}{b_{ij}=b}$. 
	\end{itemize}
\end{myremark}

In Section~\ref{section:discussion}, we computed these coefficients for the default weights. 
Let us redo this exercise for the weighting scheme of \citet{Garreau_von_Luxburg_2020}.

\begin{myexample}[Basic computations, smooth weights]
	\label{ex:basic-computations-old-analysis}
	Let us compute the $\ew_{j,b}$ when $\tau_j=\id$. 
	We write 
	\begin{align*}
	\ew_{j,b} &= \condexpec{\exp{\frac{-(x_{ij}-\xi_j)^2}{2\nu^2}}}{b_{ij} = b} \tag{Eq. \eqref{eq:def-ew-general-weights}} \\
	&= \frac{1}{\sigma_{jb}\sqrt{2\pi}}\cdot \frac{1}{\Phi(r_{j,b}) - \Phi(\ell_{j,b})} \int_{q_{j,b-1}}^{q_{j,b}}\exp{\frac{-(x-\mu_{j,b})^2}{2\sigma_{j,b}^2} + \frac{-(x-\xi_j)^2}{2\nu^2}}\Diff t \tag{Eq. \eqref{eq:tn-density}} \\
	&= \frac{1}{\Phi(r_{j,b}) - \Phi(\ell_{j,b})} \cdot \frac{\nu\exps{\frac{-(\xi_j-\mu_{j,b})^2}{2(\nu^2+\sigma_{j,b}^2)}}}{\sqrt{\nu^2+\sigma^2_{j,b}}}\left[\frac{1}{2}\erfun{\frac{\nu^2(x-\mu_{j,b})+\sigma_{j,b}^2(x-\xi_j)}{\nu\sigma_{j,b}\sqrt{2(\nu^2+\sigma_{j,b}^2)}}} \right]_{q_{j,b-1}}^{q_{j,b}}
	\, ,
	\end{align*}
	where we used Lemma~11.1 in \citet{Garreau_von_Luxburg_2020} in the last display.
	For any $1\leq j\leq d$ and $1\leq b\leq p$, let us set 
	\[
	m_{j,b}\defeq \frac{\nu^2\mu_{j,b} + \sigma_{j,b}^2\xi_j}{\nu^2+\sigma_{j,b}^2}
	\quad \text{and}\quad
	s_{j,b}^2 \defeq \frac{\nu^2\sigma_{j,b}^2}{\nu^2+\sigma_{j,b}^2}
	\, .
	\]
	Then it is straightforward to show that 
	\[
	\frac{\nu^2(x-\mu_{j,b})+\sigma_{j,b}^2(x-\xi_j)}{\nu\sigma_{j,b}\sqrt{2(\nu^2+\sigma_{j,b}^2)}} = \frac{x-m_{j,b}}{\sqrt{2}s_{j,b}}
	\, ,
	\]
	and the expression of $\ew_{j,b}$ simplifies slightly:
	\[
	\ew_{j,b} = \frac{1}{\Phi(r_{j,b}) - \Phi(\ell_{j,b})} \cdot \frac{\nu\exps{\frac{-(\xi_j-\mu_{j,b})^2}{2(\nu^2+\sigma_{j,b}^2)}}}{\sqrt{\nu^2+\sigma^2_{j,b}}}\left[\frac{1}{2}\erfun{\frac{x-m_{j,b}}{\sqrt{2}s_{j,b}} } \right]_{q_{j,b-1}}^{q_{j,b}}
	\, .
	\]
	Now recall that \citet{Garreau_von_Luxburg_2020} chose to consider $\mu_{j,b}=\mu_j$ and $\sigma_{j,b}=\sigma$ constant. 
	As a consequence, $m_{j,b}$ does not depend on $b$ anymore, and $s_{j,b}$ is a constant equal to $s\defeq \nu\sigma/\sqrt{\nu^2 + \sigma^2}$. 
	Moreover, the $q_{j,b}$ are, in this case, the normalized Gaussian quantiles, and therefore
	\[
	\Phi(r_{j,b}) - \Phi(\ell_{j,b}) = \frac{1}{p}
	\, .
	\]
	We deduce that 
	\[
	\ew_{j,b} = \frac{p\nu\exps{\frac{-(\xi_j-\mu_{j})^2}{2(\nu^2+\sigma^2)}}}{\sqrt{\nu^2+\sigma^2}}\left[\frac{1}{2}\erfun{\frac{x-m_{j}}{\sqrt{2}s} } \right]_{q_{j,b-1}}^{q_{j,b}}
	\, ,
	\]
	and 
	\[
	\littlec_j = \frac{\nu}{\sqrt{\nu^2+\sigma^2}} \exp{\frac{-(\xi_j-\mu_j)^2}{2(\nu^2+\sigma^2)}}
	\, , 
	\]
	which yields
	\begin{equation}
	\label{eq:def-bigc-old-analysis}
	\bigc = \left( \frac{\nu^2}{\nu^2+\sigma^2}\right)^{d/2} \exp{\frac{-\norm{\xi-\mu}^2}{2(\nu^2+\sigma^2)}}
	\, .
	\end{equation}
	We see that Eq.~\eqref{eq:def-bigc-old-analysis} is indeed Eq.~(7.2) in \citet{Garreau_von_Luxburg_2020}. 
\end{myexample} 


\subsection{Computation of $\Sigma$}

We now have the necessary notation to turn to the computation of $\Sigma$. 
The main idea is that $\Sigma$ depends only on $p$, $\nu$, and the input parameters of the algorithm. 
By chance, the sampling of the new examples make $\Sigma$ simple enough and we can obtain a closed-form expression. 

\begin{myproposition}[Computation of $\Sigma$, general weights]
	\label{prop:sigma-computation}
	Recall that $\Sigmahat_n = \frac{1}{n}Z^\top WZ$ and $\Sigma=\smallexpec{\Sigmahat_n}$, where $Z$ was defined in Section~\ref{section:sampling-procedure}. 
	Then it holds that
	\[
	\Sigma = \bigc 
	\begin{pmatrix}
	1 & \frac{\ew_{1,\bxi_1}}{p\littlec_1} & \cdots & \frac{\ew_{d,\bxi_d}}{p\littlec_d} \\
	\frac{\ew_{1,\bxi_1}}{p\littlec_1} & \frac{\ew_{1,\bxi_1}}{p\littlec_1} & & \frac{\ew_{j,\bxi_j}\ew_{k,\bxi_k}}{p\littlec_jp\littlec_k} \\
	\vdots & & \ddots & \\
	\frac{\ew_{d,\bxi_d}}{p\littlec_d} & \frac{\ew_{j,\bxi_j}\ew_{k,\bxi_k}}{p\littlec_jp\littlec_k}& & \frac{\ew_{d,\bxi_d}}{p\littlec_d}
	\end{pmatrix}
	\] 
\end{myproposition}

Proposition~\ref{prop:sigma-computation} changes slightly if $p$ is no longer constant across dimensions, but we do not tackle this case in the present paper. 

\begin{proof}
We introduce a phantom coordinate in $Z$ since the surrogate linear model uses an offset. 
	A straightforward computation shows that 
	\begin{equation}
	\label{eq:sigmahat-computation}
	\Sigmahat = \frac{1}{n}
	\begin{pmatrix}
	\sum_{i=1}^{n}\pi_i & \sum_{i=1}^{n} \pi_i z_{i1} & \cdots & \sum_{i=1}^{n} \pi_iz_{id} \\
	\sum_{i=1}^{n} \pi_i z_{i1} & \sum_{i=1}^{n} \pi_i z_{i1} & & \sum_{i=1}^{n} \pi_i z_{ij}z_{ik} \\
	\vdots & & \ddots & \\
	\sum_{i=1}^{n} \pi_i z_{id} & \sum_{i=1}^{n} \pi_i z_{ij}z_{ik}& & \sum_{i=1}^{n} \pi_i z_{i1}
	\end{pmatrix}
	\, .
	\end{equation}
	Since the $x_i$s are i.i.d., the result follows from the computation of $\expec{\pi}$, $\expec{\pi z_{j}}$, and $\expec{\pi z_{j}z_{k}}$. 
	This is achieved in Lemma~\ref{lemma:expectation-computations}. 
\end{proof}

As examples, we can explicit the computation of $\Sigma$ for default weights and smooth weights. 

\begin{myexample}[Default implementation, computation of $\Sigma$]
	As we have seen in Section~\ref{section:discussion}, in this case
	\[
	\ew_{j,b} =
	\begin{cases}
	1 &\text{ if }b = \bxi_j \\
	\exps{\frac{-1}{2\nu^2}} &\text{ otherwise,}
	\end{cases}
	\]
	and $\littlec_j=c$ is a constant. 
	Therefore, according to Proposition~\ref{prop:sigma-computation}, the expression of $\Sigma$ simplifies greatly into
	\[
	\Sigma = \littlec^d
	\begin{pmatrix}
	1 & \frac{1}{p\littlec} & \cdots & \frac{1}{p\littlec} \\
	\frac{1}{p\littlec} & \frac{1}{p\littlec} & & \frac{1}{p^2\littlec^2} \\
	\vdots & & \ddots & \\
	\frac{1}{p\littlec} & \frac{1}{p^2c^2}& & \frac{1}{p\littlec}
	\end{pmatrix}
	\, .
	\]
\end{myexample}

\begin{myexample}[Computation of $\Sigma$, smooth weights]
	According to  Example \ref{ex:basic-computations-old-analysis}, we have that
	\[
	\frac{\ew_{j,\bxi_j}}{p\littlec_j} = \left[\frac{1}{2} \erfun{\frac{x-m_j}{\sqrt{2}s}}\right]_{q_{j,\bxi_j-1}}^{q_{j,\bxi_j}}
	\, .
	\]
	We recover the $\alpha_j$ coefficients (Eq.~(7.3) in \citet{Garreau_von_Luxburg_2020}) and as a direct consequence, Lemma~8.1:
	\[
	\Sigma \defeq \bigc\begin{pmatrix}
	1 & \alpha_1 & \cdots & \alpha_d \\
	\alpha_1 & \alpha_1 & & \alpha_i\alpha_j \\
	\vdots & & \ddots & \\
	\alpha_d & \alpha_i\alpha_j & & \alpha_d
	\end{pmatrix}
	\, .
	\]
\end{myexample}


\subsection{Computation of $\Sigma^{-1}$}

The structure of $\Sigma$ allows to invert it in closed-form, even in the case of general weights. 
This is the extent of the next result.

\begin{myproposition}[Computation of $\Sigma^{-1}$, general weights]
	\label{prop:computation-inverse-sigma}
	Let $\Sigma$ be defined as before. 
	Then $\Sigma$ is invertible and 
	\[
	\Sigma^{-1} = \bigc^{-1}
	\begin{pmatrix}
	1+\sum_{j=1}^d \frac{\ew_{j,\bxi_j}}{p\littlec_j-\ew_{j,\bxi_j}} & \frac{-p\littlec_1}{p\littlec_1-\ew_{1,\bxi_1}} & \cdots & \frac{-p\littlec_d}{p\littlec_d - \ew_{d,\bxi_d}} \\
	\frac{-p\littlec_1}{p\littlec_1-\ew_{1,\bxi_1}} & \frac{p^2\littlec_1^2}{\ew_{1,\bxi_1}(p\littlec_1 - \ew_{1,\bxi_1})} & & 0 \\
	\vdots & & \ddots & \\
	\frac{-p\littlec_d}{p\littlec_d-\ew_{d,\bxi_d}} & 0 & & \frac{p^2\littlec_d^2}{\ew_{d,\bxi_d}(p\littlec_d - \ew_{d,\bxi_d})}
	\end{pmatrix}
	\, .
	\]
\end{myproposition}


\begin{proof}
	Set $\alpha_j\defeq \ew_{j,\bxi_j}/(p\littlec_{j})$, and define $\alpha\in\Reals^{d}$ the vector of the $\alpha_j$s. 
	Set $E\defeq 1$, $F\defeq \alpha^\top$, $G\defeq\alpha$, and
	\[
	H \defeq \begin{pmatrix}
	\alpha_1 & & \alpha_j\alpha_k \\
	& \ddots & \\
	\alpha_j\alpha_k & & \alpha_d 
	\end{pmatrix}
	\, .
	\]
	Then $\Sigma$ is a block matrix that can be written $\Sigma = \bigc \begin{bmatrix} E & F \\ G & H\end{bmatrix}$. 
	We notice that 
	\[
	H-GE^{-1}F = \diag{\alpha_1(1-\alpha_1),\ldots,\alpha_d(1-\alpha_d)}
	\, .
	\]
	Since the $\ew_{j,b}$ are positive for any $j,b$, the $\alpha_j$s belong to $(0,1)$ and $H-GE^{-1}F$ is an invertible matrix. 
	According to the block matrix inversion formula,
	\[
	\begin{bmatrix}
	E & F \\
	G & H
	\end{bmatrix}^{-1}
	=
	\begin{pmatrix}
	E^{-1} + E^{-1}F(H-GE^{-1}F)^{-1}GE^{-1} & -E^{-1}F(H-GE^{-1}F)^{-1} \\
	-(H-GE^{-1}F)^{-1}GE^{-1} & (H-GE^{-1}F)^{-1}
	\end{pmatrix}
	\]
	Thus
	\begin{equation}
	\label{eq:inverse-matrix}
	\Sigma^{-1} = \bigc^{-1}
	\begin{pmatrix}
	1+\sum_{j=1}^d \frac{\alpha_j}{1-\alpha_j} & \frac{-1}{1-\alpha_1} & \cdots & \frac{-1}{1-\alpha_d} \\
	\frac{-1}{1-\alpha_1} & \frac{1}{\alpha_1(1-\alpha_1)} & & 0 \\
	\vdots & & \ddots & \\
	\frac{-1}{1-\alpha_d} & 0 & & \frac{1}{\alpha_d(1-\alpha_d)}
	\end{pmatrix}
	\, .
	\end{equation}
	The result follows from the definition of the $\alpha_j$s. 
\end{proof}

We can be explicit about the computation of $\Sigma^{-1}$ in the case of default and smooth weights. 

\begin{myexample}[Computation of $\Sigma^{-1}$, default implementation]
	Using our basic computations in this case in conjunction with Proposition~\ref{prop:computation-inverse-sigma}, we obtain 
	\[
	\Sigma^{-1} = \frac{\littlec^{-d}}{p\littlec - 1}
	\begin{pmatrix}
	p\littlec + d - 1 & -p\littlec & \cdots & -p\littlec \\
	-p\littlec & p^2\littlec^2 & & 0 \\
	\vdots & & \ddots & \\
	-p\littlec & 0 & & p^2\littlec^2
	\end{pmatrix}
	\, .
	\]
\end{myexample}

\begin{myexample}[Computation of $\Sigma^{-1}$, smooth weights]
	From the definition of $\alpha_j$ in Example \ref{ex:basic-computations-old-analysis}, we see that Eq.~\eqref{eq:inverse-matrix} is in fact Lemma~8.2 in \citet{Garreau_von_Luxburg_2020}. 
\end{myexample}


\subsection{Control of $\opnorm{\Sigma^{-1}}$}

Our analysis requires a control of $\opnorm{\Sigma^{-1}}$ when we want to concentrate $\betahat_n$ (see Appendix~\ref{section:concentration-betahat}). 
We show how to achieve this control when the functions $\tau_j$ are bounded. 

\begin{myproposition}[Upper bound on $\opnorm{\Sigma^{-1}}$, general weights]
	\label{prop:upper-bound-operator-norm}
	Let $\Sigma$ be as before, and suppose that $\tau_j$ satisfies Assumption~\ref{ass:bounded-weights}. 
	Then 
	\[
	\opnorm{\Sigma^{-1}} \leq 2\sqrt{2} \bigc^{-1}dp^2 \exps{\frac{2}{\nu^2}}
	\, .
	\]
\end{myproposition}

\begin{proof}
	According to Lemma~\ref{lemma:operator-norm:bounds}, we can find an upper bound for $\opnorm{\Sigma^{-1}}$ just by computing $\frobnorm{\Sigma^{-1}}$. 
	Proposition~\ref{prop:computation-inverse-sigma} gives us 
	\begin{equation}
	\label{eq:aux:upper-bound-frobenius}
	\frobnorm{\bigc\Sigma^{-1}}^2 = \left(1 + \sum_{j=1}^{d}\frac{\ew_{j,\bxi_j}}{p\littlec_{j} - \ew_{j,\bxi_j}}\right)^2 + 2 \sum_{j=1}^{d} \frac{p^2\littlec_j^2}{(p\littlec_j - \ew_{j,\bxi_j})^2} + \sum_{j=1}^{d}\frac{p^4\littlec_j^4}{\ew_{j,\bxi_j}(p\littlec_j - \ew_{j,\bxi_j})^2}
	\, .
	\end{equation}
	We first notice that all the terms involved in Eq.~\eqref{eq:aux:upper-bound-frobenius}are positive. 
	Moreover, we see that $\ew_{j,b}\leq 1$ and $p\littlec_j\leq p$ for any $j,b$. 
	Under Assumption~\ref{ass:bounded-weights}, $\abs{\tau_j(\xi_j) - \tau_j(x_{ij})} \leq 1$ almost surely. 
	We deduce that 
	\[
	\ew_{j,b} = \condexpec{\exp{\frac{-1}{2\nu^2}(\tau_j(\xi_j) - \tau_j(x_{ij}))^2}}{b_{ij}=b}\geq \exps{\frac{-1}{2\nu^2}}
	\, .
	\]
	As a consequence, $p\littlec_j - \ew_{j,\bxi_j} \geq (p-1)\exps{\frac{-1}{2\nu^2}}$. 
	Plugging these bounds in Eq.~\eqref{eq:aux:upper-bound-frobenius} and using $(x+y)^2\leq 2(x^2+y^2)$, we obtain
	\[
	\frobnorm{\bigc\Sigma^{-1}}^2 \leq 2\left(1 + \frac{d^2\exps{\frac{1}{\nu^2}}}{(p-1)^2} \right) + \frac{2dp^2\exps{\frac{1}{\nu^2}}}{(p-1)^2} + \frac{dp^4\exps{\frac{4}{\nu^2}}}{(p-1)^4}
	\, .
	\]
	Finally, we conclude using $p\geq 2$ and $d\geq 1$. 
\end{proof}

Of course, a more precise knowledge of the weights can give a better bound for $\opnorm{\Sigma^{-1}}$. 
For instance, it is possible to show that, for default weights,
\[
\opnorm{\Sigma^{-1}} \lesssim \littlec^{-d}\exps{\frac{1}{2\nu^2}}(d+p)
\, ,
\, 
\]
by studying closely the spectrum of $\Sigma$. 
Since the difference with Proposition~\ref{prop:upper-bound-operator-norm} is not that impressive, we keep the bound given for general weights. 
In particular, the dependency in $\nu$ is the same, indicating that the behavior for small $\nu$ is poor independently of the proof technique. 


\section{The study of $\Gamma^f$}
\label{section:gamma}

In this section, we turn to the study of $\Gamma^f$, the second step of our analysis. 
In the previous section, we dealt with the computation of $\Sigma$ and there was no dependency in~$f$. 
This is not the case anymore. 
We will assume from now on that $f$ is bounded on $\supp$, the hyperrectangle containing the training data, as it is done in Theorem~\ref{th:betahat-concentration-general-f-general-weights}. 
In particular, this assumption guarantees that all the expectations involving $f$ are well-defined (since $\pi$ and $z$ are also bounded almost surely). 


\subsection{Computation of $\Gamma^f$}

We begin by computing $\Gamma^f$. 
A straightforward computation shows that
\begin{equation*}
\Gammahat_j = 
\begin{cases}
\frac{1}{n}\sum_{i=1}^n \pi_i f(x_i) &\text{ if }j=0 \, ,\\
\frac{1}{n}\sum_{i=1}^n \pi_i z_{ij}f(x_i) &\text{ otherwise.}
\end{cases}
\end{equation*}
Since the $x_i$ are i.i.d., a straightforward consequence of the previous display is
\begin{equation}
\label{eq:computation-gamma-general}
\Gamma^f_j = 
\begin{cases}
\expec{ \pi f(x)} &\text{ if }j=0 \, ,\\
\expec{\pi z_{j}f(x)} &\text{ otherwise.}
\end{cases}
\end{equation}
Note that under Assumption~\ref{assump:bounded} $\Gamma^f$ is bounded component-wise.

\begin{myexample}[Constant model, general weights]
In order to get familiar with Eq.~\eqref{eq:computation-gamma-general}, let us focus on a constant model. 
In this case, we just need to compute $\expec{\pi}$ and $\expec{\pi z_{j}}$. 
According to Lemma~\ref{lemma:expectation-computations}, we have
\[
\expec{\pi} = \bigc \quad \text{and}\quad \expec{\pi z_{j}} = \bigc \frac{\ew_{j,\bxi_j}}{p\littlec_j}
\, .
\]
We have just showed that, if $f=f_0$ a constant, then
\begin{equation}
\label{eq:constant-gamma}
\bigc^{-1}\Gamma^f_0 = f_0 \quad\text{and}\quad \bigc^{-1}\Gamma^f_j = \frac{\ew_{j,\bxi_j}}{p\littlec_j} f_0 \quad\forall j \geq 1
\, .
\end{equation}
\end{myexample}

\subsection{Additive $f$}

We can specialize the computation of $\Gamma^f$ when $f$ is additive, always in the case of general weights. 

\begin{myproposition}[Computation of $\Gamma^f$, additive $f$]
	\label{prop:gamma-computation-additive-f-general-weights}
	Assume that $f$ satisfies Assumption \ref{ass:additive} and that $f$ is bounded on $\supp$. 
Then
for any $1\leq j\leq d$,
\[
\Gamma^f_j = \sum_{k=1}^{d}\frac{\ew_{j,\bxi_j}}{p\littlec_j}\cdot \frac{1}{p\littlec_k}\sum_{b=1}^{p}\ew_{k,b}^{f_k} + \frac{1}{p\littlec_j}\left[\ew_{j,\bxi_j}^{f_j} - \frac{\ew_{j,\bxi_j}}{p\littlec_j}\sum_{b=1}^{p}\ew_{j,b}^{f_j}\right]
\, ,
\]
and
	\[
	\Gamma^f_0 = \sum_{k=1}^{d} \frac{1}{p\littlec_k}\sum_{b=1}^{p}\ew_{k,b}^{f_k}
	\, .
	\]
\end{myproposition}

\begin{proof}
By linearity and since $f$ is additive, we can restrict ourselves to the case $f(x)=\psi (x_k)$. 
	We first look into $j=0$ in Eq.~\eqref{eq:computation-gamma-general}. 
	Then 
	\[
	\expec{\pi_i f(x_i)} = \expec{\pi_i \psi(x_{ik})} =  \frac{\bigc }{p\littlec_k} \sum_{b=1}^{p} \ew_{k,b}^{\psi} 
	\, ,
	\]
	according to Lemma~\ref{lemma:key-computation}. 
	Setting $\psi=f_k$ in the previous display and summing for $k\in\{1,\ldots,d\}$, we obtain the first part of the result. 
	
	Now we turn to the case $j\geq 1$. 
	Again, by linearity and since $f$ is additive, we can restrict ourselves to the case $f(x)=\psi (x_k)$. 
	There are two possibilities in this case: either $j=k$, and then 
	\[
	\expec{\pi_i z_{ij} f(x_i)} = \expec{\pi_i z_{ij} \psi(x_{ij})} = \bigc \frac{\ew_{j,\bxi_j}^\psi}{p\littlec_j}
	\, ,
	\]
	according to Lemma~\ref{lemma:key-computation};
	or $j\neq k$, and then 
	\[
	\expec{\pi_i z_{ij} f(x_i)} = \expec{\pi_i z_{ij} \psi(x_{ik})} = \bigc \frac{\ew_{j,\bxi_j}}{p\littlec_j}\frac{1}{p\littlec_k}\sum_{b=1}^{p}\ew_{k,b}^\psi
	\, ,
	\]
	according to Lemma~\ref{lemma:key-computation}. 
	Setting $\psi=f_k$ in the last displays and summing over $k\in\{1,\ldots,d\}$, we obtain 
	\[
	\bigc^{-1}\expec{\pi_i z_{ij} f(x_i)} = \sum_{\substack{k=1\\ k\neq j}}^{d} \frac{\ew_{j,\bxi_j}}{p\littlec_j}\frac{1}{p\littlec_k}\sum_{b=1}^{p}\ew_{k,b}^\psi + \frac{\ew_{j,\bxi_j}^{f_j}}{p\littlec_j}
	\, .
	\]
	Rearranging the terms in the sum yields the final result. 
\end{proof}

Let us specialize Proposition~\ref{prop:gamma-computation-additive-f-general-weights} for default weights.

\begin{myexample}[Computation of $\Gamma^f$, additive $f$, default weights]
	\label{ex:computation-gamma-additive-default}
	We can use Proposition~\ref{prop:gamma-computation-additive-f-general-weights} to compute $\Gamma^f$ for an additive $f$ when the weights are given by the default implementation. 
	Indeed, recall that $\condexpec{x_{ij}}{b_{ij}=b}=\mutilde_{j,b}$. 
	Since the weights are constant on each box, we can compute easily $\ew_{j,b}^\times$ as a function of $\mutilde_{j,b}$: 
	\[
	\ew_{j,b}^\times = 
	\begin{cases}
	\mutilde_{j,\bxi_j}\text{ if }b=\bxi_j \, ,\\
	\mutilde_{j,b}\exps{\frac{-1}{2\nu^2}} \text{ otherwise.}
	\end{cases}
	\]
	Also recall that the $\littlec_j$ are constant equal to $\littlec=\frac{1}{p}+(1-\frac{1}{p})\exps{\frac{-1}{2\nu^2}}$. 
We deduce that
\[
(\littlec^{-d}\Gamma^f)_0 = \sum_{k=1}^{d} \frac{1}{p\littlec}\sum_{b=1}^{p}\ew_{k,b}^{f_k}
\, ,
\]
and, for any $1\leq j\leq d$,
\[
(\littlec^{-d}\Gamma^f)_j = \sum_{k=1}^{d}\frac{\ew_{j,\bxi_j}}{p\littlec}\cdot \frac{1}{p\littlec}\sum_{b=1}^{p}\ew_{k,b}^{f_k} + \frac{1}{p\littlec}\left(\ew_{j,\bxi_j}^{f_j} - \frac{1}{p\littlec}\sum_{b=1}^{p}\ew_{j,b}^{f_j}\right)
\, .
\]
\end{myexample}

We can specialize Proposition~\ref{prop:gamma-computation-additive-f-general-weights} even further if $f$ is linear.  

\begin{mycorollary}[Computation of $\Gamma^f$, linear case, general weights]
	\label{cor:gamma-computation-linear-general-weights}
	Assume that $f(x)=f_0+f_1x_1+\cdots +f_dx_d$ for any $x\in\Reals^d$. 
Then for any $1\leq j\leq d$, 
\[
\Gamma_j^f = \frac{\ew_{j,\bxi_j}}{p\littlec_j}f(\gamma) + \frac{f_j}{p\littlec_j}\left\{\ew_{j,\bxi_j}^\times - \gamma_j \cdot \ew_{j,\bxi_j} \right\}
\, ,
\]
where we defined, for any $1\leq j\leq d$,
\[
\gamma_j \defeq \frac{1}{p\littlec_j}\sum_{b=1}^{p}\ew_{j,b}^\times
\, .
\]
Moreover,
	\[
	\Gamma^f_0 = f(\gamma)
	\, .
	\]
\end{mycorollary}

\begin{proof}
	Again, we use the fact that $\Gamma^f$ is linear as a function of $f$. 
	We have already looked into the constant case, and one can check that Eq.~\eqref{eq:constant-gamma} coincides with Corollary~\ref{cor:gamma-computation-linear-general-weights} when $f_j=0$ for all $j\geq 1$. 
	We then apply Proposition~\ref{prop:gamma-computation-additive-f-general-weights} with $f_j(x)=f_j \cdot x$ for any $j\geq 1$. 
	We notice that, in this case, for any $j\in \{1,\ldots,d\}$ and $b\in\{1,\ldots,p\}$, 
	\[
	\ew_{j,b}^{f_j} = f_j \cdot \ew_{j,b}
	\, .
	\]
	Substituting the last display yields the result. 
\end{proof}

Let us specialize Corollary~\ref{cor:gamma-computation-linear-general-weights} for default weights and smooth weights. 

\begin{myexample}[Computation of $\Gamma^f$, linear $f$, default weights]
	We can further specialize Example~\ref{ex:computation-gamma-additive-default}. 
	The only remaining computation is
	\[
	\gamma_j = \frac{1}{p\littlec_j}\sum_{b=1}^{p}\ew_{j,b}^\times = \frac{\mutilde_{j,\bxi_j} + \sum_{b\neq \bxi_j}  \exps{\frac{-1}{2\nu^2}}\mutilde_{j,b}}{1 + (p-1)\exps{\frac{-1}{2\nu^2}} } \eqdef \muttilde_j
	\, .
	\]
	Note that $\muttilde_j$ is a barycenter of the $\mutilde_j$ with weight $1$ for the box $\bxi_j$ and $\exps{\frac{-1}{2\nu^2}}$ for the others. 
	Finally, recall that $\bigc=\littlec^d$ and $\ew_{j,\bxi_j}=1$. 
	We have obtained that 
	\[
	\littlec^{-d}\Gamma^f_j = 
	\begin{cases}
	f(\muttilde) \text{ if }j=0 \, ,\\
	\frac{1}{p\littlec}f(\muttilde) + \frac{f_j}{p\littlec}(\mutilde_{j,\bxi_j}-\muttilde_j)\text{ otherwise.}
	\end{cases}
	\]
\end{myexample}

\begin{myexample}[Computation of $\Gamma^f$, linear $f$, smooth weights]
	\label{ex:computation-gamma-linear-old}
	We first compute
	\begin{align*}
	\ew_{j,\bxi_j}^\times &= \condexpec{x_{ij}\exp{\frac{-(x_{ij}-\xi_j)^2}{2\nu^2}}}{b_{ij}=\bxi_j} \\
	&= \frac{p}{\sigma\sqrt{2\pi}} \int_{q_{j,\bxi_j-1}}^{q_{j,\bxi_j}} x\cdot \exp{\frac{-(x-\xi_j)^2}{2\nu^2} + \frac{-(x-\mu_j)^2}{2\sigma^2}} \Diff x \tag{Eq. \eqref{eq:tn-density}} \\
	&= (\alpha_j m_j - \theta_j) \cdot \frac{p\nu}{\sqrt{\nu^2+\sigma^2}}\exps{\frac{-(\xi_j-\mu_j)^2}{2(\nu^2+\sigma^2)}} \tag{Lemma 11.2 in \citet{Garreau_von_Luxburg_2020}}
	\, ,
	\end{align*}
	where we set 
	\[
	\theta_j \defeq \left[\frac{1}{s\sqrt{2\pi}}\exp{\frac{-(x-m_j)^2}{2s^2}} \right]_{q_{j,\bxi_j-1}}^{q_{j,\bxi_j}}
	\, .
	\]
	We deduce that 
	\[
	\frac{\ew_{j,\bxi_j}^\times}{p\littlec_j} = \alpha_j m_j - \theta_j
	\, .
	\]
	Moreover, by the law of total expectation and Lemma~11.2 in~\citet{Garreau_von_Luxburg_2020},
	\[
	\frac{1}{p}\sum_{b=1}^{p} \ew_{j,b}^\times = m_j\cdot \frac{\nu}{\sqrt{\nu^2+\sigma^2}}\exp{\frac{-(\xi_j-\mu_j)^2}{2(\nu^2+\sigma^2)}}
	\, .
	\]
	Finally, since $\ew_{j,\bxi_j}/(p\littlec_j) = \alpha_j$ and $\gamma_j=m_j$, we find that $\Gamma^f_0=f(m)$ and 
	\[
	\Gamma^f_j = \alpha_j f(m) + f_j(\alpha_j m_j - \theta_j - \alpha_j m_j) = \alpha_j f(m) - f_j \theta_j
	\, .
	\]
We recover Lemma~9.1 in \citet{Garreau_von_Luxburg_2020}. 
\end{myexample}


\section{The study of $\beta^f$}
\label{section:beta}

After focusing on $\Sigma$ and $\Gamma^f$, we can now turn our attention to $\beta^f=\Sigma^{-1}\Gamma^f$. 
This section consists mostly in straightforward consequences of the computations achieved in Appendix~\ref{section:study-of-sigma} and~\ref{section:gamma}.

\subsection{Computation of $\beta^f$}

We begin by computing $\beta^f$ in the general case in closed-form.

\begin{myproposition}[Computation of $\beta^f$, general $f$, general weights]
	\label{prop:beta-computation-general-f-general-weights}
	Assume that $f$ is bounded on $\supp$. 
	Then 
for any $1\leq j\leq d$, 
\[
\beta^f_j = \bigc^{-1} \left\{\frac{-p\littlec_j}{p\littlec_j - \ew_{j,\bxi_j}}\expec{\pi f(x)} + \frac{p^2\littlec^2_j}{\ew_{j,\bxi_j}(p\littlec_j - \ew_{j,\bxi_j})}\expec{\pi z_{j}f(x)}\right\}
\, .
\]
Moreover,
	\[
	\beta^f_0 = \bigc^{-1}\left(1 + \sum_{j=1}^{d}\frac{\ew_{j,\bxi_j}}{p\littlec_j - \ew_{j,\bxi_j}}\right)\expec{\pi f(x)} -\bigc^{-1}\sum_{j=1}^{d}\frac{p\littlec_j}{p\littlec_j - \ew_{j,\bxi_j}}\expec{\pi z_{j}f(x)}
	\, .
	\]
\end{myproposition}

\begin{proof}
Direct computation from Proposition~\ref{prop:computation-inverse-sigma} and Eq.~\eqref{eq:computation-gamma-general}. 
\end{proof}

We can easily specialize Proposition~\ref{prop:beta-computation-general-f-general-weights} for the default weights. 
Recall that, in this case, $\ew_{j,\bxi_j}=1$ and $\littlec_j=\littlec$ is a constant. 
Furthermore, $\bigc = \littlec^d$. 

\begin{mycorollary}[Computation of $\beta^f$, general $f$, default weights]
	\label{cor:beta-computation-general-f-default-weights}
	Assume that $f$ is bounded on $\supp$. 
	Then it holds that
for any $1\leq j\leq d$,
\[
\beta^f_j = \littlec^{-d} \left\{\frac{-p\littlec}{p\littlec - 1}\expec{\pi f(x)} + \frac{p^2\littlec^2}{p\littlec - 1}\expec{\pi z_{j}f(x)} \right\}
\, .
\]
Moreover
	\[
	\beta^f_0 = \littlec^{-d}\left(1 + \frac{d}{p\littlec - 1}\right)\expec{\pi f(x)} -\littlec^{-d}\frac{p\littlec}{p\littlec - 1}\sum_{j=1}^{d}\expec{\pi z_{j}f(x)}
	\, .
	\]
	and, 
\end{mycorollary}

This last result gives the expression of $\beta^f$ used in our main result. 


\subsection{Additive $f$}

We can specialize the results of the previous section when $f$ has a specific structure. 
We begin with the case of an additive $f$. 

\begin{myproposition}[Computation of $\beta^f$, additive $f$, general weights]
\label{prop:beta-computation-additive-f-general-weights}
Assume that $f$ satisfies Assumption~\ref{ass:additive}. 
Then 

for any $1\leq j\leq d$, 
\[
\beta^f_j = \frac{p\littlec_j}{\ew_{j,\bxi_j}(p\littlec_j - \ew_{j,\bxi_j})} \left(\ew_{j,\bxi_j}^{f_j} - \frac{\ew_{j,\bxi_j}}{p\littlec_j}\sum_{b=1}^{p}\ew_{j,b}^{f_j}\right)
\, .
\]
Moreover, 
\[
\beta^f_0 = \sum_{k=1}^{d} \frac{1}{p\littlec_k-\ew_{k,\bxi_k}}\sum_{b\neq \bxi_k}\ew_{k,b}^{f_k}
\, .
\]
\end{myproposition}

\begin{proof}
	First let us treat the case $j=0$. 
	We write
	\begin{align*}
	\beta^f_0 &= \left(1+\sum_{j=1}^{d}\frac{\ew_{j,\bxi_j}}{p\littlec_j - \ew_{j,\bxi_j}}\right)\left(\sum_{k=1}^{d} \frac{1}{p\littlec_k}\sum_{b=1}^{p}\ew_{k,b}^{f_k}\right) \\
	&- \sum_{j=1}^{d}\frac{p\littlec_j}{p\littlec_j - \ew_{j,\bxi_j}} \left(\sum_{k=1}^{p} \frac{\ew_{j,\bxi_j}}{p^2\littlec_j\littlec_k}\sum_{b=1}^{p}\ew_{k,b}^{f_k} + \frac{1}{p\littlec_j}\left[\ew_{j,\bxi_j}^{f_j} - \frac{\ew_{j,\bxi_j}}{p\littlec_j}\sum_{b=1}^p \ew_{j,b}^{f_j}\right]\right) \\
	&= \sum_{k=1}^{d} \frac{1}{p\littlec_k} \sum_{b=1}^{p} \ew_{k,\bxi_k}^{f_k} - \sum_{j=1}^{d} \frac{1}{p\littlec_j - \ew_{j,\bxi_j}} \left[\ew_{j,\bxi_j}^{f_j} - \frac{\ew_{j,\bxi_j}}{p\littlec_j}\sum_{b=1}^{p}\ew_{j,b}^{f_j}\right]\, .
	\end{align*}
	We conclude after changing the indices in the sum and some algebra. 
	As for the other terms,
	\begin{align*}
	\beta^f_j &= \frac{-p\littlec_j}{p\littlec_j - \ew_{j,\bxi_j}}\sum_{k=1}^{d}\frac{1}{p\littlec_k}\sum_{b=1}^{p} \ew_{k,b}^{f_k} \\
	&+ \frac{p^2\littlec_j^2}{\ew_{j,\bxi_j}(p\littlec_j - \ew_{j,\bxi_j})}\left[\sum_{k=1}^{d}\frac{\ew_{j,\bxi_j}}{p^2\littlec_j\littlec_k}\sum_{b=1}^{p}\ew_{k,b}^{f_k} + \frac{1}{p\littlec_j}\left(\ew_{j,\bxi_j}^{f_j} - \frac{\ew_{j,\bxi_j}}{p\littlec_j}\sum_{b=1}^{p}\ew_{j,b}^{f_j}\right)\right]\, , \\
	\end{align*}
	and we obtain the promised result after some simplifications.
\end{proof}

We can specialize Proposition~\ref{prop:beta-computation-additive-f-general-weights} even further if $f$ is linear. 
For any $1\leq j\leq d$, define
\[
\gamma_j \defeq \frac{1}{p\littlec_j}\sum_{b=1}^{p}\ew_{j,b}^\times
\, .
\]
Note that $\gamma_j=\muttilde_j$ when default weights are used. 

\begin{mycorollary}[Computation of $\beta^f$, linear $f$, general weights]
	\label{cor:beta-computation-linear}
Assume that for any $x\in\Reals^d$, $f(x)=f_0+f_1x_1+\cdots+f_dx_d$.
Then, for any $1\leq j\leq d$, 
\[
\beta_j^f = \frac{p\littlec_j}{\ew_{j,\bxi_j}(p\littlec_j - \ew_{j,\bxi_j})}(\ew_{j,\bxi_j}^\times - \gamma_j\cdot \ew_{j,\bxi_j})f_j
\, .
\]
Moreover,
\[
\beta^f_0 = f(\gamma) - \sum_{j=1}^{d} \frac{1}{p\littlec_j - \ew_{j,\bxi_j}} (\ew_{j,\bxi_j}^\times - \gamma_j\cdot \ew_{j,\bxi_j}) f_j
\, .
\]
\end{mycorollary}


Let us see how we can recover the analysis of \citet{Garreau_von_Luxburg_2020} in the linear case. 

\begin{myexample}[Computation of $\beta^f$, linear $f$, smooth weights]
	We have seen before that, in this case, $\gamma = m$. 
	Moreover, 
	\[
	p\littlec_j - \ew_{j,\bxi_j} = (1-\alpha_j)\frac{p\nu}{\sqrt{\nu^2+\sigma^2}}\exp{\frac{-(\xi_j-\mu_j)^2}{2(\nu^2+\sigma^2)}}
	\, .
	\]
	Then we write
	\begin{align*}
	\ew_{j,\bxi_j}^\times - \gamma_j \ew_{j,\bxi_j} &= (\alpha_j m_j - \theta_j)\frac{p\nu}{\sqrt{\nu^2+\sigma^2}}\exp{\frac{-(\xi_j-\mu_j)^2}{2(\nu^2+\sigma^2)}} \\
	&\phantom{blablabla}- m_j \alpha_j \frac{p\nu}{\sqrt{\nu^2+\sigma^2}}\exp{\frac{-(\xi_j-\mu_j)^2}{2(\nu^2+\sigma^2)}} \\
	&= -\theta_j \frac{p\nu}{\sqrt{\nu^2+\sigma^2}}\exp{\frac{-(\xi_j-\mu_j)^2}{2(\nu^2+\sigma^2)}}
	\, . 
	\end{align*}
	We deduce that 
	\begin{equation}
	\label{eq:aux:linear-old}
	\frac{\ew_{j,\bxi_j}^\times - \gamma_j \ew_{j,\bxi_j}}{p\littlec_j - \ew_{j,\bxi_j}} = \frac{-\theta_j}{1-\alpha_j}
	\, ,
	\end{equation}
	and therefore
	\[
	\beta^f_0= f(m) + \sum_{j=1}^{d} \frac{\theta_jf_j}{1-\alpha_j} 
	\, .
	\]
	Now, for any given $j>0$, $p\littlec_j / \ew_{j,\bxi_j} = 1/\alpha_j$. 
	Combining with Eq.~\eqref{eq:aux:linear-old} we obtain
	\[
	\beta^f_j = \frac{-\theta_jf_j}{\alpha_j(1-\alpha_j)}
	\, .
	\]
	This is the expression of $\beta^f$ appearing in Theorem~3.1 of \citet{Garreau_von_Luxburg_2020}. 
\end{myexample}


\subsection{Multiplicative $f$}

When $f$ is multiplicative (Assumption~\ref{ass:multiplicative}), we can also be more precise in the computation of $\beta^f$. 

\begin{mylemma}[Computation of $\beta^f$, multiplicative $f$, general weights]
	\label{lemma:beta-computation-multiplicative-f-general-weights}
	Assume that~$f$ satisfies Assumption~\ref{ass:multiplicative} and is bounded on $\supp$. 
	Then, for any $1\leq j\leq d$, 
	\[
	\beta^f_j = \frac{\prod_{k=1}^{d}\littlec_k^{f_k}}{\bigc} \cdot \frac{p\littlec_j}{p\littlec_j - \ew_{j,\bxi_j}} \left( \frac{\ew_{j,\bxi_j}^{f_j}}{\ew_{j,\bxi_j}}\cdot \frac{\littlec_j}{\littlec_j^{f_j}} - 1\right)
	\, .
	\]
Moreover,
	\[
	\beta^f_0 = \frac{\prod_{k=1}^{d}\littlec_k^{f_k}}{\bigc} \left\{1 + \sum_{j=1}^{d} \frac{\ew_{j,\bxi_j}}{p\littlec_j-\ew_{j,\bxi_j}} \left(1 - \frac{\ew_{j,\bxi_j}^{f_j}}{\ew_{j,\bxi_j}}\cdot \frac{\littlec_j}{\littlec_j^{f_j}}\right)\right\}
	\, .
	\]
\end{mylemma}

\begin{proof}
In view of Proposition~\ref{prop:beta-computation-general-f-general-weights}, we just have to compute $\expec{\pi f(x)}$ and $\expec{\pi z_{j}f(x)}$ for any given $1\leq j\leq d$. 
	We begin with the computation of $\expec{\pi f(x)}$:
	\begin{align*}
	\expec{\pi f(x)} &= \expec{\prod_{k=1}^k \exp{\frac{-(\tau_k(x_{k}) - \tau_k(\xi_k))^2}{2\nu^2}} f_k(x_{k})} \tag{Assumption \ref{ass:multiplicative} $+$ Eq. \eqref{eq:definition-weights}} \\
	&= \prod_{k=1}^d \expec{\exp{\frac{-(\tau_k(x_{k}) - \tau_k(\xi_k))^2}{2\nu^2}}f_k(x_{k})} \tag{independence}\\
	\expec{\pi f(x)} &= \prod_{k=1}^d \littlec_k^{f_k} \, .\tag{Lemma \ref{lemma:base-computation}}
	\end{align*}
	The second computation is very similar in spirit:
	\begin{align*}
	\expec{\pi z_{j} f(x)} &= \expec{\prod_{\substack{k=1\\ k\neq j}}^k \exps{\frac{-(\tau_k(x_{k}) - \tau_k(\xi_k))^2}{2\nu^2}} f_k(x_{k})\cdot \exps{\frac{-(\tau_j(x_{j}) - \tau_j(\xi_j))^2}{2\nu^2}} z_{j}f_j(x_{j})} \tag{Assumption \ref{ass:multiplicative} $+$ Eq. \eqref{eq:definition-weights}} \\
	&= \prod_{\substack{k=1\\ \neq j}}^d \expec{\exp{\frac{-(\tau_k(x_{k}) - \tau_k(\xi_k))^2}{2\nu^2}}f_k(x_{k})} \cdot \expec{\exps{\frac{-(\tau_j(x_{j}) - \tau_j(\xi_j))^2}{2\nu^2}} z_{j}f_j(x_{j})}\tag{independence}\\
	\expec{\pi z_{j} f(x)} &= \prod_{\substack{k=1 \\ k\neq j}}^d \littlec_k^{f_k} \cdot \frac{\ew_{j,\bxi_j}^{f_j}}{p}\tag{Lemma \ref{lemma:base-computation}}
	\, .
	\end{align*}
Simple algebra concludes the proof.
\end{proof}


\section{Proof of Proposition~\ref{prop:explanation-stability-default-weights}}
\label{section:proof-stability-corollary}

In this Appendix, we prove the regularity result for Tabular LIME,  Proposition~\ref{prop:explanation-stability-default-weights} of the main paper. 

\begin{proof}
First let us set $h\defeq f-g$ and $\epsilon\defeq \infnorm{h}$. 
We notice that
\[
\smallnorm{\beta^f - \beta^g} = \smallnorm{\Sigma^{-1}\Gamma^f - \Sigma^{-1}\Gamma^g} = \smallnorm{\Sigma^{-1}\Gamma^h}
\, .
\]
Let us focus first on the first coordinate:
\begin{align*}
(\Sigma^{-1}\Gamma^h)_0 &= \bigc^{-1}\left(1+\sum_{j=1}^{d}\frac{1}{p\littlec-1} \right)\expec{\pi h(x)} + \bigc^{-1}\sum_{j=1}^{d} \frac{-p\littlec}{p\littlec - 1}\expec{\pi z_{j}h(x)} \\
&= \bigc^{-1}\expec{\pi h(x)} + \bigc^{-1}\sum_{j=1}^{d}\frac{1}{p\littlec - 1}\expec{\pi (1-p\littlec z_{j})h(x)}
\, .
\end{align*}
Recall Lemma~\ref{lemma:expectation-computations}:
\[
\abs{\bigc^{-1}\expec{\pi h(x)}} \leq \epsilon
\, .
\]
As for the second part, we write
\begin{align*}
\expec{\pi \abs{1-p\littlec z_{j}}h(x)} &\leq (p\littlec \expec{\pi z_{j}} + \expec{\pi})\epsilon \\
&= (p\littlec \cdot \frac{\bigc}{p\littlec} + \bigc)\epsilon \tag{Lemma \ref{lemma:expectation-computations}} \\
\expec{\pi\abs{1-p\littlec z_{j}}h(x)} &\leq 2\bigc\epsilon
\end{align*}
We deduce 
\begin{equation}
\label{eq:aux:bound-beta-0}
\abs{(\Sigma^{-1}\Gamma^h)_0} \leq \left(1+ \frac{2d}{p-1}\exps{\frac{1}{2\nu^2}}\right)\epsilon
\, .
\end{equation}
Now let us set $j\geq 1$. 
\begin{align*}
(\Sigma^{-1}\Gamma^h)_j &= \bigc^{-1}\left[\frac{-p\littlec}{p\littlec - 1}\expec{\pi h(x)} + \frac{p^2\littlec^2}{p\littlec - 1}\expec{\pi z_{j}h(x)}\right] \\
&= \frac{\bigc^{-1}p\littlec}{p\littlec - 1} \expec{\pi (p\littlec z_{j}- 1)h(x)}
\, .
\end{align*}
As before, we obtain
\begin{equation}
\label{eq:aux:bound-beta-j}
\abs{(\Sigma^{-1}\Gamma^h)_j} \leq \frac{2p\littlec\epsilon }{p\littlec - 1} \leq \frac{2p\exps{\frac{1}{2\nu^2}}\epsilon}{p-1}
\, .
\end{equation}
We then collect Eq.~\eqref{eq:aux:bound-beta-0} and \eqref{eq:aux:bound-beta-j} to obtain the promised bound. 
\end{proof}


\section{Proof of Proposition~\ref{prop:beta-computation-indicator}}
\label{section:proof-computation-beta-indicator}

In this section, we show how to compute $\beta^f$ for indicator function with rectangular support (Proposition~\ref{prop:beta-computation-indicator} of the main paper). 

\begin{proof}
Our main task is to compute the $\ew_{j,b}^{a_j}$ coefficients in this specific case. 
By definition of the $\ew_{j,b}$, we have
\begin{equation}
\label{eq:ew-tree-computation}
\ew_{j,b}^{a_j} = \exps{\frac{-\indic{b=\bxi_j}}{2\nu^2}}\condexpec{\indic{x_j\in [s_j,t_j]}}{b_j\neq b}
\eqdef \exps{\frac{-\indic{b\neq \bxi_j}}{2\nu^2}} \et_{j,b}
\, .
\end{equation}
Since $x_j$ has support on $[q_{j,b-1},q_{j,b}]$ conditionally to the event $\{b_j=b\}$, it is straightforward to compute $\et_{j,b}$ with respect to the relative position of $[s_j,t_j]$ and $[q_{j,b-1},q_{j,b}]$. 
In particular, $\et_{j,b}=0$ if the intersection is empty, and we find
\begin{equation}
\label{eq:expression-et}
\et_{j,b} = \frac{\Phi\left(\frac{t_j\wedge q_{j,b} -\mu_{j,b}}{\sigma_{j,b}}\right) - \Phi\left(\frac{s_j\vee q_{j,b-1}-\mu_{j,b}}{\sigma_{j,b}}\right)}{\Phi\left(\frac{q_{j,b}-\mu_{j,b}}{\sigma_{j,b}}\right) - \Phi\left(\frac{q_{j,b-1}-\mu_{j,b}}{\sigma_{j,b}}\right)}
\, 
\end{equation}
otherwise. 
Now recall that we assumed that there exist $1\leq \bzero_j\leq p$ such that $[s_j,t_j]\subseteq [q_{j,\bzero_j-1},q_{j,\bzero_j}]$. 
Therefore Eq.~\eqref{eq:expression-et} simplifies and we find 
\[
\et_{j,b} = 
\begin{cases}
\frac{\Phi\left(\frac{t_j -\mu_{j,b}}{\sigma_{j,b}}\right) - \Phi\left(\frac{s_j-\mu_{j,b}}{\sigma_{j,b}}\right)}{\Phi\left(\frac{q_{j,b}-\mu_{j,b}}{\sigma_{j,b}}\right) - \Phi\left(\frac{q_{j,b-1}-\mu_{j,b}}{\sigma_{j,b}}\right)} \eqdef \lambda_j &\text{ if }b = \bzero_j , \\
0 &\text{ otherwise.}
\end{cases}
\]
Straightforward computations yield
\[
\frac{\littlec_j^{a_j}}{\littlec} = \frac{\et_{j,\bxi_j}+\exps{\frac{-1}{2\nu^2}} \sum_{b\neq\bxi_j}\et_{j,b}}{1+(p-1)\exps{\frac{-1}{2\nu^2}} }
\]
and 
\[
\frac{p\littlec}{p\littlec - 1}\left( \ew_{j,\bxi_j}^{a_j}\cdot \frac{\littlec}{\littlec_j^{a_j}} - 1\right) = \frac{1}{p-1}\sum_{b\neq \bxi_j}(\et_{j,\bxi_j}-\ew_{j,b}) \cdot \frac{\littlec}{\littlec_j^{a_j}}
\, .
\]
For any $1\leq k\leq d$, there are two possible cases: either $\xi_k\in\Abar_k$ ($\xi$ is aligned with $A$ along dimension $k$), or $\xi_k\notin\Abar_k$ ($\xi$ is not aligned). 
In the first case, $\et_{k,\bxi_k}=\lambda_k$ and all the other $\et_{k,b}$ are equal to zero. 
Therefore $\frac{\littlec_k^{a_k}}{\littlec} = \frac{\lambda_k}{p\littlec}$, and $\frac{1}{p-1}\sum_{b\neq \bxi_j}(\et_{k,\bxi_k}-\ew_{k,b}) = \lambda_k$. 
In the second case, $\et_{k,\bxi_k}=0$ and there is only one $b\neq \bxi_k$ such that $\et_{k,b}=\lambda_k$. 
We deduce that $\frac{\littlec_k^{a_k}}{\littlec} = \frac{\lambda_k\exps{\frac{-1}{2\nu^2}}}{p\littlec}$ and $\frac{1}{p-1}\sum_{b\neq \bxi_k}(\et_{k,\bxi_k}-\ew_{k,b}) = \frac{-\lambda_k}{p-1}$. 
From this discussion, we obtain
\[
\frac{\prod_{k=1}^d \littlec_k^{a_k}}{\littlec^d} = \frac{\prod_{k=1}^d \lambda_k}{p^d\littlec^d}\exp{\frac{-\card{k \text{ s.t. } \xi_k \notin \Abar_k}}{2\nu^2}} = \frac{\relimp{A }}{p^d\littlec^d}\exp{\frac{-\dist{\xi}{A}}{2\nu^2}}
\, .
\]
Finally, we can use Proposition~\ref{prop:beta-computation-multiplicative-default} to conclude. 
\end{proof}


\section{Concentration of $\Sigmahat_n$}
\label{section:concentration-sigmahat}

In this section, we show that $\Sigmahat_n$ is concentrated around $\Sigma$ in operator norm. 
The idea is to use standard results on the concentration of sum of independent random matrices \citep{Vershynin_2018}. 
Indeed, $\Sigmahat$ can be written as $\frac{1}{n}\sum_{i=1}^{n}\pi_i Z_iZ_i^\top$. 
Since each of these matrices are bounded and identically distributed, we turn to a Hoeffding-type inequality. 
We borrow the following result from \citet{Tropp_2012}. 

\begin{mytheorem}[Matrix Hoeffding \citep{Tropp_2012}]
	\label{th:matrix-hoeffding}
	Consider a finite sequence $M_i$ of independent, random, symmetric matrices with common dimension $D$, and let $A_i$ be a sequence of fixed symmetric matrices. 
	Assume that each random matrix satisfies 
	\[
	\expec{M_i} = 0 \quad \text{and}\quad M_i^2 \preccurlyeq A_i^2 \quad \text{almost surely}
	\, .
	\] 
	Then, for all $t\geq 0$, 
	\[
	\proba{\lambdamax{\sum_{i=1}^{n} M_i} \geq t } \leq D \cdot \exp{\frac{-t^2}{8\sigma^2}}
	\, ,
	\]
	where $\sigma^2 \defeq \opnorm{\sum_{i=1}^{n}A_i^2}$. 
\end{mytheorem}

We slightly adapt this result for the situation at hand. 

\begin{mycorollary}[Matrix Hoeffding, bounded entries]
	\label{th:matrix-hoeffding-bounded}
	Consider a finite sequence $M_i$ of independent, centered, random, symmetric matrices with dimension $D$. 
	Assume that the entries of each matrix satisfy
	\[
	(M_i)_{j,k} \in [-1,1] \quad \text{almost surely}
	\, .
	\] 
	Then, for all $t\geq 0$, 
	\[
	\proba{\opnorm{\frac{1}{n}\sum_{i=1}^{n} M_i} \geq t } \leq 2D \cdot \exp{\frac{-nt^2}{8D^2}}
	\, .
	\]
\end{mycorollary}

\begin{proof}
	Let $u\in\Reals^D$ such that $\norm{u}=1$. 
	We write 
	\begin{align*}
	\abs{(Mi_u)_j} &= \sum_{k=1}^{D} (M_i)_{j,k}u_k \\ 
	&\leq \norm{(M_i)_j} \cdot \norm{u} \tag{Cauchy-Schwarz} \\
	&\leq \sqrt{D} \tag{bounded a.s. $+$ $\norm{u}=1$}
	\, .
	\end{align*}
	We deduce that $\norm{M_iu}\leq D$ almost surely. 
	Since we considered an arbitrary $u$, we have showed that $\opnorm{M_i}\leq D$ almost surely for any $i$. 
	Thus we can apply Theorem~\ref{th:matrix-hoeffding} with $A_i=D\Identity_D$ to obtain
	\[
	\proba{\lambdamax{\sum_{i=1}^{n} M_i} \geq t} \leq D \cdot \exp{\frac{-t^2}{8nD^2}}
	\, ,
	\]
	using the fact that $A_i$ commutes with $M_i$ and thus $M_i\preccurlyeq A_i$ implies $M_i^2\preccurlyeq A_i^2$. 
	The result is a direct consequence from the last display.
\end{proof}

We can now proceed to the main result of this section, the concentration of $\Sigmahat$ around its mean $\Sigma$. 

\begin{myproposition}[Concentration of $\Sigmahat$, general weights]
	\label{prop:sigmahat-concentration-general-weights}
	For any $t\geq 0$, 
	\[
	\proba{\smallopnorm{\Sigmahat - \Sigma} \geq t} \leq 4 d \cdot \exp{\frac{-nt^2}{32d^2}}
	\, .
	\]
\end{myproposition}

\begin{proof}
	Recall Eq.~\eqref{eq:sigmahat-computation}: the entries of $\pi_i Z_iZ_i^\top$ belong to $[0,1]$ almost surely since $\pi_i\in [0,1]$ for any weights satisfying Eq.~\eqref{eq:definition-weights} and $z_{ij}\in [0,1]$. 
	As a consequence, so do the entries of~$\Sigma$. 
	Let us set 
	\[
	M_i \defeq \pi_i Z_iZ_i^\top - \Sigma
	\, .
	\]
	Then $M_i$ satisfies the assumptions of Theorem~\ref{th:matrix-hoeffding-bounded} with $D=d+1$ and the result follows since $\frac{1}{n}\sum_{i=1}^{n}M_i = \Sigmahat - \Sigma$. 
\end{proof}


\section{Concentration of $\Gammahat$}
\label{section:concentration-gammahat}

The goal of this section is the concentration of $\Gammahat_n$. 
Under a boundedness assumption on~$f$, each coordinate of $\Gammahat_n$ is bounded almost surely, and it seems natural to use Hoeffding's inequality \citep{hoeffding1963inequality}.
Interestingly, we could not find a multivariate extension of Hoeffding's inequality. 
\citet{Vershynin_2018} gives results for the concentration of the norm of a vector with sub-Gaussian coordinates, but the coordinates are assumed to be independent and the constants are not explicit.  
We resort to a combination of Hoeffding's inequality in the univariate case and a union bound argument. 
We use the following version of Hoeffding's inequality:

\begin{mytheorem}[Hoeffding's inequality]
	\label{th:hoeffding-univariate}
	Let $M_i$ be a finite sequence of centered random variables such that 
	\[
	M_i \in [-M,M] \quad \text{almost surely}
	\, .
	\]
	Then, for any $t\geq 0$, 
	\[
	\proba{\abs{\frac{1}{n}\sum_{i=1}^{n} M_i} \geq t} \leq 2 \cdot \exp{\frac{-nt^2}{2M^2}}
	\, .
	\]
\end{mytheorem}

\begin{proof}
	This is Theorem~2.8 in \citet{Boucheron_2013} in our notation.
\end{proof}

Our concentration bound for $\Gammahat_n$ reads:

\begin{myproposition}[Concentration of $\Gammahat$, general weights]
	\label{prop:gammahat-concentration}
Assume that Assumption~\ref{assump:bounded} holds. 
	Then, for any $t\geq 0$, 
	\[
	\proba{\smallnorm{\Gammahat - \Gamma^f} \geq t} \leq 4d \exp{\frac{-nt^2}{32\boundedcst d^2}}
	\, .
	\]
\end{myproposition}

\begin{proof}
Recall that the components of $\Gammahat$ are given by $\tfrac{1}{n}\sum_{i=1}^{n}\pi_i f(x_i)$ for $j=0$, and, for $1\leq j\leq d$, by $\tfrac{1}{n}\sum_{i=1}^{n}\pi_i z_{i,j}f(x_i)$. 
	Since $\abs{f}$ is bounded by $\boundedcst$ and $\pi_i,z_{i,j}\in [0,1]$ for any weights satisfying Eq.~\eqref{eq:definition-weights}, 
	we deduce that, for any $1\leq i\leq n$ and $1\leq j\leq d$, 
	\[
	\begin{cases}
	\pi_i f(x_i) - \left(\Gamma^f\right)_0 \in [-2\boundedcst,2\boundedcst] \, ,\\
	\pi_i z_{i,j}f(x_i) - \left(\Gamma^f\right)_j \in [-2\boundedcst,2\boundedcst]\, .
	\end{cases}
	\]
	almost surely. 
	We can apply Theorem~\ref{th:hoeffding-univariate} coordinate by coordinate: for any $t\geq 0$, for any $1\leq j\leq n$, 
	\[
	\begin{cases}
	\proba{\abs{\frac{1}{n}\sum_{i=1}^{n} \pi_i f(x_i) - \left(\Gamma^f\right)_0} \geq t} &\leq 2 \cdot \exp{\frac{-nt^2}{8\boundedcst^2}} \, ,\\
	\proba{\abs{\frac{1}{n}\sum_{i=1}^{n} \pi_i z_{i,j}f(x_i) - \left(\Gamma^f\right)_j} \geq t} &\leq 2 \cdot \exp{\frac{-nt^2}{8\boundedcst^2}} \, .
	\end{cases}
	\]
	By a union bound argument, for any random vector $u\in \Reals^D$, 
	\[
	\proba{\norm{u} \geq t} \leq \proba{\max_i \abs{u_i} \geq t/D} \leq D \cdot \max_i \proba{\abs{u_i} \geq t/D}
	\, .
	\]
	We deduce the result from the previous display, taking $D=d+1$ and using, as usual, $d+1\leq 2d$ for any $d\geq 1$. 
\end{proof}


\section{Proof of the main result}
\label{section:concentration-betahat}

In this section we prove that $\betahat$ is concentrated around $\beta^f$.

\subsection{Binding lemma}

We begin by a somewhat technical result, showing that we can control the behavior of $\smallnorm{\betahat - \Sigma^{-1}\Gamma^f}$ by controlling $\smallopnorm{\Sigmahat - \Sigma}$, $\smallnorm{\Gammahat-\Gamma^f}$, and $\opnorm{\Sigma^{-1}}$. 

\begin{mylemma}[Control of $(\Identity_D+H)^{-1}$]
	\label{lemma:control-inverse-identity}
	Let $H\in\Reals^{D\times D}$ be a matrix such that $\opnorm{H}\leq -1+\frac{\sqrt{7}}{2}(\approx 0.32)$. 
	Then $\Identity_D+H$ is invertible and 
	\[
	\opnorm{(\Identity_D+H)^{-1}} \leq 2
	\, .
	\]
\end{mylemma}

\begin{proof}
	We first show that $\Identity_D+H$ is invertible. 
	Indeed, suppose that $\kernel{\Identity_D + H}\neq \{0\}$. 
	Then, in particular, there exists $x\in\Reals^D$ with unit norm such that $Hx = -x$. 
	Since $\opnorm{H}=\max_{\norm{x}=1}\norm{Hx}$, we would have $\opnorm{H}\geq 1$, which contradicts $\opnorm{H}\leq 0.32$. 
	
	Now according to Lemma~\ref{lemma:operator-norm:inversion}, 
	\begin{equation}
	\label{eq:link-opnorm-min-eigenvalue}
	\opnorm{(\Identity_D+H)^{-1}} = (\lambdamin{(\Identity_D + H)^\top (\Identity_D + H)})^{-1/2}
	\, .
	\end{equation}
	Let us set $M\defeq (\Identity_D + H)^\top (\Identity_D + H)$ and find a lower bound on $\lambdamin{M}$. 
	We first notice that $M$ is positive semi-definite. 
	Thus we now that $\lambdamin{M} = \min_{\norm{x}=1} x^\top Mx$. 
	Let us fix $x\in S^{D-1}$ for now. 
	Then 
	\[
	x^\top M x = x^\top (M-\Identity_D)x + x^\top x = x^\top (M-\Identity_D)x + 1
	\, ,
	\]
	since $\norm{x}=1$. 
	But on the other side,
	\begin{align*}
	-x^\top (M-\Identity_D)x &\leq \abs{\scalar{x}{(M-\Identity_D)x}} \\
	&\leq \norm{x}\cdot \norm{(M-\Identity_D)x} \tag{Cauchy-Schwarz} \\
	&\leq \opnorm{M-\Identity_D} \tag{definition of $\opnorm{\cdot}$ $+$ $\norm{x}=1$}
	\end{align*}
	Moreover, by the triangle inequality and sub-multiplicativity of the operator norm,
	\[
	\opnorm{M-\Identity_D} = \opnorm{H+H^\top + H^\top H} \leq 2\opnorm{H}+\opnorm{H}^2
	\, .
	\]
	We deduce that
	\begin{equation}
	\label{eq:lower-bound-min-eigenvalue}
	\lambdamin{M} \geq 1 - 2\opnorm{H}-\opnorm{H}^2
	\, ,
	\end{equation}
	which is a positive quantity since we assumed $\opnorm{H}\leq 0.32$. 

	Plugging the lower bound \eqref{eq:lower-bound-min-eigenvalue} into Eq.~\eqref{eq:link-opnorm-min-eigenvalue}, we obtain 
	\[
	\opnorm{(\Identity_D + H)^{-1}} \leq \frac{1}{\sqrt{1-2\opnorm{H} - \opnorm{H}^2}}
	\, .
	\]
	Set $\psi(x)\defeq (1-2x-x^2)^{-1/2}$. 
	One can easily check that $\psi(x) \leq 2$ for any $x\in [0,-1+\sqrt{7}/2]$, which concludes the proof. 
\end{proof}

\begin{myremark}
	The numerical constant $2$ in the statement of Lemma~\ref{lemma:control-inverse-identity} can be replaced by any arbitrary constant greater than $1$, at the cost of constraining further the range of~$\opnorm{H}$. 
\end{myremark}

\begin{myremark}
	The Hoffman-Wielandt inequality also yields a lower bound on $\lambdamin{M}$, giving essentially the same result but with the Frobenius norm instead of the operator norm. 
	Since we know how to control $\opnorm{\Sigmahat - \Sigma}$, we prefer this version of the result. 
\end{myremark}

Using Lemma~\ref{lemma:control-inverse-identity} we can prove something more interesting.

\begin{mylemma}[Control of $(\Identity_D+H)^{-1}-\Identity_d$]
	\label{lemma:difference-control}
	Let $H\in\Reals^{D\times D}$ be a matrix such that $\opnorm{H}\leq -1+\frac{\sqrt{7}}{2}(\approx 0.32)$. 
	Then $\Identity_D+H$ is invertible, and
	\[
	\opnorm{(\Identity_D + H)^{-1} - \Identity_D} \leq 2 \opnorm{H}
	\, .
	\]
\end{mylemma}

\begin{proof}
	According to Lemma~\ref{lemma:control-inverse-identity}, $\Identity_D+H$ is an invertible matrix. 
	Now we write
	\begin{equation}
	\label{eq:woodbury-reduction}
	(\Identity_D + H)^{-1} - \Identity_D = (\Identity_D + H)^{-1}(\Identity_D - (\Identity_D + H)) = -(\Identity_D + H)^{-1} H
	\, .
	\end{equation}
	Since the operator norm is sub-multiplicative, Eq.~\eqref{eq:woodbury-reduction} implies that
	\[
	\opnorm{(\Identity_D + H)^{-1} - \Identity_D} \leq \opnorm{(\Identity_D + H)^{-1}} \cdot \opnorm{H}
	\, . 
	\]
	and Lemma~\ref{lemma:control-inverse-identity} guarantees that $\opnorm{(\Identity_D + H)^{-1}}\leq 2$ under our assumptions. 
\end{proof}

\begin{myremark}
	It is a good surprise that the constants in Lemma~\ref{lemma:difference-control} do not depend on the dimension. 
	In fact, we believe that this result is nearly optimal. 
	Indeed, in the uni-dimensional case, one can show that 
	\[
	\abs{\frac{1}{1+h} - 1} \leq 2\abs{h}
	\, ,
	\]
	for any $h\in\Reals$ such that $\abs{h}\leq \frac{1}{2}$, showing that the inequality cannot be significantly improved. 
\end{myremark}

We are now able to state and prove the main result of this section. 

\begin{mylemma}[Binding lemma]
	\label{lemma:binding-lemma}
	Let $X\in\Reals^{D\times D}$ such that $X$ is invertible and $Y\in\Reals^D$. 
	Then, for any $H\in\Reals^{D\times D}$ such that $\opnorm{X^{-1}H}\leq 0.32$ and any $H'\in\Reals^D$, it holds that
	\begin{equation}
	\label{eq:key-majoration}
	\norm{(X+H)^{-1}(Y+H') - X^{-1}Y} \leq 2\opnorm{X^{-1}}\norm{H'} + 2\opnorm{X^{-1}}^2 \norm{Y}\opnorm{H} 
	\, .
	\end{equation}
\end{mylemma}

In particular, we achieve the promised control by setting $X=\Sigma$, $Y=\Gamma^f$, $H=\Sigmahat-\Sigma$, and $H'=\Gammahat -\Gamma^f$ in Lemma~\ref{lemma:binding-lemma}. 
Namely, 
\begin{equation}
\label{eq:key-control}
\norm{\betahat - \beta^f} \leq 2 \opnorm{\Sigma^{-1}} \norm{\Gammahat-\Gamma^f} + 2 \opnorm{\Sigma^{-1}}^2 \norm{\Gamma^f}\opnorm{\Sigmahat - \Sigma}
\, .
\end{equation}

\begin{proof}
	We first notice that since $\opnorm{X^{-1}H}\leq 0.32$, the matrix $\Identity_D+X^{-1}H$ is invertible according to Lemma~\ref{lemma:control-inverse-identity}. 
	We deduce that $X+H$ is also invertible, with 
	\begin{equation}
	\label{eq:invertibility}
	(X+H)^{-1} = \left(X(\Identity_D + X^{-1}H)\right)^{-1} = (\Identity_D + X^{-1}H)^{-1}X^{-1}
	\, .
	\end{equation}
	
	Let us split the left-hand side of Eq.~\eqref{eq:key-majoration}: by the triangle inequality,
	\begin{align*}
	\norm{(X+H)^{-1}(Y+H') - X^{-1}Y} &\leq \norm{(X+H)^{-1}(Y+H') - (X+H)^{-1}Y} \\
	&\qquad + \norm{(X+H)^{-1}Y - X^{-1}Y} \\
	&= \norm{(X+H)^{-1}H'} + \norm{(X+H)^{-1}Y - X^{-1}Y}
	\, .
	\end{align*}
	
	Let us focus on the first term. 
	We write 
	\begin{align*}
	\opnorm{(X+H)^{-1}} &= \opnorm{(\Identity_D+X^{-1}H)^{-1}X^{-1}} \tag{Eq.~\ref{eq:invertibility}} \\
	&\leq \opnorm{(\Identity_D+X^{-1}H)^{-1}} \cdot \opnorm{X^{-1}} \tag{$\opnorm{\cdot}$ is sub-multiplicative} \\
	&\leq 2 \opnorm{X^{-1}}\tag{Lemma \ref{lemma:control-inverse-identity}} 
	\, .
	\end{align*}
	From the last display we deduce that 
	\begin{equation}
	\label{eq:control-first-term}
	\norm{(X+H)^{-1}H'} \leq 2\opnorm{X^{-1}}\norm{H'}
	\, .
	\end{equation}

	Now for the second term, we have
	\begin{align*}
	\opnorm{(X+H)^{-1} - X^{-1}} &= \opnorm{(\Identity_D + X^{-1}H)^{-1}X^{-1} - X^{-1}} \tag{Eq. \eqref{eq:invertibility}}\\
	&\leq \opnorm{(\Identity_D + X^{-1}H)^{-1}-\Identity_d} \cdot \opnorm{X^{-1}} \tag{$\opnorm{\cdot}$ is sub-multiplicative} \\
	&\leq 2\opnorm{X^{-1}H}\cdot \opnorm{X^{-1}} \tag{Lemma \ref{lemma:difference-control}} \\
	\opnorm{(X+H)^{-1} - X^{-1}}&\leq 2\opnorm{H}\cdot \opnorm{X^{-1}}^2 \tag{$\opnorm{\cdot}$ is sub-multiplicative}
	\end{align*}
	We deduce that 
	\begin{equation}
	\label{eq:control-second-term}
	\norm{(X+H)^{-1}Y - X^{-1}Y} \leq 2\opnorm{H} \cdot \opnorm{X^{-1}}^2 \cdot \norm{Y}
	\, .
	\end{equation}
	We conclude the proof by adding Eq.~\eqref{eq:control-first-term} and Eq.~\eqref{eq:control-second-term}.
\end{proof}


\subsection{Concentration of $\betahat_n$}

We are now able to state and prove our main result, the concentration of $\betahat_n$ around $\beta^f$ for a general $f$ and general weights satisfying Eq.~\eqref{eq:definition-weights}. 
Recall that $\bigc\in (0,1]$ is the normalization constant defined in Eq.~\eqref{eq:normalization-constant}. 

\begin{mytheorem}[Concentration of $\betahat_n$, bounded model, general weights]
\label{th:betahat-concentration-general-f-general-weights}
Assume that Assumption~\ref{assump:bounded} and~\ref{assump:no-regularization} hold. 
Let $\epsilon >0$ be a small constant, at least smaller than $\boundedcst$. 
Let $\eta\in (0,1)$.  
Then, for every 
\[
n\geq \max \biggl\{\frac{2^{12}\boundedcst d^4p^4\exps{\frac{1}{\nu^2}}\log \frac{8d}{\eta}}{\bigc^2\epsilon^2},\frac{2^{15}\boundedcst^2d^7p^8\exps{\frac{2}{\nu^2}} \log \frac{8d}{\eta}}{\bigc^4\epsilon^2}\biggr\}
\, ,
\]
we have $\proba{\smallnorm{\betahat_n - \beta^f} \geq \epsilon} \leq \eta$. 
\end{mytheorem}

\begin{proof}
The main idea of the proof is to use Proposition~\ref{prop:sigmahat-concentration-general-weights} and Proposition~\ref{prop:gammahat-concentration} to control $\smallopnorm{\Sigmahat_n-\Sigma}$ and $\smallnorm{\Gammahat_n - \Gamma^f}$ with high probability. 
This control, in conjunction with the bound on $\opnorm{\Sigma^{-1}}$ given by Proposition~\ref{prop:upper-bound-operator-norm}, allows use to bound $\norm{\beta}$ thanks to Lemma~\ref{lemma:difference-control}. 

We start by setting
\[
t_1 \defeq \frac{\bigc\epsilon}{8\sqrt{2}dp^2 \exps{\frac{1}{2\nu^2}}}
\quad \text{and}\quad 
t_2 \defeq \frac{\bigc^2\epsilon}{32\boundedcst d^{5/2}p^4\exps{\frac{1}{\nu^2}}}
\, ,
\]
the desired bounds on $\smallnorm{\Gammahat_n - \Gamma^f}$ and $\smallopnorm{\Sigmahat_n-\Sigma}$, respectively.
More precisely, according to Proposition~\ref{prop:gammahat-concentration} , we can choose $\Omega_1$ such that $\proba{\Omega_1} \geq 1-4d\exp{-nt_1^2/(32\boundedcst d^2)}$, and $\smallnorm{\Gammahat_n - \Gamma^f}\leq t_1$ on $\Omega_1$. 
In the same fashion, according to  Proposition~\ref{prop:sigmahat-concentration-general-weights}, we can choose $\Omega_2$ such that $\proba{\Omega_2}\geq 1-4d\exp{-nt_2^2/(32d^2)}$, and  $\smallopnorm{\Sigmahat_n-\Sigma} \leq t_2$ on $\Omega_2$. 
Let us now define 
\[
n_1 \defeq \frac{2^{12}\boundedcst d^4p^4 \exps{\frac{1}{\nu^2}}\log\frac{8d}{\eta}}{\bigc^2\epsilon^2}
\quad \text{and}\quad 
n_2 \defeq \frac{2^{15}\boundedcst^2d^7p^8\exps{\frac{2}{\nu^2}} \log \frac{8d}{\eta}}{\bigc^4\epsilon^2}
\, .
\]
By assumption, $n$ is larger than $\max(n_1,n_2)$. 
It is straightforward to check that, in this case, $\Omega_1$ and $\Omega_2$ both have probability higher than $1-\eta/2$. 
We now work on the event $\Omega\defeq \Omega_1\cap \Omega_2$. 
By the union bound, $\Omega$ has probability greater than $1-\eta$. 
Let us show that, on $\Omega$, $\smallnorm{\betahat_n-\beta^f}\leq \epsilon$, which will conclude the proof. 

First we note that, according to Proposition~\ref{prop:upper-bound-operator-norm}, $\smallopnorm{\Sigma^{-1}} \leq \frac{2\sqrt{2}dp^2\exps{\frac{1}{2\nu^2}}}{\bigc}$. 
We write
\begin{align}
\smallopnorm{\Sigma^{-1}(\Sigmahat_n-\Sigma)} &\leq \opnorm{\Sigma^{-1}} \smallopnorm{\Sigmahat_n-\Sigma} \tag{$\opnorm{\cdot}$ is sub-multiplicative} \\
&\leq \frac{2\sqrt{2}dp^2\exps{\frac{1}{2\nu^2}}}{\bigc} \cdot \frac{\bigc^2\epsilon}{32\boundedcst d^{5/2}p^4\exps{\frac{1}{\nu^2}}} \tag{$\smallopnorm{\Sigmahat_n-\Sigma} \leq t_2$} \\
&\leq \frac{\sqrt{2}}{16} \cdot \frac{\epsilon}{\boundedcst}\cdot \bigc \cdot \frac{1}{d^{3/2}} \cdot \frac{1}{p^2} \cdot \exps{\frac{-1}{\nu^2}} \notag
\, .
\end{align}
Since $\bigc\leq 1$ (see the discussion following Eq.~\eqref{eq:normalization-constant}), and since we assumed $\epsilon < \boundedcst$, we have $\smallopnorm{\Sigma^{-1}(\Sigmahat_n-\Sigma)}\leq \sqrt{2}/16 < 0.32$. 
Therefore the assumptions of Lemma~\ref{lemma:difference-control} are satisfied, and we can use Eq.~\eqref{eq:key-control} to control $\smallnorm{\betahat_n-\beta^f}$. 
We write
\begin{align}
\smallnorm{\betahat_n-\beta^f} &\leq 2 \opnorm{\Sigma^{-1}} \smallnorm{\Gammahat-\Gamma^f} + 2 \opnorm{\Sigma^{-1}}^2 \smallnorm{\Gamma^f}\smallopnorm{\Sigmahat - \Sigma} \notag \\
&\leq \frac{4\sqrt{2}dp^2\exps{\frac{1}{2\nu^2}}}{\bigc}\cdot \frac{\bigc\epsilon}{8\sqrt{2}dp^2\exps{\frac{1}{2\nu^2}}} \tag{Prop.~\ref{prop:upper-bound-operator-norm} $+$ def. of $t_1$} \\
& + \frac{16d^2p^4\exps{\frac{1}{\nu^2}}}{\bigc^2} \cdot \boundedcst\sqrt{d}\cdot \frac{\bigc^2\epsilon}{32\boundedcst d^{5/2}p^4\exps{\frac{1}{\nu^2}}} \tag{Prop.~\ref{prop:upper-bound-operator-norm} $+$ $f$ bounded $+$ def. of $t_2$} \\
\smallnorm{\betahat_n-\beta^f} &\leq \epsilon \notag 
\, ,
\end{align}
which concludes the proof. 
\end{proof}


\section{Extension of Proposition~\ref{prop:ignoring-unused-coordinates}}
\label{section:ignoring-unused-coordinates-general-weights}

We now present a generalization and a proof of Proposition~\ref{prop:ignoring-unused-coordinates} for general weights. 
That is, we show that Tabular LIME ignores unused coordinates for general weights.

\begin{myproposition}[Dummy features, general weights]
	\label{prop:ignoring-unused-coordinates-general-weights}
	Assume that $f$ satisfies satisfies the assumptions of Proposition~\ref{prop:ignoring-unused-coordinates}. 
	Let $j\in \Sbar$, where $S$ is the set of indices relevant for $f$ and $\Sbar \defeq \{1,\ldots,d\}\setminus S$. 
	Then $\beta^f_j=0$. 
\end{myproposition}

\begin{proof}
The proof is a direct application of  Theorem~\ref{th:betahat-concentration-general-f-general-weights}, and can be seen as a straightforward generalization of the proof of Proposition~\ref{prop:ignoring-unused-coordinates}. 
	We compute first
	\begin{align*}
	\expec{\pi f(x)} &= \expec{\prod_{k=1}^{d}\exps{\frac{-(\tau_k(x_{k}) - \tau_k(\xi_k))^2}{2\nu^2}} f(x)} \tag{Eq. \eqref{eq:definition-weights}} \\
	&= \expec{\prod_{k=1}^{d}\exps{\frac{-(\tau_k(x_{k}) - \tau_k(\xi_k))^2}{2\nu^2}} g(x_{j_1},\ldots,x_{j_s})} \tag{by assumption} \\
	&= \prod_{k\in\Sbar}\expec{\exps{\frac{-(\tau_k(x_{k}) - \tau_k(\xi_k))^2}{2\nu^2}}} \cdot \expec{\prod_{k\in S}\exps{\frac{-(\tau_k(x_{k}) - \tau_k(\xi_k))^2}{2\nu^2}} g(x_{j_1},\ldots,x_{j_s})} \tag{independence} \\
	\expec{\pi f(x)} &=\prod_{k\in\Sbar}\littlec_k \cdot G \tag{Lemma \ref{lemma:base-computation}}
	\, ,
	\end{align*}
	where we set
	\[
	G \defeq \expec{\prod_{k\in S}\exps{\frac{-(\tau_k(x_{k}) - \tau_k(\xi_k))^2}{2\nu^2}} g(x_{j_1},\ldots,x_{j_s})}
	\]
	in the last display. 
	The other computation is similar. 
	Recall that $j\notin S$:
	\begin{align*}
	\expec{\pi z_{j}f(x)} &= \expec{\prod_{k=1}^{d}\exps{\frac{-(\tau_k(x_{k}) - \tau_k(\xi_k))^2}{2\nu^2}} z_{j}f(x)} \tag{Eq. \eqref{eq:definition-weights}} \\
	&= \prod_{k\in\Sbar \setminus\{j\}} \littlec_k \cdot \expec{\exps{\frac{-(\tau_j(x_{j}) - \tau_j(\xi_j))^2}{2\nu^2}}z_{j}}\cdot \expec{\prod_{k\in S} \exps{\frac{-(\tau_k(x_{k}) - \tau_k(\xi_k))^2}{2\nu^2}} g(x_{j_1},\ldots,x_{j_s})} \tag{independence} \\
	\expec{\pi z_{j}f(x)} &= \frac{\prod_{k\in\Sbar}\littlec_k}{\littlec_j} \cdot \frac{1}{p}\ew_{j,\bxi_j} \cdot G
	\, .
	\end{align*}
	Finally we write
	\begin{align*}
	\beta^f_j &= \bigc^{-1}\frac{p\littlec_j}{p\littlec_j-\ew_{j,\bxi_j}}\left(-\expec{\pi f(x)} + \frac{p\littlec_j}{\ew_{j,\bxi_j}}\expec{\pi z_{j}f(x)} \right) \tag{Prop. \ref{prop:beta-computation-general-f-general-weights}} \\
	&= \bigc^{-1}\frac{p\littlec_j}{p\littlec_j-\ew_{j,\bxi_j}}\left(-\prod_{k\in\Sbar}\littlec_k \cdot G + \frac{p\littlec_j}{\ew_{j,\bxi_j}} \frac{\prod_{k\in\Sbar}\littlec_k}{\littlec_j} \cdot \frac{1}{p}\ew_{j,\bxi_j} \cdot G \right) \\
	\beta^f_j &= 0
	\, .
	\end{align*}
\end{proof}


\section{Technical results}
\label{section:technical-results}

In this Appendix we collect technical results used throughout the main paper. 

\subsection{Expected values computations}

We begin with the computation of the expected values needed for the computation of $\Sigma$ and $\Gamma^f$. 
We start with a generic lemma, a very common computation in our proofs which appears each time we use the independence assumption between the $x_{ij}$ and we split the $\pi_i$ product. 

\begin{mylemma}[Basic computations]
	\label{lemma:base-computation}
	Let $\psi:\Reals \to\Reals$ be a function which is bounded on the support of $x_j$ for any $1\leq j\leq d$. 
	Then
	\[
	\begin{cases}
	\expec{\exp{\frac{-(\tau_j(x_{j}) - \tau_j(\xi_j))^2}{2\nu^2}} \psi(x_{j})} &= \frac{1}{p}\sum_{b=1}^p \ew_{j,b}^{\psi} \, ,\\
	\expec{\exp{\frac{-(\tau_j(x_{j}) - \tau_j(\xi_j))^2}{2\nu^2}} z_{j}\psi(x_{j})} &= \frac{1}{p} \ew_{j,\bxi_j}^\psi
	\, .
	\end{cases}
	\]
\end{mylemma}

In particular, 
\begin{equation}
\label{eq:aux:total-expectation}
\expec{\exp{\frac{-1}{2\nu^2}(\tau_j(\xi_j) - \tau_j(x_{j}))^2}} = \littlec_j
\, .
\end{equation}

\begin{proof}
Straightforward from the law of total expectation and the definition of the $\ew_{j,b}$ coefficients.
\end{proof}

Lemma~\ref{lemma:base-computation} is the reason why the $\ew_{j,b}$ are ubiquitous in our results. 
If the weights have some multiplicative structure, it is easy to extend Lemma~\ref{lemma:base-computation} to the full weights, which we achieve in our next result. 

\begin{mylemma}[Key computation]
	\label{lemma:key-computation}
	Suppose that $\pi_i$ satisfies Eq.~\eqref{eq:definition-weights}. 
	Let $\psi:\Reals \to \Reals$ be a function bounded on the support of $x_j$ for any $1\leq j\leq d$. 
	Then, for any given $i,j,k$ with $j\neq k$, 
	\[
	\begin{cases}
	\expec{\pi \psi(x_{j})} &= \frac{\bigc}{p\littlec_j} \sum_{b=1}^p \ew_{j,b}^\psi \\
	\expec{\pi z_{j} \psi(x_{j})} &= \frac{\bigc }{p\littlec_j}\ew_{j,\bxi_j}^\psi \\
	\expec{\pi z_{j}\psi(x_{k})} &= \frac{\bigc}{p^2\littlec_j\littlec_k}\ew_{j,\bxi_j} \sum_{b=1}^p \ew_{k,b}^\psi
	\, .
	\end{cases}
	\]
\end{mylemma}

\begin{proof}
	We write
	\begin{align*}
	\expec{\pi \psi(x_{j})} &= \expec{\exp{\frac{-1}{2\nu^2}\sum_{k=1}^{d}(\tau_k(\xi_k) - \tau_k(x_{k}))^2} \cdot \psi(x_{j})} \tag{Eq. \eqref{eq:definition-weights}} \\
	&= \prod_{k\neq j} \littlec_k \cdot \expec{\psi(x_{j}) \exp{\frac{-1}{2\nu^2}(\tau_j(\xi_j) - \tau_j(x_j))^2}}\tag{independence $+$ Eq. \eqref{eq:aux:total-expectation}} \\
	&= \prod_{k\neq j} \littlec_k \cdot \frac{1}{p} \sum_{b=1}^p \ew_{j,b}^\psi \tag{Lemma \ref{lemma:base-computation}} \\
	\expec{\pi\psi(x_{j})} &= \frac{\bigc}{p\littlec_j} \sum_{b=1}^p \ew_{j,b}^\psi \tag{definition of $\bigc$}
	\, .
	\end{align*}
	The proofs of the remaining results are quite similar:
	\begin{align*}
	\expec{\pi z_{j}\psi(x_{j})} &= \expec{\exp{\frac{-1}{2\nu^2}\sum_{k=1}^{d}(\tau_k(\xi_k) - \tau_k(x_{k}))^2} \cdot z_{j}\psi(x_{j})} \tag{Eq. \eqref{eq:definition-weights}} \\
	&= \prod_{k\neq j} \littlec_k \cdot \expec{\psi(x_{j})z_{j} \exp{\frac{-1}{2\nu^2}(\tau_j(\xi_j) - \tau_j(x_{x_{j}}))^2}}\tag{independence $+$ Eq. \eqref{eq:aux:total-expectation}} \\
	&= \prod_{k\neq j} \littlec_k \cdot \frac{1}{p}\ew_{j,\bxi_j}^\psi \tag{Lemma \ref{lemma:base-computation}} \\
	\expec{\pi z_{j}\psi(x_{j})} &= \frac{\bigc }{p\littlec_j}\ew_{j,\bxi_j}^\psi
	\end{align*}
	\begin{align*}
	\expec{\pi z_{j}\psi(x_{k})} &= \expec{\exp{\frac{-1}{2\nu^2}\sum_{\ell=1}^{d}(\tau_\ell(\xi_\ell) - \tau_\ell(x_{\ell}))^2} \cdot z_{j}\psi(x_{k})} \tag{Eq. \eqref{eq:definition-weights}} \\
	&= \prod_{\ell\neq j,k} \littlec_\ell \cdot \expec{z_{j}\exp{\frac{-1}{2\nu^2}(\tau_j(\xi_j) - \tau_j(x_{j}))^2}} \\
	& \phantom{blablabla}\cdot \expec{\psi(x_{k})\exp{\frac{-1}{2\nu^2}(\tau_k(\xi_k) - \tau_k(x_{k}))^2}} \tag{independence} \\
	&= \prod_{\ell\neq j,k} \littlec_\ell \cdot \frac{\ew_{j,\bxi_j}}{p} \cdot \frac{1}{p}\sum_{b=1}^{p}\ew_{k,b}^\psi \tag{Lemma \ref{lemma:base-computation}} \\
	\expec{\pi z_{j}\psi(x_{k})} &= \frac{\bigc}{p^2\littlec_j\littlec_k}\ew_{j,\bxi_j} \sum_{b=1}^p \ew_{k,b}^\psi \tag{definition of $\bigc$}
	\, .
	\end{align*}
\end{proof}

We specialize Lemma~\ref{lemma:key-computation} in the case $\psi=1$, since the $\ew_{j,b}$ coefficients are ubiquitous in our computations. 

\begin{mylemma}[Expected values computations, zero-th order]
	\label{lemma:expectation-computations}
	For any $j\neq k$, 
	\[
	\begin{cases}
	\expec{\pi} = \bigc \, ,\\
	\expec{\pi z_{j}} = \bigc \frac{\ew_{j,\bxi_j}}{p\littlec_j} \, ,\\
	\expec{\pi z_{j}z_{k}} = \bigc\, , \frac{\ew_{j,\bxi_j}}{p\littlec_j}\frac{\ew_{k,\bxi_k}}{p\littlec_k}\, .
	\end{cases}
	\]
\end{mylemma}

\begin{proof}
	The first two results are a direct consequence of Lemma~\eqref{lemma:key-computation} for $\psi=1$. 
	For the third one, we set $\psi(x)=\indic{x\in[q_{k,\bxi_k-1},q_{k,\bxi_k}]}$, and we notice that, in this case, 
	\[
	\ew_{k,b}^\psi = 
	\begin{cases}
	\ew_{k,\bxi_k} &\text{ if }b=\bxi_k \, ,\\
	0 &\text{ otherwise.}
	\end{cases}
	\]
\end{proof}

The case $\psi=\id$ is also of some importance in our analysis, let us specialize Lemma~\ref{lemma:key-computation} in that case as well. 

\begin{mylemma}[Expected values, first order]
	\label{lemma:expectation-computations-1}
	Let $j,k\in\{1,\ldots,d \}$ be fixed indices, with $j\neq k$. 
	Then 
	\[
	\begin{cases}
	\expec{\pi x_{j}} &= \frac{\bigc}{p\littlec_j}\sum_{b=1}^{p} \ew_{j,b}^\times \, ,\\
	\expec{\pi z_{j}x_{j}} &= \frac{\bigc}{p\littlec_j}\ew_{j,\bxi_j}^\times \, ,\\
	\expec{\pi z_{j}x_{k}} &= \frac{\bigc}{p^2\littlec_j\littlec_k}\ew_{j,\bxi_j}\sum_{b=1}^{p} \ew_{k,b}^\times\, .
	\end{cases}
	\]
\end{mylemma}

\begin{proof}
Straightforward from Lemma~\ref{lemma:key-computation} with $\psi = \id$. 
\end{proof}

\subsection{Some facts about operator norm}

In this section, we collect some facts about the operator norm that are used in Appendix~\ref{section:concentration-betahat}. 

\begin{mylemma}[Inversion formula for the operator norm]
	\label{lemma:operator-norm:inversion}
	Let $M\in\Reals^{d\times d}$ be an invertible matrix. 
	Then 
	\[
	\opnorm{M^{-1}} = \left(\lambdamin{M^\top M}\right)^{-1/2}
	\, .
	\]
\end{mylemma}

\begin{proof}
	By the definition of the operator norm, we know that
	\[
	\opnorm{M^{-1}}^2 = \lambdamax{(M^{-1})^\top M^{-1}}
	\, .
	\]
	Since we are in a commutative ring, $(M^{-1})^\top=(M^\top)^{-1}$. 
	Additionally, for any two matrices such that $AB$ is invertible, $(AB)^{-1}=B^{-1}A^{-1}$. 
	Therefore
	\[
	\opnorm{M^{-1}}^2 = \lambdamax{(MM^\top)^{-1}}
	\, .
	\]
	Since $MM^\top$ is a positive definite matrix, $\spec{MM^\top}\subseteq \Reals_+$, and $\lambdamax{(MM^\top)^{-1}} = \lambdamin{MM^\top}^{-1}$. 
	We can conclude since for any two matrices, $AB$ and $BA$ have the same spectrum. 
\end{proof}

\begin{mylemma}[Bounding the operator norm]
	\label{lemma:operator-norm:bounds}
	For any matrix $M\in\Reals^{d\times d}$, we have
	\[
	\opnorm{M}\leq \frobnorm{M} \leq \sqrt{d}\opnorm{M}
	\, .
	\]
\end{mylemma}

\begin{proof}
	We first write
	\[
	\opnorm{M} =\lambda_{\max}(M^\top M)\leq \sum_j \lambda_j(M^\top M) = \trace{M^\top M} = \frobnorm{M}
	\, . 
	\]
	As for the second part of the result, we write 
	\begin{align*}
	\frobnorm{M}^2 &= \trace{M^\top M} \tag{definition} \\
	&= \sum_{i=1}^{d} \lambda_i(M^\top M) \tag{property of the trace} \\
	&\leq d\lambdamax{M^\top M} \tag{non-negative eigenvalues}
	\end{align*}
\end{proof}


\section{Additional experiments}
\label{sec:additional-experiments}

In this section, we present additional experiments not included in the main paper. 
Figure~\ref{fig:linear-appendix}, \ref{fig:tree-appendix}, and~\ref{fig:krr-appendix} show that our results are robust to an increase in the dimension of the model. 
Figure~\ref{fig:bandwidth-cancellation-appendix} showcases the evolution of the interpretable coefficients of the example of Figure~\ref{fig:bandwidth-cancellation} for the remaining features. 

\begin{figure}
	\begin{center}
		\includegraphics[scale=0.3]{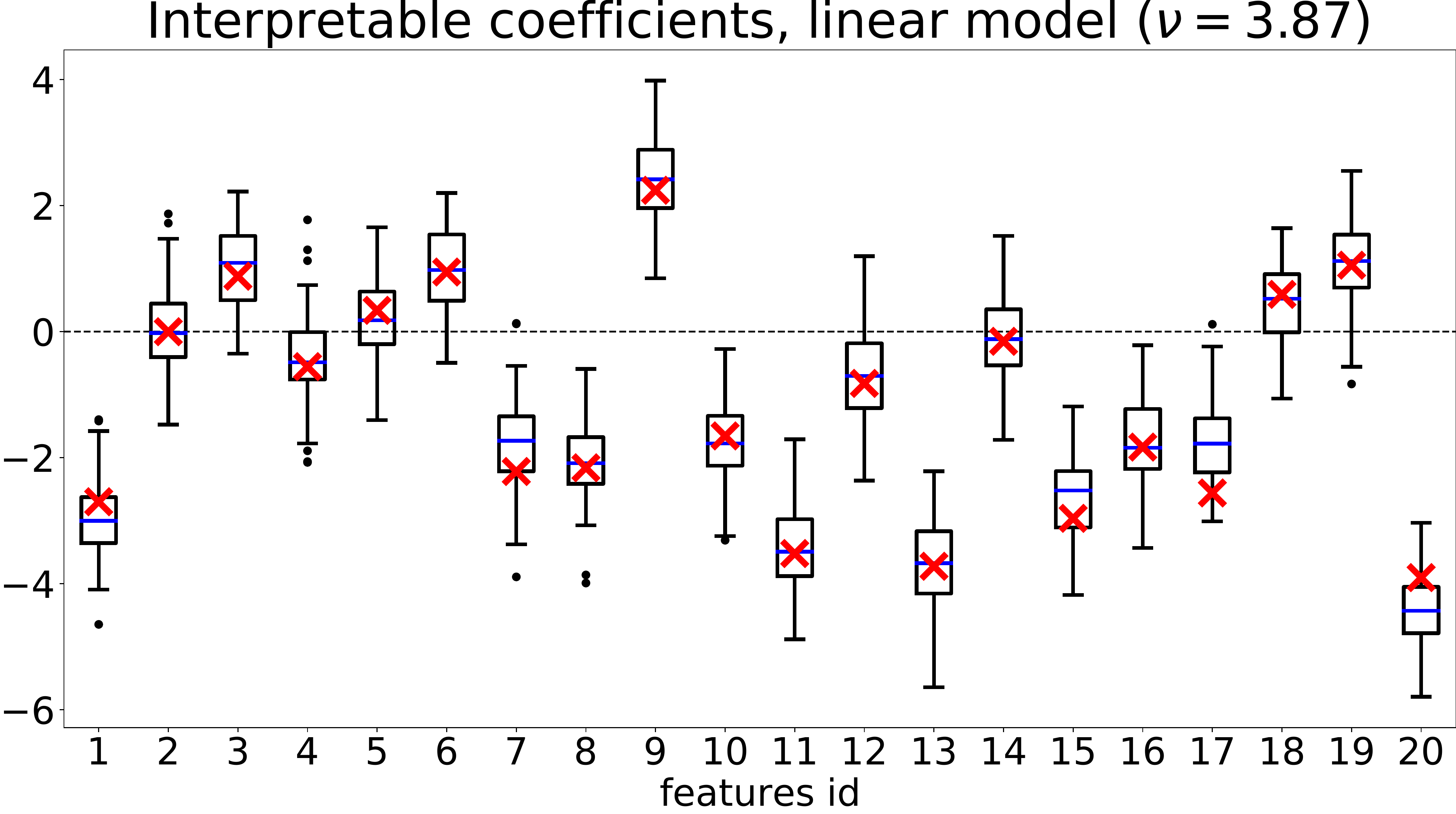}
	\end{center}
	\caption{\label{fig:linear-appendix}Interpretable coefficients for a linear model with arbitrary coefficients in dimension $20$. In red, the values given by Corollary~\ref{cor:beta-computation-liner-default-weights}. The match between theory and practice remains quite accurate.}
\end{figure}

\begin{figure}
	\centering
	\includegraphics[scale=0.3]{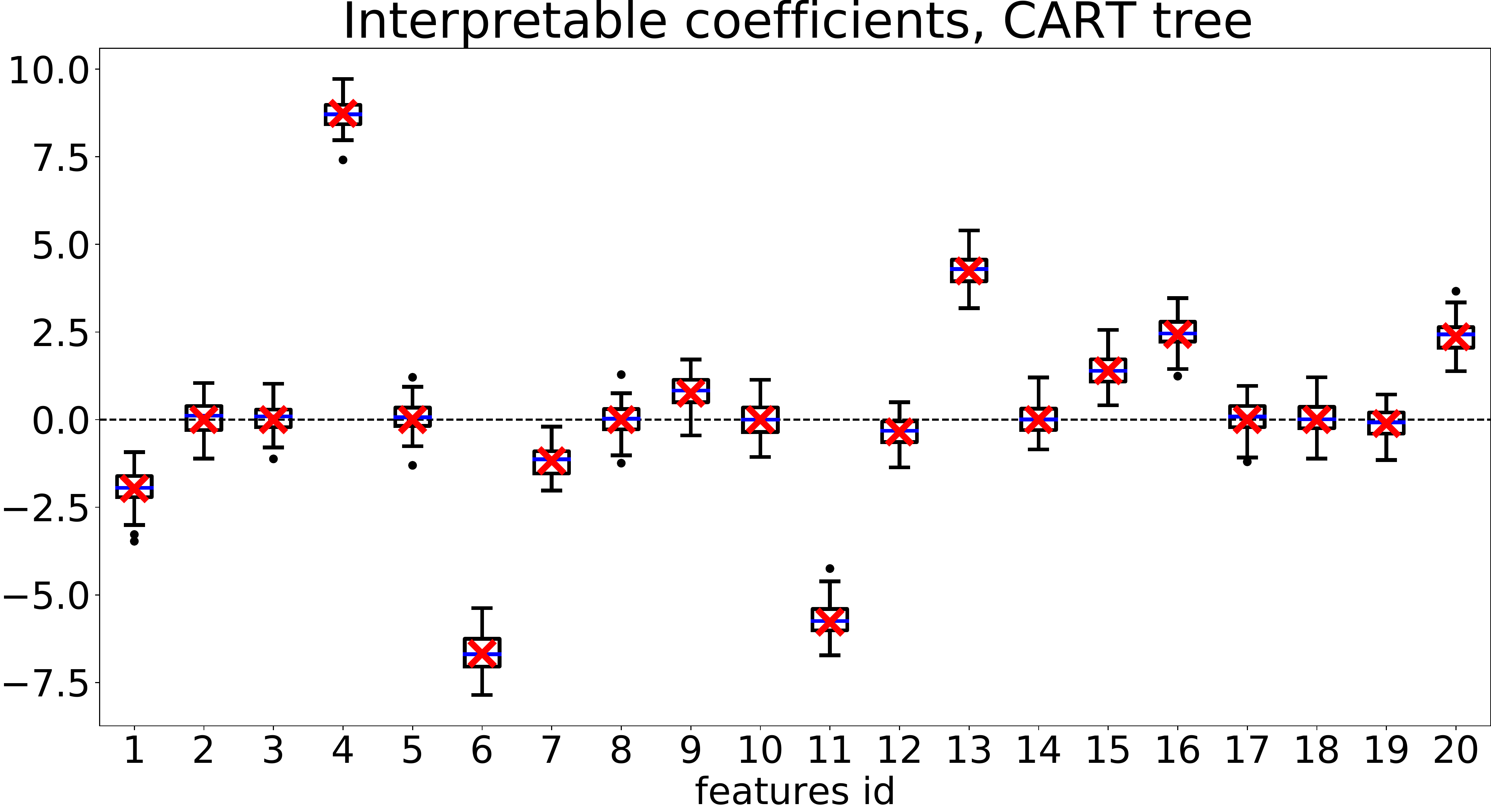}
	\caption{\label{fig:tree-appendix}Interpretable coefficients for a CART tree in dimension $20$. The tree, of depth $3$, was fitted on the function $x\mapsto \sum_j x_j$ with training data is uniform in $[-10,10]^d$. We repeated the experiment $100$ times. The theoretical predictions (red crosses) are obtained by retrieving the partition associated to the tree and using Corollary~\ref{cor:beta-computation-tree}. }
\end{figure}

\begin{figure}
	\begin{center}	
		\includegraphics[scale=0.3]{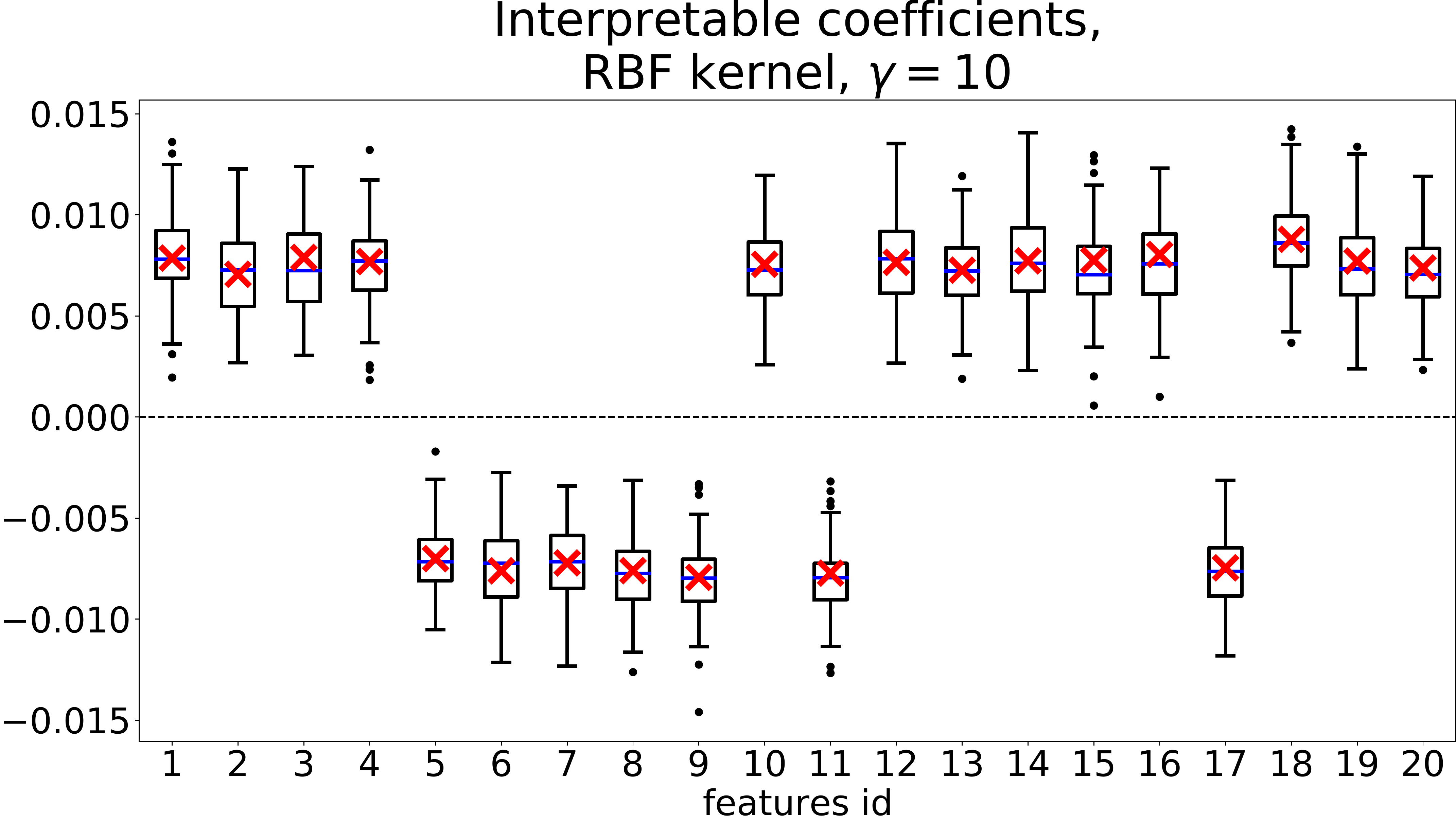}
	\end{center}
	\caption{\label{fig:krr-appendix}Interpretable coefficients for a Gaussian kernel with bandwidth parameter $\gamma=10$, in dimension $20$. We repeated the experiment $100$ times. The theoretical predictions (in red) come from Proposition~\ref{prop:ews-computation-gaussian-kernel}. }
\end{figure}

\begin{figure}
\begin{center}
\includegraphics[scale=0.25]{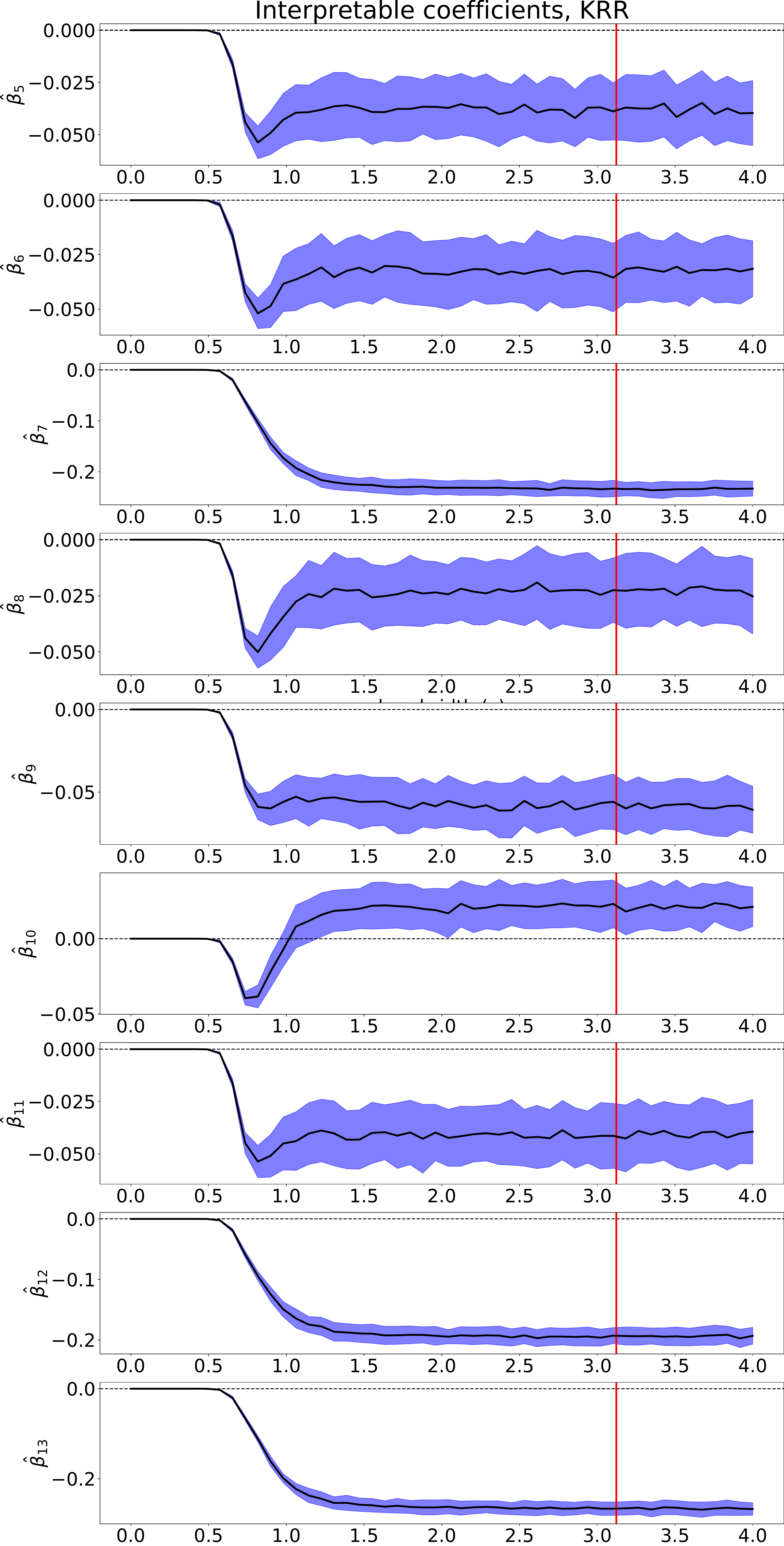}
\end{center}
\caption{\label{fig:bandwidth-cancellation-appendix}Evolution of the interpretable features as a function of the bandwidth. The model is a kernel ridge regressor trained on the Wine dataset. }
\end{figure}

\end{document}

%% file: main-setup.tex
\DeclareMathOperator*{\Argmin}{arg\,min}
\DeclareMathOperator{\Diag}{diag}
\DeclareMathOperator{\Exp}{exp}
\DeclareMathOperator{\Exps}{e}
\DeclareMathOperator{\Expec}{\mathbb{E}}
\DeclareMathOperator{\Identity}{I}
\DeclareMathOperator{\Proba}{\mathbb{P}}
\DeclareMathOperator{\Trace}{Tr}

\newcommand{\Abar}{\overline{A}}
\newcommand{\abs}[1]{\left\lvert#1\right\rvert}
\newcommand{\card}[1]{\left\lvert#1\right\rvert}
\newcommand{\condexpec}[2]{\mathbb{E}\left[#1\middle|#2\right]}
\newcommand{\condproba}[2]{\mathbb{P}\left(#1\middle|#2\right)}

\newcommand{\diag}[1]{\Diag\left(#1\right)}
\newcommand{\Diff}{\mathrm{d}}
\newcommand{\dist}[2]{d\left(#1,#2\right)}
\newcommand{\defeq}{\vcentcolon =}
\newcommand{\eqdef}{= \vcentcolon}
\renewcommand{\exp}[1]{\Exp\left(#1\right)}
\newcommand{\expec}[1]{\Expec\left[#1\right]}
\newcommand{\id}{\mathrm{id}}
\newcommand{\smallexpec}[1]{\Expec[#1]}
\newcommand{\exps}[1]{\Exps^{#1}}
\newcommand{\frobnorm}[1]{\norm{#1}_{\mathrm{F}}}
\newcommand{\Indic}{\mathds{1}}
\newcommand{\indic}[1]{\Indic_{#1}}
\newcommand{\infnorm}[1]{\norm{#1}_{\infty}}
\newcommand{\kernel}[1]{\mathrm{Ker}\left(#1\right)}
\newcommand{\mutilde}{\tilde{\mu}}
\newcommand{\muttilde}{\tilde{\tilde{\mu}}}

\newcommand{\norm}[1]{\left\lVert#1\right\rVert}
\newcommand{\opnorm}[1]{\norm{#1}_{\mathrm{op}}}
\newcommand{\partition}{\mathcal{P}}
\newcommand{\scalar}[2]{\langle#1,#2\rangle}
\newcommand{\proba}[1]{\Proba\left (#1\right )}
\newcommand{\Reals}{\mathbb{R}}
\newcommand{\relimp}[1]{\mathrm{I}(#1)}
\newcommand{\spec}[1]{\mathrm{Spec}\left(#1\right)} 
\newcommand{\Sbar}{\overline{S}}
\newcommand{\smallcondexpec}[2]{\mathbb{E}[#1|#2]}
\newcommand{\smallnorm}[1]{\lVert#1\rVert}
\newcommand{\smallopnorm}[1]{\smallnorm{#1}_{\mathrm{op}}}
\newcommand{\smallproba}[1]{\Proba (#1)}
\newcommand{\supp}{\mathcal{S}} 
\newcommand{\trace}[1]{\Trace\left(#1\right)}
\newcommand{\traindata}{\mathcal{X}}
\newcommand{\volume}[1]{\mathrm{Vol}(#1)}

\newcommand{\bigo}[1]{\mathcal{O}\left(#1\right)}

\newcommand{\cvproba}{\overset{\Proba}{\longrightarrow}}

\renewcommand{\epsilon}{\varepsilon}

\theoremstyle{plain}
\newtheorem{mytheorem}{Theorem}
\newtheorem{myproposition}[mytheorem]{Proposition}
\newtheorem{mylemma}[mytheorem]{Lemma}
\newtheorem{mycorollary}[mytheorem]{Corollary}

\theoremstyle{definition}
\newtheorem{assumption}{Assumption}
\newtheorem{mydefinition}{Definition}
\newtheorem{myremark}{Remark}
\newtheorem{myexample}{Example}

\newcommand{\betahat}{\hat{\beta}}

\newcommand{\bigc}{C}
\newcommand{\boundedcst}{C_f}
\newcommand{\bxi}{b^\star}
\newcommand{\bzero}{b^\circ} 
\newcommand{\cprop}{v}
\newcommand{\erfun}[1]{\mathrm{erf}\left(#1\right)}
\newcommand{\ew}{e}
\newcommand{\et}{e^t} 
\newcommand{\littlec}{c}
\newcommand{\Gammahat}{\hat{\Gamma}}
\newcommand{\ek}{e^k} 
\newcommand{\lambdamax}[1]{\lambda_{\mathrm{max}}\left(#1\right)}
\newcommand{\lambdamin}[1]{\lambda_{\mathrm{min}}\left(#1\right)}
\newcommand{\mtilde}{\tilde{m}}
\newcommand{\Sigmahat}{\hat{\Sigma}}
\newcommand{\sigmatilde}{\tilde{\sigma}}
\newcommand{\stilde}{\tilde{s}}
\newcommand{\truncnorm}[1]{\mathrm{TN}(#1)}
\newcommand{\uniform}[1]{\mathcal{U}(#1)}

\makeatletter
\def\hlinewd#1{%
	\noalign{\ifnum0=`}\fi\hrule \@height #1 %
	\futurelet\reserved@a\@xhline}
\makeatother

\makeatletter
\def\th@plain{%
	\thm@notefont{}
	\itshape 
}
\def\th@definition{%
	\thm@notefont{}
	\normalfont 
}
\makeatother

\ifx\ulesred\undefined
\usepackage{color}

\definecolor{ulesred}{rgb}{0.7,0,0}
\newcommand{\ulesred}{\color{ulesred}}

\definecolor{ulesgreen}{cmyk}{1,0,1,0}

\definecolor{verylightgray}{gray}{0.85} 
\definecolor{mediumgray}{gray}{0.5} 

\else

\fi

\ifx\black\undefined
\newcommand{\black}{\color{black}}
\fi

\ifx\blue\undefined
\newcommand{\blue}{\color{blue}}
\fi

\ifx\red\undefined
\newcommand{\red}{\color{red}}
\fi

\ifx\green\undefined
\newcommand{\green}{\color{green}}
\fi

\ifx\black\undefined
\newcommand{\black}{\color{black}}
\fi 

\ifx\white\undefined
\newcommand{\white}{\color{white}}
\fi